\documentclass{article}

\usepackage[top=1in, bottom=1in, left=1in, right=1in]{geometry}

\usepackage[T1]{fontenc}
\usepackage{xcolor} 
\definecolor{safeOrange}{HTML}{BDFFF4}
\definecolor{backgroundOrange}{HTML}{FDD7D7}
\definecolor{safePurple}{HTML}{FE6100}
\definecolor{safeLighterPurple}{HTML}{9582F3}
\definecolor{safePink}{HTML}{D81B60}
\definecolor{safeBlue}{HTML}{1E88E5}
\usepackage{mathtools} 
\usepackage{amsthm}
\usepackage{orcidlink}
\usepackage[table]{xcolor}
\usepackage[dvipsnames,svgnames]{xcolor}
\usepackage{pdflscape}

\usepackage{comment,url,algorithm,algorithmic,graphicx,subcaption,relsize}
\usepackage{amssymb,amsfonts,amsmath,amsthm,amscd,dsfont,mathrsfs,mathtools,nicefrac}
\usepackage{float,psfrag,epsfig,color,xcolor,url,hyperref}
\usepackage{epstopdf,bbm,enumitem}
\usepackage[toc,page]{appendix}

\def\ddefloop#1{\ifx\ddefloop#1\else\ddef{#1}\expandafter\ddefloop\fi}

\def\ddef#1{\expandafter\def\csname bb#1\endcsname{\ensuremath{\mathbb{#1}}}}
\ddefloop ABCDEFGHIJKLMNOPQRSTUVWXYZ\ddefloop

\def\ddef#1{\expandafter\def\csname c#1\endcsname{\ensuremath{\mathcal{#1}}}}
\ddefloop ABCDEFGHIJKLMNOPQRSTUVWXYZ\ddefloop

\def\ddef#1{\expandafter\def\csname s#1\endcsname{\ensuremath{\mathscr{#1}}}}
\ddefloop ABCDEFGHIJKLMNOPQRSTUVWXYZ\ddefloop

\def\ddef#1{\expandafter\def\csname v#1\endcsname{\ensuremath{\boldsymbol{#1}}}}
\ddefloop ABCDEFGHIJKLMNOPQRSTUVWXYZabcdefghijklmnopqrstuvwxyz\ddefloop

\def\ddef#1{\expandafter\def\csname v#1\endcsname{\ensuremath{\boldsymbol{\csname #1\endcsname}}}}
\ddefloop {alpha}{beta}{gamma}{delta}{epsilon}{varepsilon}{zeta}{eta}{theta}{vartheta}{iota}{kappa}{lambda}{mu}{nu}{xi}{pi}{varpi}{rho}{varrho}{sigma}{varsigma}{tau}{upsilon}{phi}{varphi}{chi}{psi}{omega}{Gamma}{Delta}{Theta}{Lambda}{Xi}{Pi}{Sigma}{varSigma}{Upsilon}{Phi}{Psi}{Omega}{ell}\ddefloop

\def\balign#1\ealign{\begin{align}#1\end{align}}
\def\baligns#1\ealigns{\begin{align*}#1\end{align*}}
\def\balignat#1\ealign{\begin{alignat}#1\end{alignat}}
\def\balignats#1\ealigns{\begin{alignat*}#1\end{alignat*}}
\def\bitemize#1\eitemize{\begin{itemize}#1\end{itemize}}
\def\benumerate#1\eenumerate{\begin{enumerate}#1\end{enumerate}}

\newenvironment{talign*}
 {\csname align*\endcsname}
 {\endalign}
\newenvironment{talign}
 {\csname align\endcsname}
 {\endalign}

\def\balignst#1\ealignst{\begin{talign*}#1\end{talign*}}
\def\balignt#1\ealignt{\begin{talign}#1\end{talign}}

\let\originalleft\left
\let\originalright\right
\renewcommand{\left}{\mathopen{}\mathclose\bgroup\originalleft}
\renewcommand{\right}{\aftergroup\egroup\originalright}

\def\tinycitep*#1{{\tiny\citep*{#1}}}
\def\tinycitealt*#1{{\tiny\citealt*{#1}}}
\def\tinycite*#1{{\tiny\cite*{#1}}}
\def\smallcitep*#1{{\scriptsize\citep*{#1}}}
\def\smallcitealt*#1{{\scriptsize\citealt*{#1}}}
\def\smallcite*#1{{\scriptsize\cite*{#1}}}

\def\mbb#1{\mathbb{#1}}

\theoremstyle{plain}  
\newtheorem*{remark}{\textbf{Remark}}

\def\<{\left\langle} 
\def\>{\right\rangle}

\def\P{\mbb{P}} 

\DeclareSymbolFont{rsfs}{U}{rsfs}{m}{n}
\DeclareSymbolFontAlphabet{\mathscrsfs}{rsfs}
\ifdefined\nonewproofenvironments\else
\ifdefined\ispres\else
\newtheorem{theorem}{Theorem}
\newtheorem{lemma}[theorem]{Lemma}
\newtheorem{corollary}[theorem]{Corollary}

\newtheorem{fact}[theorem]{Fact}

\renewenvironment{proof}{\noindent\textbf{Proof.}\hspace*{.3em}}{\qed\\}
\newenvironment{proof-sketch}{\noindent\textbf{Proof Sketch}
  \hspace*{1em}}{\qed\bigskip\\}
\newenvironment{proof-idea}{\noindent\textbf{Proof Idea}
  \hspace*{1em}}{\qed\bigskip\\}
\newenvironment{proof-of-lemma}[1][{}]{\noindent\textbf{Proof of Lemma {#1}}
  \hspace*{1em}}{\qed\\}
\newenvironment{proof-of-theorem}[1][{}]{\noindent\textbf{Proof of Theorem {#1}}
  \hspace*{1em}}{\qed\\}
\newenvironment{proof-of-corollary}[1][{}]{\noindent\textbf{Proof of Corollary {#1}}
  \hspace*{1em}}{\qed\\}
\newenvironment{proof-attempt}{\noindent\textbf{Proof Attempt}
  \hspace*{1em}}{\qed\bigskip\\}
\fi

\newtheorem{assumption}{Assumption}

\fi
\makeatletter
\@addtoreset{equation}{section}
\makeatother

\hypersetup{
  colorlinks,
  linkcolor={red!50!black},
  citecolor={blue!50!black},
  urlcolor={blue!80!black}
}

\usepackage{authblk}
\usepackage{etoc}
\usepackage{wrapfig}
\usepackage{bm}
\usepackage{algorithmic}
\usepackage{multicol}
\usepackage{booktabs,tabularx}
\usepackage[round]{natbib}
\allowdisplaybreaks

\newcommand\iidsim{\stackrel{\mathclap{i.i.d.}}{\sim}}

\makeatletter
\renewcommand{\paragraph}{%
  \@startsection{paragraph}{4}%
  {\z@}{1.5ex \@plus 1ex \@minus .2ex}{-1em}%
  {\normalfont\normalsize\bfseries}%
}
\makeatother

\title{Distinguishability and Recovery for a Doubly Sparse Spiked Wigner Model}
\author[1]{Ioana Dumitriu\thanks{Email: \texttt{idumitriu@ucsd.edu}}}
\author[1]{JD Flynn\thanks{Email: \texttt{jdflynn@ucsd.edu}}\orcidlink{0009-0007-5345-0135}}
\author[2]{Zhichao Wang\thanks{Email: \texttt{zhichao.wang@berkeley.edu}}}

\affil[1]{University of California, San Diego}
\affil[2]{University of California, Berkeley}
\date{}

\begin{document}

\maketitle

\begin{abstract}
We study distinguishability and recovery of a sparse rank-one signal in a spiked Wigner model with independently sparse noise through a variety of algorithms and parameter regimes. The signal and noise have separate sparsity parameters, while the signal-to-noise ratio (SNR) remains of constant order as the matrix size grows. 
For sub-Gaussian noise and spikes whose nonzero coordinates are bounded away from zero, a leave-one-out quadratic thresholding procedure  \citep[a generalization of][]{ding2022subexponentialtimealgorithmssparsepca} identifies the support. Sign synchronization yields exact recovery for a Rademacher spike, and spectral refinement yields strong recovery. The quadratic thresholding search parameter gives an explicit tradeoff between signal sparsity, noise sparsity, and computational cost. We also study extensions to sub-Gaussian spike priors through diagonal and quadratic thresholding. For Rademacher spikes with continuous nonzero noise, we develop an identical search algorithm that exploits unblurred signal entries, the existence of which permits distinguishability; connectivity of the graph they reveal provides a sufficient condition for exact recovery. These guarantees do not require moment assumptions on the continuous noise distribution. Together with our previous PCA results \citep{dumitriu2026bbpphasetransitiondoubly}, our analysis covers almost the entire space of parameters and regimes for distinguishing and recovering in our doubly sparse spiked Wigner model. 
\end{abstract}


\tableofcontents

\section{Introduction}

\subsection{Random Matrix Model}
Many applications across signal processing, graph theory, data science, and other statistical fields consist of observing a corrupted data set and extracting an informative signal. While this procedure was not a novel idea at the time, \citet{johnstone2001} coined ``the spiked model'' within the field of statistics for what is essentially a vector perturbed by noise. 
Since then, the Spiked Wishart Model and Spiked Wigner Model have emerged as the two predominantly studied classes of models, depending on whether the analysis is for the covariance matrix or not. Although there is literature \citep[e.g.,][]{bresler2025computationalequivalencespikedcovariance} that suggests a computational equivalence of the two models, in this paper, we will consider the \textit{Spiked Wigner Model}:
\begin{align}
    X=\frac{\theta}{n}vv^T+\frac{1}{\sqrt{n}}W~,\label{spiked_wigner_dense}
\end{align}
where $X$ is the observed matrix, $\theta$ is the SNR, $v$ is the spike vector, and $W$ is the Wigner matrix modeling the random noise. 
Once a model is decided upon, one considers two fundamental tasks: we can \textit{distinguish} if there exists a test statistic that can determine whether $X$ is drawn from the planted ($\theta > 0$) or null ($\theta=0$) distribution, and we can \textit{recover} if we can produce some sort of estimate for the spike vector $v$.
Within each problem, we consider two thresholds. First, one can search for an \textit{Information-Theoretic (IT) threshold} that determines the relationship between model parameters that allows for solvability. The IT threshold generally relates to distinguishability (i.e. above a specified value, there exists a statistic that can distinguish, and below that no such statistic exists). When we are above the IT threshold, one can then consider the threshold for a polynomial time algorithm for distinguishability and recovery. These thresholds can be distinct depending on the characteristics of the model. 

\subsection{The Spectral Approach}
\paragraph{BBP.}Using principal component analysis (PCA), one can distinguish and recover the spike by considering the top eigenvalue and corresponding (top) eigenvector of $X$, respectively. In their landmark paper, Baik, Ben-Arous, and P\'ech\'e established the BBP phenomenon/threshold for the aforementioned Spiked Wishart Model: when the signal strength is great enough, the top eigenvalue exits the spectral bulk and the corresponding top eigenvector achieves nontrivial correlation with the signal vector, while, underneath this signal threshold, the top eigenvalue sticks to the edge of the spectrum and is thus indistinguishable from the top eigenvalue of the null model \citep{baik2004phasetransitionlargesteigenvalue}. This result opened two avenues of expansion for distinguishing and recovering with the spectral method based on the following observation: in its original form, \cite{baik2004phasetransitionlargesteigenvalue} established the BBP phenomenon for a spiked model with a strict noise assumption (Gaussian) and a general spike assumption (unit norm). 

For the first avenue, researchers worked toward relaxing the assumptions on the noise and extending it to the Spiked Wigner Model while keeping a general assumption on the spike(s). In the case of Spiked Wishart, \cite{baiksilverstein} relaxed the Gaussian condition to a noise distribution with control over the fourth moment. \cite{peche2006largest} established the BBP for the Spiked Wigner model with GUE noise, \cite{F_ral_2007} followed by generalizing the result to models with noise distributions obeying a specified moment condition, and \cite{Capitaine_2009} further relaxed the assumptions to noise satisfying a specified Poincar\'e inequality. Finally, \cite{pizzo2011finiterankdeformationswigner} rounded out the result by establishing the Spiked Wigner BBP for a noise distribution with fifth moment control for the off-diagonal and third moment control of the diagonal. \cite{benaychgeorges2010eigenvalues} took a wider perspective by providing results on leading spectral behavior for more general additive models (which in turn captures the Spiked Wigner Model) and required unitary invariance among the spike or the noise; in addition, \cite{benaychgeorges2010eigenvalues} showed eigenvector correlation alongside their eigenvalue separation result. While BBP-like results continued, further noise generalizations often adopted stricter assumptions on the spike entries, which brings us to our second direction of BBP expansion.

While such a general assumption on the signal provides great applicability, the BBP’s utility dwindles when dealing with applications where one knows information about the signal vector prior to distinguishing or recovering. If, for example, we had knowledge that the spike was sparse, then the BBP certainly holds, but it does so without exploiting the knowledge of its sparsity. Thus, one can consider if there is a different method based on knowledge of the spike sparsity that succeeds below the BBP threshold or perhaps outperforms PCA when $\theta > 1$. Furthermore, if one adopts stricter spike entry assumptions (for instance, a bounded or even atomic prior), then recovery could upgrade to an estimate matching the true entries of the spike vector rather than just a vector with positive correlation with the signal vector (such as the one produced by the top eigenvector). 

\paragraph{Doubly Sparse BBP.}The consideration of these two avenues, along with motivation from a series of real-world applications and natural occurrences of sparsity  \citep[see][for examples]{candes2011robust}, motivated us to adapt the traditional \eqref{spiked_wigner_dense} model and introduce our doubly sparse spiked model (AKA doubly sparse deformed model):
\begin{align}
    X=\frac{\theta}{np}(b\odot v)(b\odot v)^T+\frac{1}{\sqrt{nq}}A\odot W~,\label{model_sparsity_introduction}
\end{align}
where $b$ is a Bernoulli ($p$) vector and $A$ is a symmetric Bernoulli ($q$) matrix. $A,W,b,$ and $v$ are all mutually independent. We delay further specifics of model assumptions until Section \ref{model_section}.   

The doubly sparse setting naturally appears in applications, since sparsity can arise simultaneously from the underlying signal and from the mechanism generating the observed interactions. As we have noted, in high-dimensional data analysis, sparse PCA is used when only a small subset of measured variables contributes to a latent factor, while the data themselves may exhibit sparse dependence structures; applications include genomic data, image analysis, and large-scale network data \citep{guerra2021guide,park2024critical}. Sparsity need not be a property solely of the underlying signal; the observation or noise mechanism may itself be sparse. For example, robust PCA models an observed data matrix as the sum of a low-rank component and a sparse component, motivated by applications such as video surveillance and face recognition in which only a small fraction of entries may be corrupted or contaminated \citep{candes2011robust}. In imaging and inverse problems, sparsity of the underlying signal is frequently combined with incomplete or sparse measurements, as in compressed sensing for magnetic resonance imaging, radar imaging, and other undersampled imaging modalities \citep{lustig2007sparse,harding2013radar}. Similarly, in graph-based applications, the observation or interaction matrix is intrinsically sparse because each vertex interacts with only a small fraction of the population, while a latent community, anomalous subgraph, or other structured signal may itself be supported on a small subset of vertices. This perspective is particularly natural for random graph models, where sparse Erd\H{o}s--R\'enyi adjacency matrices provide a canonical example of sparse random matrices and have well-studied spectral behavior \citep{he2021fluctuations,ganguly2024spectral}. Thus, allowing both the noise matrix and the planted spike to be sparse is not merely a technical extension of the classical spiked Wigner model, but provides a framework for studying problems in which sparsity is present at multiple levels of the data-generating process.

In our recent work \citep{dumitriu2026bbpphasetransitiondoubly}, we proved a BBP phase transition for a sparse, sub-Gaussian spike vector with an unbounded number of entries ($np \to \infty$) planted in sparse, sub-Gaussian noise at supercritical sparsity ($nq\gg \log(n)$). For a more thorough review of other spectral methods as well as a detailed exploration of the results of our work, we direct the reader to the introduction and main results of this companion paper. In the next section, we briefly highlight a series of results that served as motivation for our understanding of where the IT and polynomial-time algorithm thresholds lie for our model. 

\subsection{Impossibility Threshold and Approximate Message Passing (AMP)}\label{amp_contiguity} 
\paragraph{Impossibility.}In search of the information-theoretic threshold, one must consider regions where the problem cannot be solved. One can show the impossibility of distinguishing by proving contiguity, which was first introduced by \cite{lecam1960locally}. This technique relies on a second moment method followed by a standard contiguity argument to show there exists no test statistic that can reliably distinguish between the planted and null distributions. When considering a GUE noise setting and dense ($p=1$) spike, one cannot do better than PCA:  \cite{montanari2014limitationspectralmethodsgaussian} established a limitation for spectral methods in the Gaussian hidden-clique setting, showing that spectral information cannot detect the planted structure below the corresponding spectral threshold, and \cite{perry2016optimalitysuboptimalitypcaspiked} established related results for spiked matrix models with spikes drawn from the spherical, i.i.d. Rademacher, and any other prior with a ``sufficiently sub-Gaussian bound''. In the setting where one departs from Gaussian noise and instead considers a non-Gaussian Wigner matrix with a sufficiently regular continuous noise distribution, \cite{Lesieur2015} and \cite{perry2016optimalitysuboptimalitypcaspiked} showed that PCA can be suboptimal for detection. In particular, one can exploit the non-Gaussianity of the noise by pretransforming the matrix entries according to a function related to the score of the noise distribution (a method motivated by the linearization of AMP, which we discuss more about below), producing a BBP-like transition below the ordinary spectral threshold. These results require regularity assumptions on the noise distribution, including the existence of a suitable density and finite Fisher information.

Returning to the Gaussian noise setting but instead considering a sparse Rademacher spike with fixed mask support $np$ where $p=\Theta(1)$, \cite{Lesieur2015} used statistical physics-based heuristics to provide evidence of distinguishability underneath the spectral threshold in the sparse setting. These heuristics inspired \cite{perry2016optimalitysuboptimalitypcaspiked} and \cite{banks2017informationtheoreticboundsphasetransitions} to rigorously show that the optimality of PCA greatly depends on the value of $p.$ Indeed, when $p\le p^*\approx 0.6,$ the IT threshold falls beneath the spectral threshold as both contiguity results and a successful test statistic based on the likelihood ratio emerge for $\theta<1$. However, this left a gap between their positive and negative results, and, even in the informal consideration of $p\to 0$ (informal in the sense that both of their works do not consider $p$ to be a function of $n,$ and thus this limit is taken outside the context of their work), a $\sqrt{2}$ gap between possibility and contiguity stood. Turning statistical physics heuristics to rigorous mathematical results, \cite{El_Alaoui_2020} closed this gap to $p^*\approx .092$ for sparse Rademacher and even strengthened the result by providing a generalized threshold for any bounded spike prior. 

We emphasize the above papers only consider dense ($q=1$) noise. With the introduction of noise sparsity, their impossibility results no longer apply. In fact, the methods considered in their works completely fall apart in the presence of sparse noise as they rely on a continuous noise distribution to express second moment term of interest in a clear fashion. While this present paper does present impossibility results (see Section \ref{spike_support_regimes_section}), we do not use contiguity as our discrete noise setting creates a difficult second moment term that is hard to handle. Furthermore, we feel the very presence of sparsity limits areas of possible contiguity, which is supported by just how many regions of sparsity allow for distinguishability, as seen in Table \ref{distinguish_table}. 

\paragraph{AMP.} Throughout this introduction, we have often mentioned statistical physics methods, which naturally introduces a different spectral procedure called AMP. See Section \ref{section:AMP_lit_review} for a detailed literature review and discussion of possible future applications to our model.

\subsection{Beyond Spectral Methods and AMP: Thresholding Algorithms}
Following our BBP result, we naturally considered the entire landscape of solvability for our model, which brings us to the central question of this work: 
\vspace{-1.mm}
\begin{center}
    \textit{
    What regimes and models of spike and noise sparsity allow for distinguishing and recovering the spike?}
\end{center}
Under the assumptions $\theta>1,$ $np\to\infty,$ and $nq\gg {\log(n)}$, PCA achieves both distinguishability and weak recovery (correlation). This leaves the question: what occurs outside these assumptions? For example, for a noise matrix where the spectral bulk is not controlled ($q \lesssim\frac{\log(n)}{n}),$ can we still distinguish and recover? Or, given a spike with bounded support ($np\not\to\infty$), are there values of $p$ where the signal can still be recovered? Or, even if we are within the sparsity regions of the BBP, can we use other methods under the spectral threshold ($\theta <1$) for signal recovery? 

Here, we shed some light on this central question, along with all the accompanying inquiries, through the introduction of two different thresholding algorithms. 
\subsubsection{Traditional Thresholding}\label{thresholding_review}
The first method we explore in this work involves thresholding on the sizes of the entries in the observed matrix. Due to the scaling of the spike and noise terms in \eqref{model_sparsity_introduction}, which ensure the two terms are spectrally comparable, the asymptotic behavior and size of the entries of the spike and noise can vary quite drastically depending on their different sparsities, and we can exploit this difference when attempting to detect or extract the spike vector. The notion of thresholding spiked models is not a novel concept. In its inception, \cite{Johnstone_diag_threshold} first proposed a diagonal thresholding procedure to recover in the Spiked Wishart model where the spike vector $v$ is sparse in a suitable basis (specifically, the wavelet domain, meaning $v$ belongs to a weak $l_p$-ball where $p\in(0,1)$). \cite{Amini_2009} further restricted assumptions on the spike by assuming each entry had equal magnitude of $\frac{1}{\sqrt{np}}$ (i.e. support, but not sign, recovery) and established exact recovery under the sparsity condition $p\lesssim \frac{1}{\sqrt{n\log n}}.$ In pursuit of the polynomial time algorithm threshold for recovery, a series of papers worked toward outperforming diagonal thresholding \citep{moghaddam2005spectral, zou2006sparse, daspremont2007direct, daspremont2008optimal, witten2009penalized}, and, through a covariance thresholding method, \cite{Krauthgamer_2015, deshpande2016sparsepcacovariancethresholding} extended the sparsity to $p\lesssim\frac{1}{\sqrt{n}}$ for signals satisfying a lower bound assumption (and further conditions on basis vectors overlap, see A2 in their work). As previously hinted, the polynomial time algorithm threshold does not necessarily align with the IT threshold; indeed, exhaustively searching over all possible support sets of size $np$ results in exact support recovery when $p\lesssim \frac{1}{\log n}$ \citep{Amini_2009}. So, what is possible in this in-between region $\frac{1}{\sqrt{n}}\lesssim p \lesssim \frac{1}{\log n}$? Some theorize a ``possible yet hard'' regime, which is supported by reducing the problem \citep{berthet2013computational, gao2017sparse, brennan2021} to the widely-conjectured possible-yet-hard \citep{ barak2019nearly, deshpande2015improved} \textit{planted clique problem} \cite[introduced in][where the authors sought to identify a clique hidden in an Erd\H{o}s--R\'enyi graph $\mathcal G(n,1/2)$.]{jerrum1992large, alon1998finding}.

This very question of balancing sparsity with runtime brought \cite{ding2022subexponentialtimealgorithmssparsepca} to formulate a thresholding procedure for both the Spiked Wishart and the Spiked Wigner models (with spike satisfying Assumption \ref{bounded_spike_assumption} and GOE noise) that seesaws between the two parameters: a sparser signal allows for a shorter runtime (i.e. polynomial time), and a denser signal leads to longer, sub-exponential runtimes. 

\paragraph{Our Contribution.} Here, we extend the work of \cite{ding2022subexponentialtimealgorithmssparsepca} by creating an algorithm that thresholds on a quadratic form of our doubly sparse deformed model. Due to the fact that their proof procedures fall apart for our noise set up (sub-Gaussian and sparse), we present a different method, based on the leave-one-out technique that is broader and extends to our model. Similar to their findings, we see a payoff between runtime and signal sparsity, but also with an added parameter of noise sparsity.

\paragraph{Related Works on Doubly Sparse Models.} Following the efforts of  \cite{ding2022subexponentialtimealgorithmssparsepca}, \cite{dOrsiNNS22} made progress in two directions (in relation to this current work). First, unlike the sparsity-runtime tradeoff for exact recovery of \cite{ding2022subexponentialtimealgorithmssparsepca}, \cite{dOrsiNNS22} relaxed the task to weak recovery to produce a fixed quasi-polynomial time algorithm. Furthermore, their work sought to relax the constraints upon the noise assumptions by considering heavier tailed noise. In this heavier tailed setting, \cite{novikov23a} followed up by providing improved algorithmic runtime, which interpolates between polynomial-time and quasi-polynomial time algorithms. 

The heavier tailed setting of \cite{dOrsiNNS22} and \cite{novikov23a} and our doubly sparse setting do overlap in some regimes, but are generally complementary. Neither their works nor ours can capture the totality of the respective model settings and results. To summarize, they show that the Huber-loss minimizer achieves weak/strong recovery with a signal following an upper bounded prior corrupted by heavy tailed, symmetric noise. In certain overlapping regimes, our methods are better (faster or leading to better recovery), while theirs are similarly better in others. 

\subsubsection{Identical Search: Unique to Sparse, Atomic Noise}
Our second thresholding algorithm sits in a unique space within the literature of thresholding procedures. At its core, it is an algorithm based on thresholding. Yet, the power of the algorithm lies within the exploitation of the noise itself (or lack thereof) rather than the values of the noise entries. While previous thresholding procedures exploit the variation of the size of the entries between the spike and the noise, our algorithm instead balances the atomic prior of the spike with the dwindling presence of the noise. More specifically, we explore the regime of sparsity where $q$ is low enough so that most of the spike entries are revealed. Although the algorithm works for any atomic spike distribution, we adopt the commonly considered sparse Rademacher distribution, which, as we previously mentioned, has been widely considered within the spiked model literature \citep{perry2016optimalitysuboptimalitypcaspiked, banks2017informationtheoreticboundsphasetransitions, Amini_2009, El_Alaoui_2020}.

\subsection{Notation}
We use $\|\cdot\|_2$ and $\| \cdot\|$ interchangeably as the operator norm for the matrix. $\|\cdot\|_0$ represents the size of the support of a vector, most commonly used when referencing the number of nonzero components in a vector (i.e $\|v\|_0=|\text{supp}(v)|$).   $\|X\|_{\psi_l}$ represents the sub-Gaussian ($l=2$) or sub-exponential ($l=1$) norm $$\|X\|_{\psi_l}=\inf\{K>0 : \mathbb E \exp(|X|^l/K^l)\le 2 \}~.$$
In descending order (i.e. $\lambda_1\ge\lambda_2\ge\cdots\ge\lambda_n$), we use $(\lambda_i(X), u_i(X))$ to represent the top $i$-th normalized leading eigenpair of a matrix $X$.

We adopt standard asymptotic notations $O(\cdot), \Theta(\cdot),$ and $ \Omega(\cdot)$. And, with these asymptotic notations, we use $A\lesssim B$ if $A\le CB$ for some absolute constant $C$ and $A \ll B$ if $A/B=o(1)$. The $\odot$ represents the Hadamard product for vectors and matrices: $(Y\odot Z)_{ij}=Y_{ij}Z_{ij}$, for $i,j\in [n]$. $X^{(i)}$ represents the matrix $X$ with the $i$-th row and column replaced by zeros.

\section{Model Setup}\label{model_section}

To formally define the model \eqref{model}, we first define its components. 

\subsection{The Spike}

We begin by considering the underlying sparse signal. We denote the SNR $\theta>0$ paired with the spike vector $v=\tilde{v} \odot b\in\mathbb{R}^n$ where $b ~\iidsim~\mathcal B$ and $\tilde{v}~\iidsim~ \tilde{\mathcal V}_i$ for $i\in[3]$, are mutually independent. Here, $\tilde{\mathcal V}_i$ represents the dense prior, and $\mathcal B$ represents our method of sparsity sampling: 
\begin{assumption}[Small-Ball Condition]\label{spike_small_ball_assumption}
We have $\tilde v\overset{\mathrm{i.i.d.}}{\sim}\tilde{\mathcal V}_1$ if there exists a constant $C_s>0$ such that $\mathbb P(|\tilde v_j|\leq\epsilon)\leq C_s\epsilon$ for all $\epsilon>0$ and $j\in[n]$.
\end{assumption}

We use this small ball assumption in Sections \ref{quadratic_thresholding_section} and \ref{section:DT} to ensure we have no atom at $0$. We can improve our sparsity windows by focusing on the subset of this assumption where the entries are almost surely lower bounded.

\begin{assumption}[Lower Bounded Vectors]\label{bounded_spike_assumption}
    We have $\tilde v ~\iidsim ~ \tilde{\mathcal V}_2$ if $|\tilde v_i|\ge\frac{1}{B}$ almost surely, for some $B>0$. 
\end{assumption}

For comparability purposes with the PCA results of \cite{dumitriu2026bbpphasetransitiondoubly}, we introduce the following sub-Gaussian spike assumption. 

\begin{assumption}[Sub-Gaussian Vectors]\label{assumption:spike_subG}
    We have $\tilde v\overset{\mathrm{i.i.d.}}{\sim}\tilde{\mathcal V}_3$ if the following hold. For all $j\in[n]$, the entries $\tilde v_j$ satisfy $\|\tilde v_j\|_{\psi_2}\leq K_1$ and $\mathbb E|\tilde v_j|^t=O(t)^{t/2}$ as $t\to\infty$. Furthermore, $\mathbb E[\tilde v_j]=0$ and $\mathbb E[\tilde v_j^2]=1$ for all $j\in[n]$. 
\end{assumption}

Many natural sub-Gaussian distributions follow Assumptions \ref{spike_small_ball_assumption} and \ref{assumption:spike_subG}, such as Gaussian, uniform, and Rademacher. In the case of Rademacher, we establish a final separate condition, as we can often exploit its atomic nature to improve bounds and solvability.

\begin{assumption}[Rademacher Vectors]\label{rademacher_assumption}
    We have $\tilde v ~\iidsim ~ \tilde{\mathcal V}_4$ if $\mathbb P(\tilde v_j = 1)=\mathbb P(\tilde v_j =-1)=1/2$ for all $j\in [n]$. Therefore, $\mathbb E[\tilde v_j]=0$ and $\mathbb E [\tilde v_j^2]=1$ for all $j\in[n]$.
\end{assumption}

We choose \{$\pm1$\} entries in Assumption \ref{rademacher_assumption} for the sake of simplicity and comparability with other works, as this is a standard spike prior in deformed model literature. In reality, we conjecture any centered, atomic prior can be implemented for the spike prior $\tilde{\mathcal V_3}$. Furthermore, notice how Assumption \ref{rademacher_assumption} is a subset of all three of the preceding assumptions.

\begin{assumption}[Independent Sparsity]\label{ind_bern_spike_assumption}
    When $b~\iidsim~ \mathcal B,$ we consider the natural Bernoulli sparsity with parameter $p$, i.e., $b_j ~\iidsim ~ \text{Ber}(p)$ for all $j\in [n]$.
\end{assumption}

Under the above assumption, $\|v\|_0$ can vary depending on $p,$ as it is $np$ only in expectation. All of our results will implement a conditioning technique where we fix $\|v\|_0$ prior to performing our analysis. We will do so with the techniques of Section \ref{spike_support_regimes_section}.

\subsection{The Noise}

Next, we consider the noise. Depending on the algorithm, we have two different noise regimes to consider. We begin with a dense, real Wigner matrix $W \in \mathbb R^{n\times n}$ satisfying either of the following two assumptions:
\begin{assumption}[Continuous Noise Distribution]\label{continuous_noise_dist_assumption}
    For all $\{i,j\}\in[n]^2,$ W is symmetric with independent entries $W_{ij}\sim \mathcal W_{ij}$ where $\mathcal W_{ij}$ is a centered, continuous  distribution.
\end{assumption}

We use the above to differentiate the noise from the atomic spike in Section \ref{section:identical_search}.
 
\begin{assumption}[Sub-Gaussian Wigner Matrix]\label{wigner_subg_assumption}
    For all $\{i,j \}\in [n]^2,$ $W$ is symmetric with independent, centered $(\mathbb E[W_{ii}]=\mathbb E [W_{ij}]=0)$, finite variance ($\mathbb E[W_{ii}^2]=2$ and $\mathbb E[W_{ij}^2]=1$), sub-Gaussian upper triangular entries. 
     Moreover, $\sup_{i\le j}\|W_{ij}\|_{\psi_2}\le K_2$ where $K_2\ge1$.
\end{assumption} 

When comparing results, we also control the log-Laplace transform of the entries from Assumption 3 in \cite{dumitriu2026bbpphasetransitiondoubly}; however, the results of this current work only require the above. See Lemma \ref{noise_quadratic_form_control} for its implementation.

Next, we introduce the noise mask matrix $A \in \mathbb R^{n\times n}$ satisfying the following assumption: 
\begin{assumption}[Bernoulli Noise Mask]\label{noise_mask_assumption}
    $A$ is a symmetric matrix with independent upper-triangular entries $A_{ij}=A_{ji} \sim$ Ber$(q)$ for all $\{i,j\}\in[n]^2$.
\end{assumption}
\noindent Therefore, the unscaled noise matrix is the Hadamard product matrix $W\odot A$ where $W$ and $A$ are mutually independent. 

\subsection{The Model}

Therefore, we have the following real $n\times n$ matrix representing our \textit{doubly sparse spiked Wigner model}\footnote{When $p=0$ or $q=0,$ consider the signal/noise term equal to $0,$ respectively, to avoid division by $0$.}: 
\begin{align}\label{model}
     X=\frac{\theta}{np}vv^T+\frac{1}{\sqrt{nq}}W\odot A~.
\end{align}
We sometimes use $Y:=\frac{1}{\sqrt{nq}}W\odot A.$ The above normalizations are natural as they enforce $\|\frac{1}{np}vv^T\|_2\xrightarrow[]{\P}\Theta(1)$ and $\|\frac{1}{\sqrt{nq}}W\odot A\|_2\xrightarrow[]{\P}\Theta(1)$ in most cases, thus ensuring spectral comparability between the spike and the noise, while allowing $\theta=\Theta(1)$. If the spike and noise have different spectral magnitudes (for example, such cases are present in the literature with the normalization factors subsumed into the vector $v$ and $\theta$ is sometimes a function of $n$), we could account for the difference by scaling $\theta$ or $v$ accordingly. 

\subsection{The Tasks}\label{distinguishrecoverdefn}

Formally, we define our first task of distinguishing\footnote{The nomenclature of this task in similar literature varies based on discipline. In many papers, it can also be referred to as Strong Detection. However, because detection is often associated with assuming a spike is present and determining whether or not it is detectable (i.e. just the type I error), we make the distinction by using the term distinguishability to refer to the entire statistical test (control over both types of error).} between the planted $\mathbb P$ and null $\mathbb Q$ models as follows: 
\paragraph{Distinguishability. }There exists a test statistic $\mathcal T$ such that $\mathbb P(\mathcal T(X)<\epsilon)+\mathbb Q(\mathcal T(X)\ge \epsilon)\to 0$ as $n\to \infty,$ i.e. the type I and type II errors converge to $0.$

\vspace{.25cm}

Unlike Distinguishability where we are unaware of which distribution was used for the sample, we also consider the task of recovering the planted spike knowing we are in the planted distribution, i.e. $X\sim\mathbb P.$ Depending on knowledge of the spike prior, we can sometimes achieve recovery with higher accuracy than the correlation inner product \citep[like a general BBP, i.e.][]{dumitriu2026bbpphasetransitiondoubly}. Specifically, for vectors $v_1$ and $v_2$, we introduce $$\mathcal R(v_1,v_2):= \mathbbm 1_{\text{supp}(v_1)=\text{supp}(v_2)}\max_{\omega\in\{\pm 1\} }\sum_{i=1}^{n}\mathbbm 1_{\{(v_1)_i=\omega\cdot (v_2)_i\}}~,$$
which accounts for entrywise comparison of our estimate with the spike. Thus, we introduce the varying levels of recovery: 
\paragraph{Recovery. } Given an estimator $\mathcal M,$ we can achieve 
\begin{itemize}
    \item \textbf{Weak Recovery} if $\frac{|\langle\mathcal M(X), v\rangle|}{\|\mathcal M(X)\|_2\cdot\|v\|_2}>0$;
    \item \textbf{Strong Recovery} if $\frac{|\langle\mathcal M(X), v\rangle|}{\|\mathcal M(X)\|_2\cdot\|v\|_2}\to1$;
    \item \textbf{Exact Recovery} if $\mathcal R(\mathcal M(X),v)=1$
\end{itemize}
with probability tending to $1$ as $n\to \infty$. 

\section{Main Results}
For comparability purposes, we consider Model \eqref{model} with noise drawn from a sparse, continuous sub-Gaussian distribution (Assumptions \ref{continuous_noise_dist_assumption}, \ref{wigner_subg_assumption}, and \ref{noise_mask_assumption}); however, the first three methods (see below: PCA, QT, DT) hold for Assumptions \ref{wigner_subg_assumption} and \ref{noise_mask_assumption} and the last method (see below: IS) holds for Assumptions \ref{continuous_noise_dist_assumption} and \ref{noise_mask_assumption}. Furthermore, the results of \cite{dumitriu2026bbpphasetransitiondoubly} hold for a sub-Gaussian spike (Assumption \ref{assumption:spike_subG}) without the other assumptions, but we group the assumptions together for comparability and applicability to this work's algorithms.

\subsection{Distinguishability}

\begin{fact}[Theorem \ref{distinguish_them_np_equals_constant} Restatement]
    When $np\not\to\infty,$ distinguishing is impossible. 
\end{fact}

\begin{remark}
    When $np=\Theta(1),$ there is always a nonzero probability that $\|v\|_0=0,$ meaning distinguishing is impossible. However, if one instead samples from a fixed mask distribution, i.e., where the spike support size is fixed and positive, then all of the following distinguishability results below will hold (with a few minor exceptions, which we detail in their corresponding sections). Because of these exceptions and for ease of readability, we exclude this case and defer the specificities to the results in the main body of work. 
\end{remark}


\paragraph{Distinguishing Algorithms (Informal)} Given observation $X$, one distinguishes using
\begin{itemize}
    \item \textit{Principal Component Analysis} (PCA) by thresholding  $\lambda_1(X)$ by some value $t$ corresponding to $\lambda_1(Y).$
    \item \textit{Quadratic Thresholding} (Algorithm \ref{quad_form_distinguish_alg}) by choosing an appropriate $l\in\mathbb Z^{\ge0}$ inside an interval determined by sparsity conditions $p$ and $q$, searching for the Rademacher vector $w^*$ of support size $l$ that maximizes the quadratic form $w^TXw,$ and thresholding ${w^*}^{T} Xw^{*}$. 
    \item \textit{Diagonal Thresholding} (Algorithm \ref{alg:DT_dist}) by thresholding $\max_i|X_{ii}|$.
    \item \textit{Identical Search} (Algorithms \ref{identical_search_distinguish_aware} and \ref{identical_search_distinguish_blind}) by searching\footnote{Here we informally describe the spike full knowledge case where we are aware $\lambda/p.$ For details on the spike agnostic case, see Algorithm \ref{identical_search_distinguish_blind}.} the matrix for an entry equal to $\pm\frac{\theta}{np
    }$. If there exists such an entry, conclude the distribution is planted, and, if not, conclude the distribution is null. 
\end{itemize}

We move from the strictest sparsity assumption to more general priors. The following summarizes Theorems \ref{thm:quad_subG_distinguish}, \ref{quad_form_distinguish_thm}, \ref{distinguish_gaussian_DT}, and \ref{iden_search_distinguish} with Corollary 5 of \cite{dumitriu2026bbpphasetransitiondoubly}. Refer to these results for tail bounds and details on positive constants $C_1, C_2,$ and $C_5.$

\begin{theorem}[Distinguishing when $np\to\infty$]\label{main_distinguish_thm}
    Assume $np\to\infty$. 
    \begin{enumerate}
        \item If $v$ is drawn from the \textbf{Sparse Rademacher} prior (Assumptions \ref{rademacher_assumption} and \ref{ind_bern_spike_assumption}), then
        \begin{enumerate}
            \item PCA distinguishes when $\theta>1$ and $q\gg\frac{\log(n)}{n}~;$
        \item Quadratic Thresholding distinguishes when $q\ge\frac{C_1K_2^2}{\theta^2}p\log(n)$;
        \item Diagonal  Thresholding distinguishes when $q\gtrsim \frac{4}{\theta^2 K_2^2}np^2 \log(nq)$.
        \item Identical Search distinguishes when $1-q\gg\frac{1}{(np)^2}$;
        \end{enumerate}

        \item If $v$ is drawn from the \textbf{Sparse, Lower Bounded} prior (Assumptions \ref{bounded_spike_assumption} and \ref{ind_bern_spike_assumption}), then
        \begin{enumerate}
            \item PCA, with additional Assumption \ref{assumption:spike_subG}, distinguishes when $\theta>1$ and $q\gg\frac{\log(n)}{n}~;$
            \item Quadratic Thresholding distinguishes when $q\ge\frac{C_1B^4K_2^2}{\theta^2}p\log(n)$.
            \item Diagonal  Thresholding distinguishes when $q\gtrsim \frac{4B^4}{\theta^2 K_2^2}np^2 \log(nq)$.
        \end{enumerate}

        \item If $v$ is drawn from the \textbf{Sparse, Small-Ball} prior (Assumptions \ref{spike_small_ball_assumption} and \ref{ind_bern_spike_assumption}), then
        \begin{enumerate}
            \item PCA, with additional Assumption \ref{assumption:spike_subG}, distinguishes when $\theta>1$ and $q\gg\frac{\log(n)}{n}~;$
            \item Quadratic Thresholding distinguishes when $q\ge \frac{C_5K_2^2C_s^4}{\theta^2}p\log(n)$~.
            \item Diagonal  Thresholding distinguishes when $q\gg np^2\log(nq)$.
        \end{enumerate}
    \end{enumerate}
\end{theorem}

\begin{remark}
    The lower bounded spike prior is not a subset of the sub-Gaussian prior as no conditions are placed on the upper bounds of the entries or their means or variances. Thus, in order to achieve PCA in a lower bounded prior, one must also assume sufficient sub-Gaussian properties. 
\end{remark}

\subsection{Recovery}\label{section:main_results_recovery}

\paragraph{Recovery Algorithms (Informal)} Given $X$, one recovers an estimate of $v$ using
\begin{itemize}
    \item \textit{Principal Component Analysis} (PCA) by computing the eigenvector $v_1(X)$ corresponding to $\lambda_1(X).$
    \item \textit{Quadratic Thresholding} (Algorithms \ref{quad_recovery_alg_recov}, \ref{sign_recovery_alg}, \& \ref{alg:QT_weak_strong}) by choosing $l\in\mathbb Z^{\ge0}$ from an interval dictated by $p$ and $q$, searching for the Rademacher vector $w^*$ of support size $l$ that maximizes $w^TX^{(i)}w$, thresholding $(Xw^*)_i$, and repeating this process for all $i\in[n]$. 
    \item \textit{Diagonal Thresholding} (Algorithms \ref{alg:DT_recov} \& \ref{alg:QT_weak_strong}) by thresholding the diagonal entries of $X$.
    \item \textit{Identical Search} (Algorithm \ref{identical_search_recovery_alg}) by considering the graph with adjacency matrix $X$, searching for an edge equal to $\pm\frac{\theta}{np}$, and walking from vertex to vertex by crossing only edges with value $\pm\frac{\theta}{np}$.
\end{itemize}

Again, we present from strict to more general spike prior. The following summarizes Theorems \ref{thm:QT_subG_support_recovery}, \ref{quad_form_recovery_thm}, \ref{iden_search_recov_thm_np_to_inf}, and \ref{iden_search_recovery_thm_const_np} with Corollaries \ref{corollary:QT_exact}, \ref{quad_form_weak_strong_corollary}, \ref{cor:DT_recovery} and 6 \cite{dumitriu2026bbpphasetransitiondoubly}. See each result for tail bounds and details on positive constants $C_1, C_2, C_5,$ and $C_8.$ 

\begin{theorem}[Recovery when $np\to\infty$]\label{main_recov_thm_unbounded} Assume $np\to\infty$. Choose $\tau_4\to\infty$ and $\tau_6=\omega(np)$ with $\tau_6=o\left(\frac{\sqrt{np}}{\log(np)} \right)$. 
    
    \begin{enumerate}
        \item If $v$ is drawn from the \textbf{Sparse Rademacher} prior (Assumptions \ref{rademacher_assumption} and \ref{ind_bern_spike_assumption}), then
        \begin{enumerate}
            \item PCA achieves Weak Recovery when $\theta>1$ and $q\gg\frac{\log(n)}{n}~;$
            \item Quadratic Thresholding achieves Exact Recovery when $q\ge\frac{2C_8K_2^2}{\theta^2}p\log(n)$;
            \item Diagonal  Thresholding achieves Strong Recovery when $q\gtrsim \frac{4}{\theta^2 K_2^2}np^2 \log(nq)$.
            \item Identical Search achieves Exact Recovery when $q\le1-\frac{\log(np)+\tau_6}{np}$.
        \end{enumerate}

        \item If $v$ is drawn from the \textbf{Sparse, Lower Bounded} prior (Assumptions \ref{bounded_spike_assumption} and \ref{ind_bern_spike_assumption}), then
        \begin{enumerate}
            \item PCA, with additional Assumption \ref{spike_small_ball_assumption}, achieves Weak Recovery when $\theta>1$ and $q\gg\frac{\log(n)}{n}~;$
            \item Quadratic Thresholding achieves \text{Strong Recovery} when $q\ge\frac{2C_2B^4K_2^2}{\theta^2}p\log(n)$.
            \item Diagonal Thresholding achieves \text{Strong Recovery} when $q\gtrsim \frac{4B^4}{\theta^2 K_2^2}np^2 \log(nq)$.
        \end{enumerate}

        \item If $v$ is drawn from the \textbf{Sparse, Small-Ball} prior (Assumptions \ref{spike_small_ball_assumption} and \ref{ind_bern_spike_assumption}), then
        \begin{enumerate}
            \item PCA, with additional Assumption \ref{spike_small_ball_assumption}, achieves Weak Recovery when $\theta>1$ and $q\gg\frac{\log(n)}{n}~;$
            \item Quadratic Thresholding achieves Strong Recovery when $q\ge \frac{2C_7K_2^2C_s^2}{\theta^2}\cdot\tau_4^2(np)^2p\log(n)$~.
            \item Diagonal  Thresholding achieves Strong Recovery when $q\gg (np)^5p\log(nq)$.
        \end{enumerate}
    \end{enumerate}
\end{theorem}

\begin{theorem}[Recovery when $np=\Theta(1)$]\label{main_recov_thm_bounded} Assume $np=\Theta(1)$ and $\|v\|_0\in\mathbb N$. 

    \begin{enumerate}
        \item If $v$ is drawn from the \textbf{Sparse Rademacher} prior (Assumptions \ref{rademacher_assumption} and \ref{ind_bern_spike_assumption}), then
        \begin{enumerate}
            \item Quadratic Thresholding achieves Exact Recovery when $q\ge\frac{C_8K_2^2}{\theta^2}\cdot\frac{np^2\log(n)}{\|v\|_0}$;
            \item Identical Search achieves Exact Recovery when $q\ll1$;
            \item Diagonal  Thresholding achieves Strong Recovery when $q\gtrsim \frac{4B^4}{\theta^2 K_2^2}np^2 \log(nq)$.
        \end{enumerate}

        \item If $v$ is drawn from the \textbf{Sparse, Lower Bounded} prior (Assumptions \ref{bounded_spike_assumption} and \ref{ind_bern_spike_assumption}), then
        \begin{enumerate}
            \item Quadratic Thresholding achieves Strong Recovery when $q\ge\frac{C_2B^4K_2^2}{\theta^2}\cdot\frac{np^2\log(n)}{\|v\|_0}$.
            \item Diagonal  Thresholding achieves Strong Recovery when $q\gtrsim \frac{4B^4}{\theta^2 K_2^2}np^2 \log(nq)$.
        \end{enumerate}

        \item If $v$ is drawn from the \textbf{Sparse, Small-Ball} prior (Assumptions \ref{spike_small_ball_assumption} and \ref{ind_bern_spike_assumption}), then
        \begin{enumerate}
            \item Diagonal Thresholding achieves Strong Recovery when $q\gg p\log(nq)$~.
        \end{enumerate}
    \end{enumerate}
\end{theorem}
\begin{remark}
    When $np\not\to\infty,$ PCA and small-ball QT are no longer guaranteed. 
\end{remark}

\subsection{Interpretation of Main Results: Comparing Spike Assumption}

As seen above, each algorithm succeeds under different spike assumptions. And, due to these different assumptions, varying levels of recovery can be achieved: stricter assumptions on the spike allow for higher precision in recovery. By the nature of its construction, Identical Search (IS) distinguishes and exactly recovers only in the strictest Rademacher spike prior (Assumption \ref{rademacher_assumption}). When distinguishing and strongly recovering, Quadratic Thresholding (QT) and Diagonal Thresholding (DT) allow for a more general assumption: lower bounded (Assumption \ref{bounded_spike_assumption}) or small-ball (Assumption \ref{spike_small_ball_assumption}); however, exact recovery is only available in the Rademacher prior assumption for QT. Finally, distinguishing and recovering with PCA in  \cite{dumitriu2026bbpphasetransitiondoubly}, the spike can be drawn from any centered, sub-Gaussian prior (which naturally includes the Rademacher prior), but the spectral method is limited to weak recovery.
We summarize these results in the following table\footnote{Similar results (Theorem \ref{main_recov_thm_bounded}) hold for $np=\Theta(1),$ but without the "Distinguish" row and for different sparsity assumptions.}:

\begin{table}[H]
    \centering
    \renewcommand{\arraystretch}{1.1} 
\begin{tabular}{|c|c|c|c|}
\hline
$np\to\infty$ & Small-Ball & Rademacher & Lower Bounded  \\
 & Sub-Gaussian & & \\ \hline

Distinguishability & \multicolumn{2}{|c|}{ \cellcolor{backgroundOrange}PCA ($\theta>1, q\gg\log(n)/n$)}  &   \\ 
 & \multicolumn{3}{|c|}{ \cellcolor{cyan}Diagonal Thresholding}     \\ 
 & \cellcolor{cyan} ($q\gg np^2 \log(nq)$) & \multicolumn{2}{|c|}{\cellcolor{cyan} ($q\gtrsim np^2 \log(nq)$) }   \\
&  \multicolumn{3}{|c|}{\cellcolor{CarnationPink}Quadratic Thresholding ($q\gtrsim p\log(n)$)} \\
 &  & \cellcolor{Apricot}Identical Search ($1-q\gtrsim1/(np)^2$) & \\ \hline
Weak Recovery & \multicolumn{2}{|c|}{ \cellcolor{backgroundOrange}PCA ($\theta>1, q\gg\log(n)/n$)} &   \\ 
  & \multicolumn{3}{|c|}{ \cellcolor{cyan}Diagonal Thresholding}     \\ 
 & \cellcolor{cyan} ($q\gg (np)^5p \log(nq)$) & \multicolumn{2}{|c|}{\cellcolor{cyan} ($q\gtrsim np^2 \log(nq)$) }   \\
&  \multicolumn{3}{|c|}{\cellcolor{CarnationPink}Quadratic Thresholding}  \\
& ($q\gg (np)^2p\log(n)$)\cellcolor{CarnationPink} & ($q\gtrsim p\log(n)$) \cellcolor{CarnationPink} & \cellcolor{CarnationPink} ($q\gtrsim p\log(n)$)  \\
 &  & \cellcolor{Apricot}Identical Search ($1-q\gtrsim \log(np)/(np)$&  \\ \hline
Strong Recovery & \multicolumn{3}{|c|}{ \cellcolor{cyan}Diagonal Thresholding}     \\ 
 & \cellcolor{cyan} ($q\gg (np)^5p \log(nq)$)) & \multicolumn{2}{|c|}{\cellcolor{cyan} ($q\gtrsim np^2 \log(nq)$) }   \\
&  \multicolumn{3}{|c|}{\cellcolor{CarnationPink}Quadratic Thresholding}  \\
& ($q\gg (np)^2p\log(n)$)\cellcolor{CarnationPink} & ($q\gtrsim p\log(n)$) \cellcolor{CarnationPink} & \cellcolor{CarnationPink} ($q\gtrsim p\log(n)$)  \\
 &  & \cellcolor{Apricot}Identical Search ($1-q\gtrsim \log(np)/(np)$) &  \\ \hline
Exact Recovery & & \cellcolor{CarnationPink}Quadratic Thresholding ($q\gtrsim p\log(n)$)  &   \\ 
 &  & \cellcolor{Apricot}Identical Search ($1-q\gtrsim \log(np)/(np)$& \\ \hline
\end{tabular}
    \caption{Distinguishing and Recovering under varying spike assumptions for Model \ref{model} under $np\to\infty$. For spike assumptions, see Assumptions \ref{assumption:spike_subG} and \ref{spike_small_ball_assumption}, \ref{rademacher_assumption}, and \ref{bounded_spike_assumption}, respectively. All four use Assumption \ref{ind_bern_spike_assumption}. The noise follows Assumption \ref{wigner_subg_assumption} for PCA, DT, and QT, and Assumption \ref{continuous_noise_dist_assumption} for IS. All four use Assumption \ref{noise_mask_assumption}.}
    \label{tab:placeholder}
\end{table}

\subsection{Distinguishability and Recovery: The Full Picture}\label{section:table_descriptions} In Tables \ref{distinguish_table} and \ref{recovery_table}, we provide a landscape for distinguishing and recovering. We describe and interpret the contents of the table below.

\paragraph{Impossibility.} First, because $np\to0$ means $\|v\|_0=0$ with high probability, distinguishing is impossible and recovery is uninteresting\footnote{QT, DT, and IS will trivially return $0$.}. When $np=\Theta(1),$ the nonzero probability of $\|v\|_0=0$ rules out distinguishability, but recovery is still possible. We represent these impossibilities/trivialities by blacking/graying out the regions on the tables. For further details, we refer the reader to Section \ref{spike_support_regimes_section}. 

\paragraph{Identical Search.} The sparsity regimes where Identical Search succeeds differ between distinguishing and recovering. When $np\to\infty,$ one can distinguish at the higher level of sparsity $q\lesssim1-1/(np)^2,$ whereas exact recovery requires $q\lesssim 1-\log(np)/np.$ When $np=\Theta(1),$ distinguishing is no longer possible, and exact recovery requires the much stricter assumption of $q\ll1.$ For further details, see Section \ref{section:identical_search}.

Therefore, in Table \ref{distinguish_table}/Table \ref{recovery_table}, \textcolor{safePurple}{$\textbf{IS}$} represents the Identical Search algorithm successfully Distinguishing/Exactly Recovering, respectively, for a Rademacher spike and a continuous noise distribution. 

\paragraph{Quadratic Thresholding.} The sparsity guarantees of Quadratic Thresholding depend on whether the spike prior is lower bounded or small-ball, and both priors display a tradeoff between runtime and sparsity. 

When distinguishing, both priors are asymptotically equivalent (as they differ only by constants). Specifically, both priors run in polynomial time when $q\gtrsim np^2\log(n)$ and in general when $q\gtrsim p \log n$.

When recovering, the lower bounded prior asymptotically matches the above distinguishability results (as they differ only by constants). However, in the small-ball prior, the sparsity regimes shrink. First, $np\to\infty$ is required, and, given $np\to \infty$, QT succeeds in polynomial time when $q\gg (np)^3p \log(n)$ and in general when $q\gg(np)^2p\log(n)$.

The question of whether a polynomial time algorithm exists in these latter regimes has been explored, which we review in Section \ref{thresholding_review}. Note how, in our algorithm, greater noise sparsity $(q)$ makes it more challenging to achieve a polynomial-time algorithm. For further details, see Section \ref{quadratic_thresholding_section}.

For the sake of presentation, we consider only the lower bounded prior in the below tables. For the full, extended versions, we refer the reader to Section \ref{section:tables_extended}.

Therefore, in Table \ref{distinguish_table}/Table \ref{recovery_table}, \textcolor{safePink}{$\textbf{QT}_p$} represents the Quadratic Thresholding polynomial-time algorithm achieving Distinguishability/Exact Recovery for a Rademacher spike and Distinguishability/Strong Recovery, respectively, for a lower bounded spike planted in sub-Gaussian noise. \textcolor{safePink}{$\textbf{QT}$} represents the same, but the algorithm runs in non-polynomial time.

\paragraph{Diagonal Thresholding.} Similar to Quadratic Thresholding, the sparsity guarantees of Diagonal Thresholding depend on whether the spike prior is lower bounded or small-ball. However, the algorithm runs in linear time in all regions where it succeeds and for all priors. 

For the lower bounded prior, DT is asymptotically equivalent across both tables, requiring $q\gtrsim np^2 \log(nq).$

For a small-ball prior, there is an asymptotic difference between distinguishing and recovering. When distinguishing, DT succeeds in linear time when $q\gg np^2 \log(nq)$. For weak/strong recovery, $q\gg (np)^5p\log(nq)$ results in a linear runtime.

Similar to above, we include only the details for the lower bounded prior, and refer the reader to Section \ref{section:tables_extended} for the full, extensive results. Because of the $\log(nq)$ term, comparing other methods with this implicit relationship proves troublesome. While DT ($q\gtrsim np^2\log(nq)$) always improves on polynomial-time QT ($q\gtrsim np^2 \log(n)$), there is a more delicate balance when it comes to beating non-polynomial time QT $(q\gtrsim p \log(n)$). Thus, in both tables, we use ``/'' to notate these regions where the success of QT versus DT depends on the relationship between $p$ and $q$. We refer the reader to Section \ref{section:DT} for a more accurate analysis between QT and DT. In short, as $q$ becomes more sparse, DT can prove superior to QT.

Therefore, in Table \ref{distinguish_table}/Table \ref{recovery_table}, \textcolor{safeBlue}{$\textbf{DT}_l$} represents the Diagonal Thresholding linear-time algorithm achieving Distinguishability/Strong Recovery, respectively, for a small-ball spike planted in sub-Gaussian noise. 

\paragraph{PCA.} Assuming $\theta>1$, the sparsity regimes where PCA succeeds are equivalent across both tables, as both methods require $np\to\infty$. If $\theta<1,$ PCA vanishes from both tables. For further details, see \cite{dumitriu2026bbpphasetransitiondoubly}.

This does reveal an avenue of future exploration: given $1\ge p\gg\frac{1}{\log(n)}$ and $1-q\gg1/n^2,$ are there any methods that can successfully distinguish or recover when $\theta<1$? Other works have explored the dense spike ($np=\Theta(1)$) and dense noise ($q=1$) case, which we give a thorough review of in Section \ref{amp_contiguity}. These methods rely on a continuous noise distribution, and thus, as noise sparsity is discrete, a new procedure would be required when $np=\Theta(1)$ and $0\le1-q\lesssim 1/n^2$. In a forthcoming paper, we hope to shed light on this region by providing impossibility results. 

Therefore, in Table \ref{distinguish_table}/Table \ref{recovery_table}, the shaded light pink region represents PCA achieving Distinguishability/Weak Recovery, respectively, for a centered, unit variance sub-Gaussian spike (including Rademacher) planted in sub-Gaussian noise.

\paragraph{Table \ref{distinguish_table}/Table \ref{recovery_table} Discussion} The $y$-axis represents noise sparsity, and the $x$-axis represents spike sparsity. Because different methods often have overlapping conditions based on different lower and upper bounds, the axes are not perfectly segmented. One can think of the values as threshold cutoffs on a graph of $np$ by $q$. We include the asymptotic inequalities to allow the reader to connect a column and row and see which algorithms succeed under specific sparsity assumptions. 
\newpage

\begin{table}[H]
    \centering
    \renewcommand{\arraystretch}{1.5} 
\begin{tabular}{|c|c|c|c|c|c|c|} \hline
    \multicolumn{6}{|c|}{\cellcolor{lightgray}\textbf{Doubly Sparse Distinguishability}} \\ \hline
    
    \cellcolor{lightgray}$1-q\ll 1/n^2$ & \cellcolor{black} & \textcolor{safePink}{$\textbf{QT}_p$}, \textcolor{safeBlue}{$\textbf{DT}_l$}\cellcolor{backgroundOrange} & \textcolor{safePink}{$\textbf{QT}_p$}, \textcolor{safeBlue}{$\textbf{DT}_l$}\cellcolor{backgroundOrange} &  \textcolor{safePink}{$\textbf{QT}$} \cellcolor{backgroundOrange} &\cellcolor{backgroundOrange} \\ \hline

    \cellcolor{lightgray}$1-q\gtrsim1/n^2$ & \cellcolor{black} & \textcolor{safePink}{$\textbf{QT}_p$}, \textcolor{safeBlue}{$\textbf{DT}_l$} \cellcolor{backgroundOrange} & \textcolor{safePink}{$\textbf{QT}_p$}, \textcolor{safeBlue}{$\textbf{DT}_l$}\cellcolor{backgroundOrange} &  \textcolor{safePink}{$\textbf{QT}$} \cellcolor{backgroundOrange} &\textcolor{safePurple}{IS} \cellcolor{backgroundOrange}\\ \hline

    \cellcolor{lightgray}$1-q\gtrsim\frac{(\log n)^2}{n^2}$ &  \cellcolor{black} & \textcolor{safePink}{$\textbf{QT}_p$}, \textcolor{safeBlue}{$\textbf{DT}_l$} \cellcolor{backgroundOrange}& \textcolor{safePink}{$\textbf{QT}_p$}, \textcolor{safeBlue}{$\textbf{DT}_l$}\cellcolor{backgroundOrange} &  \textcolor{safePurple}{IS}, \textcolor{safePink}{$\textbf{QT}$} \cellcolor{backgroundOrange} &\textcolor{safePurple}{IS} \cellcolor{backgroundOrange} \\ \hline

    \cellcolor{lightgray}$1-q\gtrsim\frac{\log n}{n}$ & \cellcolor{black} &  \textcolor{safePink}{$\textbf{QT}_p$}, \textcolor{safeBlue}{$\textbf{DT}_l$} \cellcolor{backgroundOrange}& \textcolor{safePink}{$\textbf{QT}_p$}, \textcolor{safeBlue}{$\textbf{DT}_l$}\cellcolor{backgroundOrange} & \textcolor{safePurple}{IS}, \textcolor{safePink}{$\textbf{QT}$} \cellcolor{backgroundOrange} &\textcolor{safePurple}{IS} \cellcolor{backgroundOrange} \\ \hline


    \cellcolor{lightgray}$1-q\gtrsim\tau^{-2}$ & \cellcolor{black} & \textcolor{safePurple}{IS}, \textcolor{safePink}{$\textbf{QT}_p$}, \textcolor{safeBlue}{$\textbf{DT}_l$} \cellcolor{backgroundOrange} & \textcolor{safePurple}{IS}, \textcolor{safePink}{$\textbf{QT}_p$}, \textcolor{safeBlue}{$\textbf{DT}_l$}\cellcolor{backgroundOrange}& \textcolor{safePurple}{IS}, \textcolor{safePink}{$\textbf{QT}$} \cellcolor{backgroundOrange} & \textcolor{safePurple}{IS} \cellcolor{backgroundOrange}\\ \hline

    \cellcolor{lightgray}$q\gtrsim\Theta(1)$ & \cellcolor{black}  & \textcolor{safePurple}{IS}, \textcolor{safePink}{$\textbf{QT}_p$}, \textcolor{safeBlue}{$\textbf{DT}_l$}\cellcolor{backgroundOrange} & \textcolor{safePurple}{IS}, \textcolor{safePink}{$\textbf{QT}_p$}, \textcolor{safeBlue}{$\textbf{DT}_l$}\cellcolor{backgroundOrange} & \textcolor{safePurple}{IS}, \textcolor{safePink}{$\textbf{QT}$} \cellcolor{backgroundOrange} &\textcolor{safePurple}{IS} \cellcolor{backgroundOrange} \\ \hline


    \cellcolor{lightgray}$q\gtrsim \sqrt{\frac{\log n}{n}} $& \cellcolor{black} & \textcolor{safePurple}{IS}, \textcolor{safePink}{$\textbf{QT}_p$}, \textcolor{safeBlue}{$\textbf{DT}_l$}\cellcolor{backgroundOrange} & \textcolor{safePurple}{IS}, \textcolor{safePink}{$\textbf{QT}$} \cellcolor{backgroundOrange} &\textcolor{safePurple}{IS} \cellcolor{backgroundOrange} & \textcolor{safePurple}{IS} \cellcolor{backgroundOrange}\\ \hline


    \cellcolor{lightgray}$q\gtrsim \frac{\tau^2\log n}{n} $& \cellcolor{black}  &\textcolor{safePurple}{IS}, \textcolor{safePink}{$\textbf{QT}_p$}, \textcolor{safeBlue}{$\textbf{DT}_l$}  \cellcolor{backgroundOrange}& \textcolor{safePurple}{IS} \cellcolor{backgroundOrange}  & \textcolor{safePurple}{IS} \cellcolor{backgroundOrange} & \textcolor{safePurple}{IS} \cellcolor{backgroundOrange} \\ \hline

    \cellcolor{lightgray}$q\gtrsim \frac{\tau\log n}{n}/ \frac{\tau^2\log(nq)}{n}$& \cellcolor{black}  &\textcolor{safePurple}{IS}, \textcolor{safePink}{$\textbf{QT}$}/\textcolor{safeBlue}{$\textbf{DT}_l$}   \cellcolor{backgroundOrange}& \textcolor{safePurple}{IS} \cellcolor{backgroundOrange} & \textcolor{safePurple}{IS} \cellcolor{backgroundOrange} & \textcolor{safePurple}{IS} \cellcolor{backgroundOrange} \\ \hline
     
     \cellcolor{lightgray}$q\gtrsim \frac{\log n}{n} $& \cellcolor{black} & \textcolor{safePurple}{IS} & \textcolor{safePurple}{IS}  & \textcolor{safePurple}{IS} & \textcolor{safePurple}{IS}\\ \hline
     
     \cellcolor{lightgray}$0\le q\ll\frac{\log n}{n}$ & \cellcolor{black} & \textcolor{safePurple}{IS} & \textcolor{safePurple}{IS} & \textcolor{safePurple}{IS}  & \textcolor{safePurple}{IS}  \\ \hline
     
     \cellcolor{lightgray}& \cellcolor{lightgray}$np=O(1)$ & \cellcolor{lightgray}$np=\tau\to\infty$ & \cellcolor{lightgray}$np\lesssim\sqrt{\frac{n}{\log n}}$ & \cellcolor{lightgray}$np\lesssim \frac{n}{\log n}$ & \cellcolor{lightgray}$np=\Theta(n)$  \\ \hline
\end{tabular}
    \caption{Distinguishability results for Identical Search (\textcolor{safePurple}{IS}), polynomial-time and nonpolynomial-time Quadratic Thresholding (\textcolor{safePink}{$\textbf{QT}_p$}/\textcolor{magenta}{$\textbf{QT}$}), linear-time Diagonal Thresholding (\textcolor{safeBlue}{$\textbf{DT}_l$}), and PCA (Pink Shaded Region). Assumptions and axis details in Section \ref{section:table_descriptions}.}
    \label{distinguish_table}
\end{table}

\begin{table}[H]
    \centering
    \renewcommand{\arraystretch}{1.5} 
    \begin{tabular}{|c|c|c|c|c|c|c|c|} \hline
    \multicolumn{7}{|c|}{\cellcolor{lightgray}\textbf{Doubly Sparse Recovery}} \\ \hline
    
    \cellcolor{lightgray}$1-q\ll \frac{\log n}{n}$ & \cellcolor{lightgray} &  \textcolor{safePink}{$\textbf{QT}_{p}$}, \textcolor{safeBlue}{$\textbf{DT}_l$} & \textcolor{safePink}{$\textbf{QT}_p$}, \textcolor{safeBlue}{$\textbf{DT}_l$}\cellcolor{backgroundOrange} \!& \textcolor{safePink}{$\textbf{QT}_p$}, \textcolor{safeBlue}{$\textbf{DT}_l$} \cellcolor{backgroundOrange} & \textcolor{safePink}{$\textbf{QT}$} \cellcolor{backgroundOrange} &\cellcolor{backgroundOrange} \\ \hline

    \cellcolor{lightgray}$1-q\gtrsim\frac{\log n}{n}$ & \cellcolor{lightgray} &  \textcolor{safePink}{$\textbf{QT}_p$}, \textcolor{safeBlue}{$\textbf{DT}_l$} & \textcolor{safePink}{$\textbf{QT}_p$}, \textcolor{safeBlue}{$\textbf{DT}_l$}\cellcolor{backgroundOrange} \! & \textcolor{safePink}{$\textbf{QT}_p$}, \textcolor{safeBlue}{$\textbf{DT}_l$}  \cellcolor{backgroundOrange} & \textcolor{safePink}{$\textbf{QT}$} \cellcolor{backgroundOrange} &\textcolor{safePurple}{IS} \cellcolor{backgroundOrange}\\ \hline

    \cellcolor{lightgray}$1-q\gtrsim \frac{\log(\frac{n}{\log n})}{n/\log n}$ & \cellcolor{lightgray} &  \textcolor{safePink}{$\textbf{QT}_{p}$}, \textcolor{safeBlue}{$\textbf{DT}_l$} & \textcolor{safePink}{$\textbf{QT}_p$}, \textcolor{safeBlue}{$\textbf{DT}_l$}\cellcolor{backgroundOrange} \! & \textcolor{safePink}{$\textbf{QT}_p$}, \textcolor{safeBlue}{$\textbf{DT}_l$}  \cellcolor{backgroundOrange} & \textcolor{safePurple}{IS}, \textcolor{safePink}{$\textbf{QT}$} \cellcolor{backgroundOrange} &\textcolor{safePurple}{IS} \cellcolor{backgroundOrange} \\ \hline

    \cellcolor{lightgray}\!$1-q\gtrsim \frac{\log(\sqrt{\frac{n}{\log n}})}{\sqrt{n/\log n}}$ \! & \cellcolor{lightgray} &  \textcolor{safePink}{$\textbf{QT}_{p}$}, \textcolor{safeBlue}{$\textbf{DT}_l$} &  \textcolor{safePink}{$\textbf{QT}_p$}, \textcolor{safeBlue}{$\textbf{DT}_l$}\cellcolor{backgroundOrange} \! & \textcolor{safePurple}{IS}, \textcolor{safePink}{$\textbf{QT}_p$}, \textcolor{safeBlue}{$\textbf{DT}_l$}  \cellcolor{backgroundOrange}& \textcolor{safePurple}{IS}, \textcolor{safePink}{$\textbf{QT}$} \cellcolor{backgroundOrange} &\textcolor{safePurple}{IS}\cellcolor{backgroundOrange} \\ \hline

    \cellcolor{lightgray}$1-q\gtrsim\frac{\log \tau}{\tau}$ & \cellcolor{lightgray} &  \textcolor{safePink}{$\textbf{QT}_{p}$}, \textcolor{safeBlue}{$\textbf{DT}_l$} & \textcolor{safePurple}{IS}, \textcolor{safePink}{$\textbf{QT}_p$}, \textcolor{safeBlue}{$\textbf{DT}_l$}, \cellcolor{backgroundOrange} & \textcolor{safePurple}{IS}, \textcolor{safePink}{$\textbf{QT}_p$}, \textcolor{safeBlue}{$\textbf{DT}_l$}  \cellcolor{backgroundOrange} \! & \textcolor{safePurple}{IS}, \textcolor{safePink}{$\textbf{QT}$}\cellcolor{backgroundOrange} & \textcolor{safePurple}{IS}\cellcolor{backgroundOrange}\\ \hline

    \cellcolor{lightgray}$q\gtrsim\Theta(1)$ & \cellcolor{lightgray} &  \textcolor{safePink}{$\textbf{QT}_{p}$}, \textcolor{safeBlue}{$\textbf{DT}_l$} & \textcolor{safePurple}{IS}, \textcolor{safePink}{$\textbf{QT}_p$}, \textcolor{safeBlue}{$\textbf{DT}_l$}\cellcolor{backgroundOrange} \! & \textcolor{safePurple}{IS}, \textcolor{safePink}{$\textbf{QT}_p$}, \textcolor{safeBlue}{$\textbf{DT}_l$} \cellcolor{backgroundOrange} & \textcolor{safePurple}{IS}, \textcolor{safePink}{$\textbf{QT}$} \cellcolor{backgroundOrange} &\textcolor{safePurple}{IS} \cellcolor{backgroundOrange} \\ \hline

    \cellcolor{lightgray}$q\gtrsim \sqrt{\frac{\log n}{n}} $& \cellcolor{lightgray} & \textcolor{safePurple}{IS}, \textcolor{safePink}{$\textbf{QT}_{p}$}, \textcolor{safeBlue}{$\textbf{DT}_l$} & \textcolor{safePurple}{IS}, \textcolor{safePink}{$\textbf{QT}_p$}, \textcolor{safeBlue}{$\textbf{DT}_l$}\cellcolor{backgroundOrange} \! & \textcolor{safePurple}{IS} \cellcolor{backgroundOrange} &\textcolor{safePurple}{IS} \cellcolor{backgroundOrange} & \textcolor{safePurple}{IS} \cellcolor{backgroundOrange}\\ \hline

    \cellcolor{lightgray}$q\gtrsim \frac{\tau^2\log n}{n} $& \cellcolor{lightgray} & \textcolor{safePurple}{IS}, \textcolor{safePink}{$\textbf{QT}_{p}$}, \textcolor{safeBlue}{$\textbf{DT}_l$} & \textcolor{safePurple}{IS}, \textcolor{safePink}{$\textbf{QT}_{p}$}, \textcolor{safeBlue}{$\textbf{DT}_l$} \cellcolor{backgroundOrange} & \textcolor{safePurple}{IS} \cellcolor{backgroundOrange} \! & \textcolor{safePurple}{IS} \cellcolor{backgroundOrange} & \textcolor{safePurple}{IS} \cellcolor{backgroundOrange} \\ \hline

    \cellcolor{lightgray}$q\gtrsim \frac{\tau\log n}{n}/\frac{\tau^2\log(nq)}{n} $& \cellcolor{lightgray} & \textcolor{safePurple}{IS}, \textcolor{safePink}{$\textbf{QT}_p$}, \textcolor{safeBlue}{$\textbf{DT}_l$} & \textcolor{safePurple}{IS}, \textcolor{safePink}{$\textbf{QT}$}/\textcolor{safeBlue}{$\textbf{DT}_l$} \cellcolor{backgroundOrange} & \textcolor{safePurple}{IS} \cellcolor{backgroundOrange} & \textcolor{safePurple}{IS} \cellcolor{backgroundOrange} & \textcolor{safePurple}{IS} \cellcolor{backgroundOrange} \\ \hline
     
     \cellcolor{lightgray}$q\gtrsim \frac{\log n}{n} $& \cellcolor{lightgray} & \textcolor{safePurple}{IS}, \textcolor{safePink}{$\textbf{QT}_p$}, \textcolor{safeBlue}{$\textbf{DT}_l$} & \textcolor{safePurple}{IS} & \textcolor{safePurple}{IS} & \textcolor{safePurple}{IS} & \textcolor{safePurple}{IS}\\ \hline
     
     \cellcolor{lightgray}$0\le q\ll\frac{\log n}{n}$ & \cellcolor{lightgray} & \textcolor{safePurple}{IS} & \textcolor{safePurple}{IS} & \textcolor{safePurple}{IS} & \textcolor{safePurple}{IS}  & \textcolor{safePurple}{IS}  \\ \hline
     
     \cellcolor{lightgray}& \cellcolor{lightgray}$np\to0$ & \cellcolor{lightgray}$np=\Theta(1)$ & \cellcolor{lightgray}$np=\tau\to\infty$ & \cellcolor{lightgray}$np\lesssim\sqrt{\frac{n}{\log n}}$ & \cellcolor{lightgray}$np\lesssim\frac{n}{\log n}$ & \cellcolor{lightgray}$np=1$ \\
     
\end{tabular}
    \caption{Recovery results for Identical Search (\textcolor{safePurple}{IS}), polynomial-time and nonpolynomial-time Quadratic Thresholding (\textcolor{safePink}{$\textbf{QT}_p$}/\textcolor{magenta}{$\textbf{QT}$}), linear-time Diagonal Thresholding (\textcolor{safeBlue}{$\textbf{DT}_l$}), and PCA (Pink Shaded Region). Assumptions and axis details in Section \ref{section:table_descriptions}.}
    \label{recovery_table}
\end{table}

\newpage

\section{Recovery Experiments}\label{sec:experiments}

\paragraph{Setup.}
We simulate model \eqref{model} with $W$ being GOE, a symmetric Bernoulli$(q)$
mask $A$ and $v=b\odot\tilde v$, $b_i\sim\mathrm{Ber}(p)$. We use three spike priors, all with $\mathbb E\tilde v_i^2=1$:
Rademacher (Assumption \ref{rademacher_assumption}); a lower-bounded prior $|\tilde v_i|\sim \text{Unif}[0.8,1.187]$ with a random sign
(Assumption \ref{bounded_spike_assumption} with $B=1.25$); and, respectively, standard Gaussian (Assumption \ref{assumption:spike_subG}). 

We compare the four methods for signal recovery highlighted at the beginning of Section \ref{section:main_results_recovery} (with $l=2$ for Algorithm \ref{quad_recovery_alg_recov}). DT (Algorithm~\ref{alg:DT_recov}) selects $\{i:|X_{ii}|\ge t\}$; the threshold $t$ follows the proof of Theorem~\ref{recovery_gaussian_DT} with explicit constants, as follows. The noise term (\ref{noise_cond}) requires $t\ge L:=2\sqrt{\log(nq)/(nq)}$, while for lower bounded $|\tilde v_i|\ge1/B$, the small-ball term of (\ref{small_ball_cond}) vanishes for $t<U:=\theta/(2B^2np)$. Note that for Rademacher prior, $B=1$. For Gaussian prior, we require $t\ll U$ where $U=\theta/(2C_s^2(np)^3)$. In the case of Rademacher and lower bounded priors, we take $t=U$ when $L<U$ and $t=L$ otherwise. In the Gaussian setting, $(L,U)$ does not contain an integer at every $n$ we simulate ($n$ would need to be very large for the interval to contain an integer).
For $l=2$, Algorithm \ref{quad_recovery_alg_recov} runs in $O(n^2)$ time: for $w=s_ae_a+s_be_b$ one has
$w^\top Xw=X_{aa}+X_{bb}+2s_as_bX_{ab}$, so the maximizer is the pair yielding $X_{aa}+X_{bb}+2|X_{ab}|$.
For each experiment, each method evaluates the recovery metric its theorem provides: correlation
$|\langle \hat v,v\rangle|/(\|\hat v\|\|v\|)$ for PCA, exact recovery $R(\hat v,v)=1$ for QT and IS with the Rademacher
prior, and support recovery $\mathrm{supp}(\hat v)=\mathrm{supp}(v)$ or correlation otherwise. 

See each figure and table for details on the number of trials; in each trial, all methods
and all values of $\theta$ are applied to the same draw of $(v,W,A)$. Error bands are 95\% Wilson intervals for
success rates and mean $\pm1.96$ standard errors for correlations.

\subsection{Extremely Sparse Signal}\label{sec:exp-extreme}

\paragraph{Quadratic Thresholding below the spectral threshold.}
We take $p=8/n$ with dense noise $q=1$, so that Theorem~\ref{main_recov_thm_bounded}(a) guarantees exact recovery by QT for the Rademacher
prior, while PCA has no guarantee. Figure~\ref{fig:exp-case1} and Table~\ref{tab:exp-case1} show that the
 probability of exact recovery of QT increases with $n$ at all $\theta$, including $\theta<1$: at $n=10^4$, it is
$0.96$ for $\theta=0.8$ and $1$ for $\theta\ge1$, and the signal strength needed for exact recovery decreases from
$\theta\approx2$ at $n=10^3$ to $\theta\approx0.8$ at $n=10^4$. The PCA correlation, in contrast and as expected, does not improve with
$n$ (between $0.15$ and $0.76$ at $n=10^4$, similar to $0.20$--$0.79$ at $n=500$). IS fails, as no entry is noise-free when $q=1$.

Results for DT depend on its interval. For the Rademacher
prior, $U=\theta/(2B^2np)>L$ at $n=10^4$ for $\theta\ge1$, and DT recovers the support in $95$--$100\%$ of the
trials ($90\%$ at $\theta=0.8$, just outside the interval); at every $\theta$, the rate increases with $n$. For the lower-bounded prior, $U$ is smaller as
$B^2=1.56$ and the interval contains an integer at $n=10^4$ only for $\theta>1.52$; DT recovers the support in every trial at
$\theta=2$ and, with $t=L$, in $47$--$83\%$ of the trials for $\theta\le1.5$. For the Gaussian prior, the interval needs $n\gtrsim10^7$ and the rate is at most $5\%$: with
$\|v\|_0\approx8$ some nonzero coordinate almost always has $\theta v_i^2/(np)$ below the noise level, so no threshold
recovers the full support. Nevertheless, the refined eigenvector reaches exceptional correlation $0.97$--$1$ (lower bounded) and
$0.89$--$0.98$ (Gaussian) at $n=10^4$, since the missed entries carry little of $\|v\|^2$.


For the lower-bounded prior, QT recovers the support at $n=10^4$ in $93$--$99\%$ of the trials for $\theta\ge1.2$
($34\%$ at $\theta=1$, none at $\theta=0.8$), with correlation $0.98$--$0.99$ after the eigenvector refinement. After the eigenvector refinement, this
matches Theorem \ref{quad_form_recovery_thm}. 

\begin{figure}[H]
  \centering
  \begin{subfigure}[t]{0.48\linewidth}\centering
    \includegraphics[width=\linewidth]{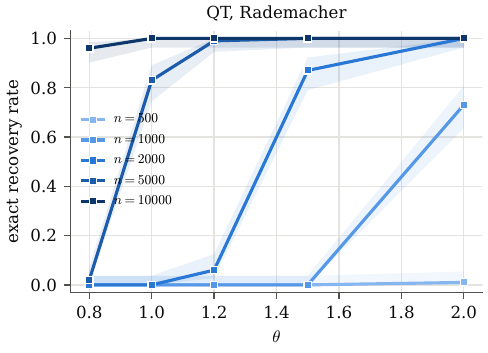}\caption{QT exact recovery}
  \end{subfigure}\hfill
  \begin{subfigure}[t]{0.48\linewidth}\centering
    \includegraphics[width=\linewidth]{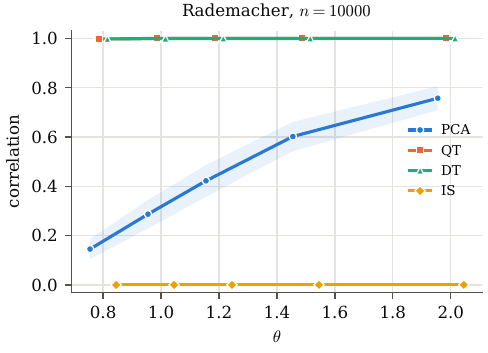}\caption{all methods, $n=10^4$}
  \end{subfigure}
  \caption{Sparse spike and dense noise: $p=8/n$, $q=1$, Rademacher prior, 100 trials per point.
  (a) Exact-recovery probability of QT ($l=2$) for different $\theta$, one curve per $n$.
  (b) Correlation results for each of the four methods at $n=10^4$ for different $\theta$.}
  \label{fig:exp-case1}
\end{figure}

\begin{table}[H]
  \centering\small
  \resizebox{\linewidth}{!}{
\begin{tabular}{ll lccccc}
\toprule
prior & method & theory & $\theta=0.8$ & $\theta=1$ & $\theta=1.2$ & $\theta=1.5$ & $\theta=2$ \\
\midrule
Rademacher & PCA (corr) & $\circ$ needs $np\to\infty$ & 0.15 & 0.29 & 0.42 & 0.60 & 0.76 \\
 & QT (exact) & \textcolor{LimeGreen}{$\checkmark$} Thm \ref{main_recov_thm_bounded}.1(a) exact & 0.96 & 1.00 & 1.00 & 1.00 & 1.00 \\
 & DT (supp) & \textcolor{LimeGreen}{$\checkmark$} Thm \ref{main_recov_thm_bounded}.1(c) support, $\theta\ge1$ & 0.90 & 0.95 & 1.00 & 1.00 & 1.00 \\
 & IS (exact) & \textcolor{red}{$\times$} $q=1$ & 0.00 & 0.00 & 0.00 & 0.00 & 0.00 \\
\midrule
lower bounded & PCA (corr) & $\circ$ needs $np\to\infty$ & 0.14 & 0.30 & 0.46 & 0.62 & 0.79 \\
 & QT (supp) & \textcolor{LimeGreen}{$\checkmark$} Thm \ref{main_recov_thm_bounded}.2(a) support, $\theta\ge1.2$ & 0.00 & 0.34 & 0.93 & 0.99 & 0.99 \\
 & DT (supp) & \textcolor{LimeGreen}{$\checkmark$} Thm \ref{main_recov_thm_bounded}.2(b) support, $\theta\ge1.6$ & 0.47 & 0.77 & 0.81 & 0.83 & 1.00 \\
 & IS (supp) & \textcolor{red}{$\times$} $q=1$ & 0.00 & 0.00 & 0.00 & 0.00 & 0.00 \\
\midrule
Gaussian & PCA (corr) & $\circ$ needs $np\to\infty$ & 0.21 & 0.33 & 0.44 & 0.55 & 0.67 \\
 & QT ($\tau_3=1$) (supp) & \textcolor{red}{$\times$} Thm \ref{thm:QT_subG_support_recovery} needs $np\to\infty$ & 0.08 & 0.31 & 0.33 & 0.32 & 0.31 \\
 & QT ($\tau_3=2$) (supp) & \textcolor{red}{$\times$} Thm \ref{thm:QT_subG_support_recovery} needs $np\to\infty$ & 0.00 & 0.00 & 0.00 & 0.01 & 0.56 \\
 & QT ($\tau_3=4$) (supp) & \textcolor{red}{$\times$} Thm \ref{thm:QT_subG_support_recovery} needs $np\to\infty$ & 0.00 & 0.00 & 0.00 & 0.00 & 0.00 \\
 & DT (supp) & \textcolor{red}{$\times$} Thm \ref{recovery_gaussian_DT} window empty & 0.00 & 0.00 & 0.01 & 0.03 & 0.05 \\
 & IS (supp) & \textcolor{red}{$\times$} $q=1$ & 0.00 & 0.00 & 0.00 & 0.00 & 0.00 \\
\bottomrule
\end{tabular}
}
  \caption{Sparse spike and dense noise ($p=8/n$, $q=1$) at $n=10^4$, 100 trials. ``theory'': \textcolor{LimeGreen}{$\checkmark$} a theorem
  guarantees the metric, $\times$ its condition fails, $\circ$ no guarantee.}
  \label{tab:exp-case1}
\end{table}

\paragraph{Bounded expected support and PCA.}
With $p=5/n$ and $q=1$ the support size $K=\|v\|_0$ is approximately Poisson$(5)$, so the strength of the planted
rank-one perturbation, $\theta_{\rm eff}=\theta K/5$, is random. Figure~\ref{fig:exp-bounded}(a) shows that, trial by
trial and for every $n$, the squared PCA correlation follows the BBP curve $1-1/\theta_{\rm eff}^2$; PCA fails exactly
in the trials with $\theta_{\rm eff}\le1$. With the support size fixed to $K\in\{4,5,6\}$
(Figure~\ref{fig:exp-bounded}(b)), the squared correlation converges to the BBP value 
and decays to zero when $\theta_{\text{eff}}\le1$. Consequently, the average PCA performance with
Bernoulli support does not converge as $n$ grows as expected, in line with remark after Theorem~\ref{main_recov_thm_bounded} \citep[PCA requires $np\to\infty$ by][]{dumitriu2026bbpphasetransitiondoubly}, while QT recovers the spike exactly in $95$--$98\%$ of the trials for
$n\ge2000$ in $\theta=1.2$; for $n\ge2000$ its only failures are trials with $\|v\|_0=1$, where a search over pairs cannot succeed. 

\begin{figure}[H]
  \centering
  \begin{subfigure}[t]{0.48\linewidth}\centering
    \includegraphics[width=\linewidth]{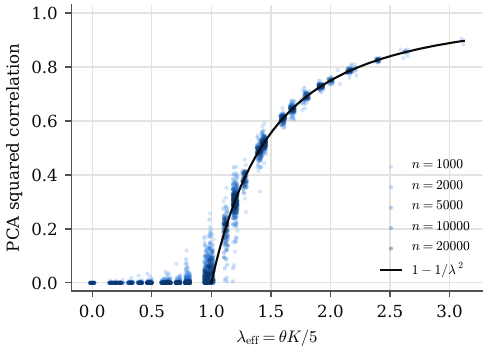}\caption{per trial}
  \end{subfigure}\hfill
  \begin{subfigure}[t]{0.48\linewidth}\centering
    \includegraphics[width=\linewidth]{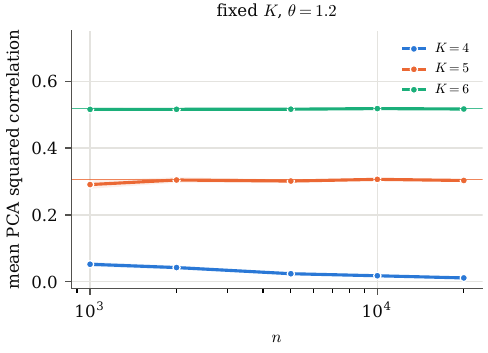}\caption{fixed $K$, $\theta=1.2$}
  \end{subfigure}\hfill
  \caption{Bounded expected support, $p=5/n$, $q=1$, Rademacher prior.
  (a) Squared PCA correlation against $\theta_{\rm eff}=\theta K/5$ for $\theta\in\{0.8,1,1.2\}$ and
  $n\in\{10^3,\dots,2\cdot10^4\}$ (200 trials each), with the BBP curve $1-1/\lambda^2$.
  (b) Fixed support size $K$: mean squared correlation against $n$; thin lines: BBP limits.
  }
  \label{fig:exp-bounded}
\end{figure}

\subsection{Moderately Sparse Signal}\label{sec:exp-moderate}

Theorem \ref{quad_form_recovery_thm} predicts that the polynomial-time version of QT ($l$ fixed) succeeds when
$R_{\rm QT}=q\theta^2l/(np^2\log n)$ exceeds a constant. To capture the transition, we set $p=c/\sqrt{n\log n}$, so that $np^2\log n=c^2$ and
$R_{\rm QT}=2\theta^2/c^2$ for $l=2$, $q=1$, and vary $c\in[0.2,1.5]$, $\theta\in\{1,1.5,2\}$ and
$n\in\{10^3,\dots,2\cdot10^4\}$. The exact-recovery curves for all $n$, plotted against $R_{\rm QT}$, converge
(Figure~\ref{fig:exp-qtscaling}). 

For DT, the condition $L<U$ gives $\theta/(2np)>2\sqrt{\log n/n}$, i.e. $c<\theta/4$ independently of $n$ (the condition $q\gg np^2\log(nq)$ of Theorems~\ref{distinguish_gaussian_DT} and~\ref{recovery_gaussian_DT}): DT recovers the support in 99–100\% of the trials for $c=0.2$ and $(c,\theta)=(0.4,2)$, and in at most 5\% for $c\ge0.8$.

One could consider moderately sparse noise in this setting; however, this would only capture the region of transition between polynomial and non-polynomial time QT, so we exclude these experiments. 

\begin{figure}[H]
  \centering
  \begin{subfigure}[t]{0.48\linewidth}\centering
    \includegraphics[width=\linewidth]{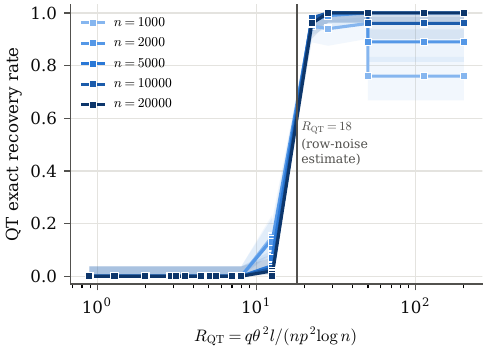}\caption{against $R_{\rm QT}$}
  \end{subfigure}\hfill
  \begin{subfigure}[t]{0.48\linewidth}\centering
    \includegraphics[width=\linewidth]{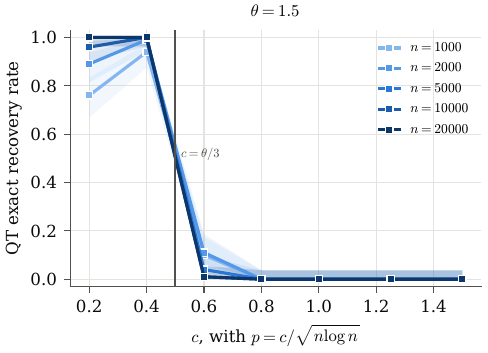}\caption{against $c$, $\theta=1.5$}
  \end{subfigure}
  \caption{QT sparsity scaling: $p=c/\sqrt{n\log n}$, $q=1$, $l=2$, Rademacher prior, 100 trials per point.
  (a) Exact-recovery probability against $R_{\rm QT}=q\theta^2l/(np^2\log n)=2\theta^2/c^2$, all $(c,\theta)$ pooled, one
  curve per $n$. (b) The same at $\theta=1.5$ against $c$; the vertical line is $c=\theta/3$.}
  \label{fig:exp-qtscaling}
\end{figure}

\subsection{Dense Signal}\label{sec:exp-dense}

\paragraph{Sparse noise: Identical Search.}
With $p=1/2$ and $q=5/n$ the noise is too sparse for any spectral method, but most entries of $X_{\text{supp}}$
are noise-free. IS recovers the spike exactly in every trial, for all $n\in\{10^3,\dots,2\cdot10^4\}$ and
all $\theta\in\{0.3,\dots,1.2\}$, including $\theta=0.3$, far below the spectral threshold, while the PCA correlation is
at most $0.07$ and QT and DT fail (Figure~\ref{fig:exp-is}(a), Table~\ref{tab:exp-case2}).

On the support, IS sees the Erdős–Rényi graph $G(K,1-q)$ of noise-free entries and succeeds exactly when it is connected. Theorem~\ref{iden_search_recov_thm_np_to_inf} guarantees this when $(1-q)K\ge\log K+\tau_6$, i.e. for $1-q=\alpha\log K/K$ when $\alpha>1$. 
The experiment shows this threshold is sharp.
In Figure~\ref{fig:exp-is}(c), we performed 2000 planted and 100 null trials, as in this case the null error is insignificant as per Theorem \ref{iden_search_distinguish}. 

\begin{figure}[H]
  \centering
  \begin{subfigure}[t]{0.32\linewidth}\centering
    \includegraphics[width=\linewidth]{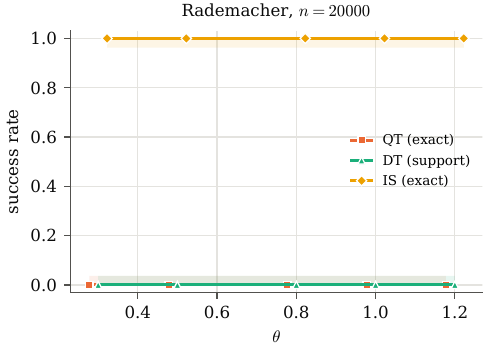}\caption{$p=1/2$, $q=5/n$}
  \end{subfigure}\hfill
  \begin{subfigure}[t]{0.32\linewidth}\centering
    \includegraphics[width=\linewidth]{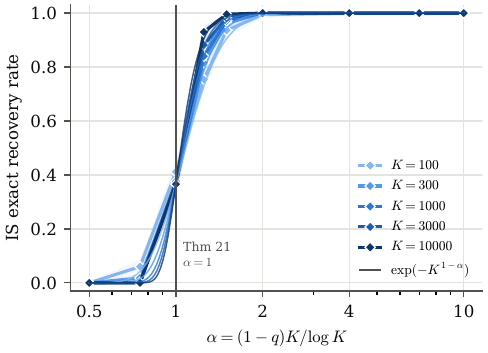}\caption{exact recovery}
  \end{subfigure}\hfill
  \begin{subfigure}[t]{0.32\linewidth}\centering
    \includegraphics[width=\linewidth]{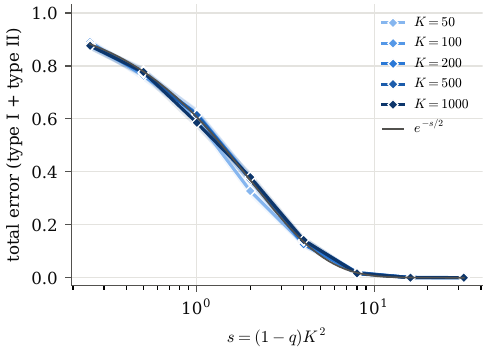}\caption{distinguishing}
  \end{subfigure}
  \caption{Identical Search. (a) Success rates at $n=2\cdot10^4$ for $p=1/2$, $q=5/n$ (100 trials).
  (b) Exact recovery on the noise-free graph $G(K,1-q)$, $1-q=\alpha\log K/K$, 500 trials; thin lines $\exp(-K^{1-\alpha})$, the connectivity limit used in Theorem~\ref{iden_search_recov_thm_np_to_inf}; vertical line: its threshold $\alpha=1$. (c) Total error of the full-knowledge test, $1-q=s/K^2$, 2000 planted and 100 null trials.}
  \label{fig:exp-is}
\end{figure}

\begin{table}[H]
  \centering\small
\begin{tabular}{ll lccccc}
\toprule
prior & method & theory & $\theta=0.3$ & $\theta=0.5$ & $\theta=0.8$ & $\theta=1$ & $\theta=1.2$ \\
\midrule
Rademacher & PCA (corr) & \textcolor{red}{$\times$} $nq\not\gg\log n$ & 0.01 & 0.01 & 0.01 & 0.01 & 0.01 \\
 & QT (exact) & \textcolor{red}{$\times$} $q\not\gtrsim p\log n$ & 0.00 & 0.00 & 0.00 & 0.00 & 0.00 \\
 & DT (supp) & \textcolor{red}{$\times$} Thm \ref{recovery_gaussian_DT} (remark) empty interval & 0.00 & 0.00 & 0.00 & 0.00 & 0.00 \\
 & IS (exact) & \textcolor{LimeGreen}{$\checkmark$} Thm \ref{main_recov_thm_unbounded}.1(d) exact & 1.00 & 1.00 & 1.00 & 1.00 & 1.00 \\
\bottomrule
\end{tabular}

  \caption{Dense spike and very sparse noise ($p=1/2$, $q=5/n$) at $n=2\cdot10^4$, 100 trials.}
  \label{tab:exp-case2}
\end{table}

\paragraph{Dense noise.}
For $p=q=1$ PCA follows the BBP prediction (Table~\ref{tab:exp-case3}): its correlation is $0.55$ at $\theta=1.2$ and
$0.75$ at $\theta=1.5$ against $\sqrt{1-1/\theta^2}=0.553$ and $0.745$, and it decays to zero with $n$ for $\theta\le1$.
QT, DT, and IS fail in this regime, as predicted.

\begin{table}[H]
  \centering\small
\begin{tabular}{ll lccccc}
\toprule
prior & method & theory & $\theta=0.5$ & $\theta=0.8$ & $\theta=1$ & $\theta=1.2$ & $\theta=1.5$ \\
\midrule
Rademacher & PCA (corr) & \textcolor{LimeGreen}{$\checkmark$} weak iff $\theta>1$ & 0.02 & 0.04 & 0.17 & 0.55 & 0.75 \\
 & QT (exact) & \textcolor{red}{$\times$} $q\not\gtrsim p\log n$ & 0.00 & 0.00 & 0.00 & 0.00 & 0.00 \\
 & DT (supp) & \textcolor{red}{$\times$} Thm \ref{recovery_gaussian_DT} (remark) window empty & 0.00 & 0.00 & 0.00 & 0.00 & 0.00 \\
 & IS (exact) & \textcolor{red}{$\times$} $q=1$ & 0.00 & 0.00 & 0.00 & 0.00 & 0.00 \\
\bottomrule
\end{tabular}

  \caption{Dense spike and dense noise ($p=q=1$) at $n=10^4$, 100 trials.}
  \label{tab:exp-case3}
\end{table}


\newpage

\section{Spike Support Regimes}\label{spike_support_regimes_section}

We often condition on $\|v\|_0$ being known. Model \ref{model} means $\mathbb E[\|v\|_0]=np$, but fluctuation can occur. The scope of this fluctuation, as well as its connection to whether the problem can be solved or not, relies on the sparsity parameter $p$. 

\subsection{Unbounded Spike Support Size}

We begin with the regime $np\to \infty.$ Here, we face no obstacles when distinguishing or recovering and instead focus on pinning down the fluctuation of $\|v\|_0$. We defer the proof of the following to Section \ref{appendix:fixed_mask_proof} as it is a simple application of Bernstein's Inequality.

\begin{lemma}\label{fixedmasklemma}
Let $\tilde v$ satisfy any dense spike prior, and $b$ satisfy Assumption \ref{ind_bern_spike_assumption}. Then, for $v=b\odot\tilde v$ and $np\to\infty$, we have $
\mathbb P \left( \|v\|_0\not\in \left[np-\log(np)\sqrt{np}, np+\log(np)\sqrt{np}\right] \right)\le (np)^{-2}\to 0~.
$

\end{lemma}

\subsection{Bounded Spike Support Size}\label{constant_np_section}

Next, we explore $np=O(1)$. First, we establish a negative result for distinguishability. 

\begin{theorem}\label{distinguish_them_np_equals_constant}
When $np=O(1)$, distinguishing between $\mathbb P$ and $\mathbb Q$ is impossible.
\end{theorem}

\begin{proof}
Let
$
\alpha_n:=(1-p)^n
$
denote the probability that the spike has empty support. Let $P_+$ denote the planted distribution conditional on $\|v\|_0>0$. Since, on the event $\|v\|_0=0$, we have $v=0$ and hence the planted model coincides with the null model, the planted distribution admits the mixture decomposition $
\mathbb P
=
\alpha_n\mathbb Q+(1-\alpha_n)P_+.
$

Consider any test $\phi\in{0,1}$, where $\phi=1$ denotes selecting the planted model. Its sum of type-I and type-II errors is
$
\mathbb P(\phi=0)+\mathbb Q(\phi=1).
$
Using the mixture decomposition,
\begin{align*}
\mathbb P(\phi=0)+\mathbb Q(\phi=1)
&=
\alpha_n\mathbb Q(\phi=0)
+(1-\alpha_n)P_+(\phi=0)
+\mathbb Q(\phi=1)\\
&\geq
\alpha_n\mathbb Q(\phi=0)+\mathbb Q(\phi=1)\\
&=\alpha_n+(1-\alpha_n)\mathbb Q(\phi=1)\\
&\geq \alpha_n.
\end{align*}
Thus, whenever $\alpha_n$ is bounded away from zero, the sum of the type-I and type-II errors cannot converge to zero. If $np\to\lambda\in[0,\infty)$ along a subsequence, then

$$
\alpha_n=(1-p)^n
=\left(1-\frac{np}{n}\right)^n
\longrightarrow e^{-\lambda}>0.
$$

If $np=O(1)$ but does not converge, then it has a bounded subsequence along which $np\to\lambda\in[0,\infty)$. Along this subsequence, $
\alpha_n\to e^{-\lambda}>0.
$

Therefore, along this subsequence, the sum of the type-I and type-II errors is bounded away from zero. Hence the errors cannot both converge to zero, and distinguishing is impossible.
\end{proof}

While distinguishing, by the nature of its definition, may not be possible, one can still recover as we know $X\sim \mathbb P$ and do not have to worry about determining otherwise. When $np=\Theta(1)$, the size of the support can still fluctuate and equal $0.$ However, if the support is nonexistent, recovery defaults to a null case and is in fact successful if the algorithm returns the $0$ vector. We often use the fact that, for $np=\Theta(1),$  $\|v\|_0=O_P(1)$ as $\|v\|_0$ is a binomial random variable.

\section{Quadratic Thresholding}\label{quadratic_thresholding_section}

In this section, we extend the results of \cite{ding2022subexponentialtimealgorithmssparsepca}. In short, this procedure allows one to make an initial guess by searching over possible vectors with support size $l$. We can then balance this runtime parameter $l$, which controls the size of the search set, with sparsity parameters $p$ and $q.$

In \cite{ding2022subexponentialtimealgorithmssparsepca}, the authors consider the noise distribution to be GOE. Because of the inherent rotational invariance, this allows them to implement a handy method of adding and subtracting a newly sampled GUE matrix to introduce independence into their algorithm. Here, we present a much broader approach that does not necessitate the rotational invariance and allows us to generalize the noise matrix to the Hadamard product of a sub-Gaussian and Bernoulli matrix. Moving forward, we let $Y:=\frac{1}{\sqrt{nq}}W\odot A.$

Furthermore, \cite{ding2022subexponentialtimealgorithmssparsepca} assumed the spike satisfied Assumption \ref{bounded_spike_assumption}. Here, we also include results satisfying this assumption but first extend the procedure for the more general Assumption \ref{spike_small_ball_assumption}. This small ball assumption comes from our desire to apply the procedure to more general sub-Gaussian distributions. Indeed, Assumption \ref{spike_small_ball_assumption} is very natural and covers common sub-Gaussian distributions such as Gaussian, uniform, and Rademacher distributions. Essentially, it allows any continuous distributions with bounded density near zero and discrete distributions whose support is bounded away from zero, while excluding distributions with atoms at zero. 

\subsection{Main Results}


We operate in a partial knowledge setting where parameter values are known. However, the success of the following algorithms relies on choosing the correct tuning parameter $l$ in a window controlled by the unknown value $\|v\|_0.$ In the following theorems, we condition on the value $\|v\|_0$ being fixed but unknown. Then, we remove this conditioning to establish high probability windows based on $np$ (see Section \ref{quad_form_interpretation}). When $np\to\infty,$ the window for $l$ is fully visible (w.h.p.). When $np=\Theta(1)$ (in the case of recovery), the left side of the interval is visible, but the right side is some unknown $\Theta(1)$ integer. Thus, one would need to start at the bottom and incrementally increase $l$ to reveal the value of $\|v\|_0.$

Because the distinguishability proof does not require independence, our distinguishability result follows a similar approach to \cite{ding2022subexponentialtimealgorithmssparsepca}, but generalizes the result for a small-ball spike and sparse noise.

\begin{algorithm}
    \caption{Distinguishability by Quadratic Thresholding }\label{quad_form_distinguish_alg}
    \begin{algorithmic}
        \STATE Inputs: $t$, $\theta>0, p\in[0,1], q \in [0,1],n\in \mathbb N$ and $X$ where $X\sim \mathbb P$ or $X \sim \mathbb Q$.
        \STATE 1. Consider search set $\mathcal I_{n,l}:=\{w\in\{0,\pm1 \}^n : \|w\|_0=l \}$
        \STATE 2. Determine test statistic $\mathcal T:= \max_{w\in\mathcal I_{n,l}} w^TXw$
        \IF{$\mathcal T\ge t$}
            \STATE return $\mathbb P$
        \ELSE
            \STATE return $\mathbb Q$
        \ENDIF
    \end{algorithmic}
\end{algorithm}

\begin{remark}
    The exhaustive search over the set $\mathcal I_{n,l}$ dominates the runtime of this algorithm, which in turn determines whether or not the method produces a polynomial time algorithm. Because $|\mathcal I_{n,l}|=\binom{n}{l}2^l\le (2n)^{l},$ the runtime is $n^{O(l)}.$ Thus, if $l=\Theta(1),$ then we have a polynomial time algorithm.
\end{remark}

\begin{theorem}[Quadratic Thresholding Distinguishability Small-Ball Spike]
\label{thm:quad_subG_distinguish}
Assume that $np\to\infty$, $v$ satisfies Assumption \ref{spike_small_ball_assumption}, and $W$ satisfies Assumption~\ref{wigner_subg_assumption}. For $l\leq\|v\|_0$, define $\mathcal T(X)
:=
\max \{w^TXw \mid w\in\mathcal I_{n,l} \}$ and set $t
:=
{\theta l^2}/({128C_s^2np})$.
Then, with $C_5$ sufficiently large, given that we can choose some integer $l$ in the interval
\begin{align*}
\mathcal L_{D_{SG}}
&:=
\left[
\frac{C_5K_2^2C_s^4}{\theta^2}
\frac{np^2\log(n)}{q},
\|v\|_0
\right]~,
\end{align*}
we can distinguish between the $\mathbb P$ and $\mathbb Q$ models with Algorithm \ref{quad_form_distinguish_alg}.
More specifically,
we have, with constants defined in the proof,
\begin{align*}
\mathbb P\left(\mathcal T(X)<t\right)
+
\mathbb Q\left(\mathcal T(X)\geq t\right)
&\leq
\exp\left(
-\frac{
q\theta^2l^2
-C_5K_2^2C_s^4\cdot l\cdot np^2\log(n)
}{
K_2^2C_s^4np^2
}
\right)
+
\exp(-c_6\|v\|_0)
=o(1)~.
\end{align*}
\end{theorem}

See Section \ref{section:subg_dist} for the proof. For a small-ball prior, we can increase the above region of sparsity by using Diagonal Thresholding instead (see Section \ref{section:DT}).

When considering the lower-bounded spike assumption, one could establish an immediate corollary of the above theorem by replacing $C_s$ with $B$ and removing the second term in the tail bound. However, to establish tighter results with respect to constants, we provide the following theorem and defer the proof details to Section \ref{quad_dist_proof} as it follows very closely with the sub-Gaussian generalization above.

\begin{theorem}[Distinguishability by Quadratic Thresholding: Lower Bounded Spike]\label{quad_form_distinguish_thm}
    Assume $np\to\infty$ and that $v$ satisfies Assumption \ref{bounded_spike_assumption}. Let $\mathcal T(X):=\max\limits_{w\in\mathcal I_{n,l}}w^TXw$ where $\mathcal I_{n,l}:=\{w\in \{0,\pm1\}^n:\|w\|_0=l \}.$ Consider the set
    \begin{align*}
        \mathcal L_{D}:=\left[\frac{C_{1}K_2^2B^4}{\theta^2}\cdot \frac{np^2\log(n)}{q}, \|v\|_0\right]~.
    \end{align*}
    Given that $p$ and $q$ are chosen so that $|\mathcal L_D~\cap~ \mathbb Z^+|\ge1$ and $C_{1}$ is sufficiently large, we can distinguish between the planted $\mathbb P$ and null $\mathbb Q$ distributions with Algorithm \ref{quad_form_distinguish_alg} with threshold parameter $t:=\frac{\theta l^2}{2B^2np}$. More specifically, for some absolute constant $c$, we have 
    \begin{align*}
        \mathbb P(\mathcal T(X)<t)+\mathbb Q(\mathcal T(X)>t)\le \exp\left(-\left(\frac{cq\theta^2 l^2-16K_2^2B^4\cdot l\cdot np^2\log(n)}{4K_2^2B^4np^2}\right)\right)=o(1)~,
    \end{align*}
    for any $l\in\mathcal L_D$ when $|\mathcal L_D\cap \mathbb Z^+|\ge1$.
\end{theorem}
Next, we establish our adapted leave-one-out recovery procedure:

\begin{figure}[ht]
  \begin{minipage}[t]{0.6\textwidth}
    \begin{algorithm}[H]
    \caption{Support Recovery by Quadratic Thresholding}
    \label{quad_recovery_alg_recov}
    \begin{algorithmic}
        \STATE Inputs: $t$, $l,$ $\theta>0, p\in[0,1], q \in [0,1],n\in \mathbb N$ and $X$ where $X\sim \mathbb P$. 
        \IF{$l=1$}
            \item Perform \cite[Algorithm 4]{ding2022subexponentialtimealgorithmssparsepca}
        \ELSE
            \FOR{$i\in[n]$}
            \STATE Construct $\mathcal I_{n,l}^{(i)}:=\{w\in\{0,\pm1 \}^n : \|w\|_0=l, w_i=0 \}$
            \STATE Determine initial estimate $v^{*(i)}:= \arg \max\limits_{w\in\mathcal I_{n,l}^{(i)}} w^TX^{(i)}w$
            \IF{$|(Xv^{*(i)})_i|\ge t$}
                \STATE $(v^*)_i:=\text{sign}((Xv^{*(i)})_i)$
            \ELSE
                \STATE $(v^*)_i:=0$
        \ENDIF
        \ENDFOR
        \ENDIF
\STATE \textbf{Output:} $ v^*$
    \end{algorithmic}
\end{algorithm}
  \end{minipage}
  \hfill
  \begin{minipage}[t]{0.38\textwidth}
    \begin{algorithm}[H]
\caption{Sign Recovery by Pairwise Comparison}
\label{sign_recovery_alg}
\begin{algorithmic}
\STATE Inputs: $X\sim\mathbb P$ and $v^*$ and ${v^{*(i)}}_{i\in\operatorname{supp}(v^*)}$ from Algorithm \ref{quad_recovery_alg_recov}.
\STATE 1. Choose $r\in\operatorname{supp}(v^*)$
\STATE 2. Initialize $\widehat v:=0$.
\FOR{$i\in\operatorname{supp}(v^*)$}
\STATE Compute
\(                \widehat s_i:=\operatorname{sign}(v^{*(i)T}Xv^{*(r)}).
           \)
\STATE Compute
\(                \widehat v_i
                :=
                \widehat s_i
                \operatorname{sign}\left((Xv^{*(i)})_i\right)
           \)
\ENDFOR
\STATE \textbf{Output:} $\widehat v$
\end{algorithmic}
\end{algorithm}
\vspace{-.2in}
\begin{algorithm}[H]
\caption{Weak/Strong Recovery}
\label{alg:QT_weak_strong}
\begin{algorithmic}
\STATE Inputs: $X\sim\mathbb P$ \& $v^*$ from Algorithm \ref{quad_recovery_alg_recov}.
\STATE 1. Compute $v_1(X_{\text{supp}(v^*)})$
\STATE \textbf{Output:} $v_1(X_{\text{supp}(v^*)})$
\end{algorithmic}
\end{algorithm}
  \end{minipage}
\end{figure}

\begin{theorem}[Support Recovery by Quadratic Thresholding: Small-Ball Spike]\label{thm:QT_subG_support_recovery}
Assume $np\to\infty$ and that $v$ satisfies Assumption \ref{spike_small_ball_assumption}. Consider the set

$$
\mathcal L_{R_{SG}}
:=
\left[
C_7\frac{C_s^2}{\theta^2}\cdot\frac{\tau_4^2\|v\|_0^2np^2\log(n)}{q},
\|v\|_0-1
\right]~,
$$
where $C_7$ is a sufficiently large absolute constant specified within the proof. Given that $p$ and $q$, as well as $\tau_4\gg 1$, are chosen so that
$|\mathcal L_{R_{SG}}\cap\mathbb Z^+|\geq 1$,
Algorithm \ref{quad_recovery_alg_recov} achieves support recovery of the spike $v$ with threshold $
t:={\theta l^2}/({48C_s^2np}).$
More specifically, when $l\in\mathcal L_{R_{SG}}$,
\begin{align*}
\mathbb P\left(\operatorname{supp}(v^*)=\operatorname{supp}(v)\right)
\geq{}&
1-\exp\left(
-\left(
\frac{\theta^2ql-C_7C_s^2K_2^2\tau_4^2\|v\|_0^2np^2\log(n)}
{9126C_s^2K_2^2\tau_4^2\|v\|_0^2np^2}
\right)
\right)-\frac{C_s}{\tau_4}
-\exp(-c_6\|v\|_0)
=1-o(1).
\end{align*}
\end{theorem}

See Section \ref{section:QT_subG_support_recov} for the proof. Notice the above theorem requires $np\to\infty.$ To achieve support recovery with a small-ball spike prior when $np=\Theta(1)$, as well as access greater sparsity regions, one must use Diagonal Thresholding, which we delay to Section \ref{section:DT}. 

Once again, support recovery for a lower-bounded spike could be achieved as an immediate corollary to the theorem above; however, we separate it to achieve optimal sparsity bounds and account for the fact that support recovery can be achieved when $np=\Theta(1)$. We defer the proof details to Section \ref{quad_recov_proof_exact} as it follows closely with its small-ball generalization.

\begin{theorem}[Support Recovery by Quadratic Thresholding: Lower Bounded Spike]\label{quad_form_recovery_thm}
    Assume $v$ satisfies Assumption \ref{bounded_spike_assumption}. Consider the set $$\mathcal L_R:= \left[\frac{C_2K_2^2B^4}{\theta^2}\cdot\frac{np^2\log(n)}{q}, \max\{\|v\|_0-1,1 \}\right]~,$$ where $C_2$ is a large enough constant specified within the proof. Given that $p$ and $q$ are chosen so $|\mathcal L_R~\cap~\mathbb Z^+|\ge1$ and $t=\frac{\theta l}{3B^2np}$, Algorithm \ref{quad_recovery_alg_recov} achieves support recovery of the spike $v.$ More specifically, we have
    \begin{align*}
        \mathbb P\left(\text{supp}(v^*)=\text{supp}(v)\right)\ge 1-\exp\left(-\left(\frac{q\theta^2 l^2-C_2K_2^2B^4np^2\log(n)}{4K_2^2B^4np^2}\right)\right)=1-o(1)~,
    \end{align*}
    for any $l\in\mathcal L_R$ when $|\mathcal L_R\cap\mathbb Z^+|\ge1$.
\end{theorem}

Refer to Section \ref{quad_form_interpretation} to see how $q\gtrsim p\log(n)$ ensures distinguishability and recovery. In the Rademacher case, we use the following corollary for exact recovery. 
\begin{corollary}[Exact Recovery by Quadratic Thresholding ]\label{corollary:QT_exact}
    Assume $v$ satisfies Assumption \ref{rademacher_assumption} (Rademacher Spike). Assuming Algorithm \ref{quad_recovery_alg_recov} produces $v^*$, which, by Theorem \ref{quad_form_recovery_thm}, achieves support recovery, Algorithm \ref{sign_recovery_alg} achieves exact recovery. More specifically, we have 
    \begin{align*}
        \mathbb P(\mathcal R(\widehat v,v)=1)\ge 1-\exp\left(-\left(\frac{q\theta^2 l^2-C_8K_2^2np^2\log(n)}{4K_2^2np^2}\right)\right)\ge1-n^{-C_8/4}=1-o(1)~,
    \end{align*}
    where $\widehat v$ is the output of Algorithm \ref{sign_recovery_alg} under the conditions of Theorem \ref{quad_form_recovery_thm} with $B=1.$
\end{corollary}

See Section \ref{section:Rademacher_sign_recovery} for the proof. If the spike entries are lower bounded, we can prove the following corollary for weak and strong recovery: 

\begin{corollary}[Weak/Strong Recovery by Quadratic Thresholding: Lower Bounded Spike]\label{quad_form_weak_strong_corollary}
Suppose we have recovered $\text{supp($v$)}$ through Algorithm \ref{quad_recovery_alg_recov}. Given that $v$ satisfies Assumption \ref{bounded_spike_assumption}, Algorithm \ref{alg:QT_weak_strong} achieves weak or strong recovery under the conditions of Theorem \ref{quad_form_recovery_thm}. More specifically, given $v_1(X_{\text{supp}(v)})$ is the leading eigenvector of $X_{\text{supp}(v)},$ we observe two cases. First, if $np\to\infty$, we have, for any $\tau_7>0$ and any $\epsilon \in\left(\frac{2C_3K_2B^2}{C_4\theta}\left(\sqrt{\frac{p}{q}}+\sqrt{\frac{\tau_7}{nq}}\right)(1+o(1)),~1\right)$,
\begin{align*}
\mathbb P\left(\left\langle v_1(X_{\text{supp}(v)}), \frac{1}{\|v\|_2}v\right\rangle^2\le1-\epsilon\right)
\le~
2e^{-\tau_7}+(np)^{-2}~.
\end{align*}
Next, for $np=O(1)$ and any $\tau_5\to\infty,$ we have, for any $\tau_7>0$ and any $\epsilon\in\left(
\frac{2C_3K_2B^2}{C_4\theta}
\left(
\sqrt{\frac{\tau_5}{nq}}+
\sqrt{\frac{\tau_7}{nq}}
\right),~1\right)$,
\begin{align*}
\mathbb P\left(\left\langle v_1(X_{\text{supp}(v)}), \frac{1}{\|v\|_2}v\right\rangle^2\le1-\epsilon\right)
\le
2e^{-\tau_7}
+
\left(\frac{e np}{\lceil\tau_5\rceil}\right)^{\lceil\tau_5\rceil}~.
\end{align*}
See the proof for the definitions of positive constants $C_3$ and $C_4$.
\end{corollary}

See Section \ref{quad_recov_strong_proof} for the proof. 

\begin{remark}
    Notice if we are interested in the optimal regime of strong recovery ($\epsilon \to 0$), then, under our assumption $q\gtrsim p\log(n),$ we have $\frac{p}{q}\lesssim\frac{1}{\log n}\to0$ and $\frac{\max \{\tau_5, \tau_7 \}}{nq}\lesssim\frac{\max \{\tau_5, \tau_7 \}}{\log(n)}\to0$ if we choose $1\ll\max \{\tau_5, \tau_7 \}\ll\log(n)$. Thus, $\epsilon$ can be chosen to be arbitrarily small at the bottom range of its interval while still ensuring the tail bounds converge to $0.$
\end{remark}

We can also achieve weak/strong recovery for a small-ball spike prior, but we defer this result to Corollary \ref{cor:DT_recovery} as it requires a guarantee on Diagonal Thresholding established in Section \ref{section:DT}.

\subsection{Interpretation of Sparsity Conditions}\label{quad_form_interpretation}

For brevity, let us take a closer look at the interval $\mathcal L_R$\footnote{Analysis on $\mathcal L_D$ as well as Assumption \ref{spike_small_ball_assumption} is almost identical up to constants and tuning parameter $\tau_4$.} under Assumption \ref{bounded_spike_assumption} to see how $l$ can be chosen to balance sparsity and runtime. 

\subsubsection{Infinite Spike Support Size}
When $np\to\infty$, $\|v\|_0=np+k$ where $k\in[-\log(np)\sqrt{np}, \log(np)\sqrt{np}]$ by Lemma \ref{fixedmasklemma}. Thus, for large enough $n$, we find
$$\mathcal L_R= \left[\frac{C_2K_2^2B^4}{\theta^2}\cdot\frac{np^2\log(n)}{q}, np+k-1\right]\supset\left[\frac{C_2K_2^2B^4}{\theta^2}\cdot\frac{np^2\log(n)}{q}, \frac{1}{2}np\right]~.$$
Thus, we need $q\gtrsim\frac{C_2K_2^2B^4}{\theta^2}{p\log(n)}$\footnote{We include the $\gtrsim$ to ensure an integer exists in the original interval}. When $p>\frac{\theta^2}{C_1K_2^2B^4}\frac{1}{\log(n)},$ the algorithm fails as this would require $q$ to be greater than $1.$ In the regime $\frac{1}{\sqrt{n\log(n)}}\lesssim p \lesssim \frac{1}{\log(n)},$ the algorithm can succeed, but only subexponentially as the left side of the interval $\mathcal L_R$ grows with $n.$ Thus, the interval may be nonempty, but $l$ must also grow with $n.$ For $p\lesssim\frac{1}{\sqrt{n\log(n)}},$ the left side of $\mathcal L_R$ is $O(1/q),$ and therefore we can perform a polynomial time algorithm for a suitable choice of $q$. 

\subsubsection{Constant Spike Support Size}

In this regime, the spike is as sparse as possible with $np=\gamma=\Theta(1)$. Since $\|v\|_0\sim\operatorname{Bin}(n,p)$, we have $\|v\|_0=O_P(1)$; that is, the realized support size does not grow with $n$. Therefore,

$$
\mathcal L_R= \left[\frac{C_2K_2^2B^4}{\theta^2}\cdot\frac{\gamma^2\log(n)}{nq}, \max\{\|v\|_0-1,1\}\right]~.
$$

Thus, whenever the interval is nonempty, its upper bound is $\Theta(1)$, regardless of the particular realization of the support. We can therefore always choose $l=\Theta(1)$, meaning that the resulting algorithm runs in polynomial time. However, notice this regime restricts $q$ as it cannot fall below the spectral threshold of $\frac{\log(n)}{n}$ (but can meet it with $
q\gtrsim\frac{C_2K_2^2B^4\gamma^2}{\theta^2\max\{(\|v\|_0-1),1\}}\cdot\frac{\log(n)}{n}
$ whenever $\|v\|_0\ge1$).

$$
$$

\subsection{Quadratic Form Control for Sparse, Sub-Gaussian Noise}

Prior to proving Theorems~\ref{quad_form_distinguish_thm} and~\ref{quad_form_recovery_thm}, we establish the following lemma to control the quadratic form involving our sparse noise matrix. 

\begin{lemma}\label{noise_quadratic_form_control}
Given deterministic vectors $a,b\in\mathbb R^n$, we have
\[    \mathbb P (|a^TYb|\ge t)
    \le
    2\exp\left(-c\frac{nq}{K_2^2\|a\|_2^2\|b\|_2^2}t^2 \right),
   \]
for $t>0$ where $K_2=\max_{{k,j}\in[n^2]}\|W_{kj}\|_{\psi_2}\ge \max_{{k,j}\in[n^2]}\|(W\odot A)_{kj}\|_{\psi_2}$ and $c>0$ is an absolute constant. Additionally, the same bound holds for $Y^{(i)}$ where the $i$-th row and column is removed from $Y$.
\end{lemma}

\begin{proof}
We follow a similar procedure to Theorem 4.4.3 in \cite{vershynin2018high}, where $c$ and $C$ denote positive absolute constants. Since $Y$ is symmetric, we regroup the quadratic form over the independent upper-triangular entries:
\begin{align*}
    a^TYb
    =
    \sum_{i<j}(a_i b_j+a_j b_i)Y_{ij}
    +
    \sum_i a_i b_iY_{ii}.
\end{align*}
By Proposition 2.7.1 in \cite{vershynin2018high}, the entries $Y_{ij}=\frac{1}{\sqrt{nq}}(W\odot A)_{ij}$ are sub-Gaussian with
\(    \|Y_{ij}\|_{\psi_2}
    \le
    \frac{K_2}{\sqrt{nq}}.
   \)
The terms in the above sum are independent, since they depend on distinct entries of the upper triangle of $Y$. Thus, another application of Proposition 2.7.1 in \cite{vershynin2018high} gives
\begin{align*}
\|a^TYb\|_{\psi_2}^2
&\le
\frac{CK_2^2}{nq}
\left(
\sum_{i<j}(a_i b_j+a_j b_i)^2
+
\sum_i(a_i b_i)^2
\right) 
\le \frac{3CK_2^2}{nq}\|a\|_2^2\|b\|_2^2~,
\end{align*}
where we used the elementary inequality $(x+y)^2\le 2x^2+2y^2$
and
\(    \sum_i(a_i b_i)^2
    \le
    \|a\|_2^2\|b\|_2^2.
   \)
Applying Proposition 2.2.6(i) of \cite{vershynin2018high} yields
the final result.

The same argument applies to $Y^{(i)}$, since removing the $i$-th row and column leaves a symmetric matrix whose upper-triangular entries remain independent and satisfy $\max_{\{k,j\}\in[n^2]}
\|((W\odot A)^{(i)})_{kj}\|_{\psi_2}
\le K_2.$
\end{proof}

\subsection{Quadratic Thresholding Distinguishability Proof: Small-Ball Spike}\label{section:subg_dist}

This proof follows closely with Proof \ref{quad_dist_proof}, but adapts the procedure for sub-Gaussian spike entries.

\begin{proof-of-theorem}[\ref{thm:quad_subG_distinguish}] We show one can distinguish by using the quadratic thresholding procedure, but for a limited region of sparsity.  \\

{\textit{Type I Error.}} Assume $X \sim \mathbb{P}$. Consider the vector $\bar{v} \in \mathcal{I}_{n,l}$ that correctly guesses the signs of the top $l$ entries of the true spike $v.$ Then, we have
\begin{align*}
    \bar{v}^{T}X\bar{v} =~& \frac{\theta}{np}\langle \bar{v}, v \rangle^2 + \bar{v}^{T}Y\bar{v} ~.
\end{align*}
Therefore, we require a high-probability lower bound on $\langle \bar{v}, v \rangle = \sum_{i \in \text{supp}(\bar{v})} |v_i|$.

By the small-ball assumption, setting $a=1/(2C_s)$ gives $\mathbb{P}\left\{|\widetilde v_i|>a\right\}\ge 1/2,$
where the inequality follows by substituting $a$ into $\mathbb{P}(|\widetilde v_i|\le a)\le C_sa=1/2.$ Conditional on $\|v\|_0$, define $ N_a:=\left|\left\{i\in\operatorname{supp}(v):|\widetilde v_i|>a\right\}\right|.$ Assuming the spike entries are independent, a binomial Chernoff bound gives, for a sufficiently large universal constant $c_6$, $\mathbb{P}\left(N_a<{\|v\|_0}/{4}\right)\le \exp(-c_6\|v\|_0).$
Therefore, with probability at least $1-\exp(-c_6\|v\|_0)$,
\begin{align*}
\sum_{i\in\operatorname{supp}(v)}|v_i|
\geq
aN_a
\geq
\frac{\|v\|_0}{8C_s}.
\end{align*}
Since $\bar{v}$ contains the $l$ largest entries of $v$ in magnitude and $l\leq \|v\|_0$, the sum of the $l$ largest absolute entries is at least the corresponding fraction of the sum over the entire support. Hence,
\begin{align*}
\sum_{i\in\operatorname{supp}(\bar{v})}|v_i|
\geq
\frac{l}{K}\sum_{i\in\operatorname{supp}(v)}|v_i|
\geq
\frac{l}{8C_s}.
\end{align*}
Thus, with probability at least $1-\exp(-c_6\|v\|_0)$, we have:
\begin{align}
\langle \bar{v}, v \rangle \geq \frac{l}{8C_s}~.\label{eqn:QT_subG_innerprod_lower_bound}
\end{align}
Squaring this relation guarantees that with high probability, the signal term satisfies:
\begin{align*}
\frac{\theta}{np}\langle \bar{v}, v \rangle^2 \geq \frac{\theta}{np} \left(\frac{l}{8C_s}\right)^2 = \frac{\theta l^2}{64 C_s^2 np} =: 2t~.
\end{align*}
Thus, we set $\mathcal{T} := \max\limits_{w \in \mathcal{I}_{n,l}} w^T X w$ and $t := \frac{\theta l^2}{128 C_s^2 np}$ to see:
\begin{align*}
\mathbb{P}\left(\mathcal{T} < t\right) \leq~& \mathbb{P}\left(\bar{v}^{T}X\bar{v} < t\right) \\
\leq~& \mathbb{P}\left(2t + \bar{v}^{T}Y\bar{v} < t\right) + \exp(-c_6\|v\|_0) \\
=~& \mathbb{P}\left(\bar{v}^{T}Y\bar{v} < -t\right) + \exp(-c_6\|v\|_0) \\
\leq~& 2\exp\left(-c\frac{nq}{K_2^2 |\bar{v}|_0^2}t^2\right) + \exp(-c_6\|v\|_0) \\
=~& 2\exp\left(-c\frac{nq}{K_2^2 l^2}\left(\frac{\theta l^2}{128 C_s^2 np}\right)^2\right) + \exp(-c_6\|v\|_0) \\
=& 2\exp\left(-c'\frac{q \theta^2 l^2}{K_2^2 C_s^4 np^2}\right) + \exp(-c_6\|v\|_0) ~,
\end{align*}
where we applied Lemma~\ref{noise_quadratic_form_control} with $\|\bar{v}\|_0 = l$.

{\textit{Type II Error.}} Next, we consider pulling from the null distribution $\mathbb{Q}$ and mistakenly labeling it as the planted distribution. We consider any vector $\tilde{v} \in\mathcal{I}_{n,l}$ where $\|\tilde{v}\|_0 = l$. Using Lemma~\ref{noise_quadratic_form_control}, we see:
\begin{align*}
\mathbb{Q}\left(\tilde{v}^T X \tilde{v} \ge t\right) = \mathbb{Q}\left(\tilde{v}^T Y \tilde{v} \ge t\right) \le 2\exp\left(-c'\frac{q \theta^2 l^2}{K_2^2 C_s^4 np^2}\right) ~.
\end{align*}
Taking a union bound over the set $\mathcal{I}_{n,l},$ we find:
\begin{align*}
\mathbb{Q}\left(\mathcal{T} \ge t\right) \leq& 2|\mathcal{I}_{n,l}|\exp\left(-c'\frac{q \theta^2 l^2}{K_2^2 C_s^4 np^2}\right) \\
=& 2\binom{n}{l}2^{l}\exp\left(-c'\frac{q \theta^2 l^2}{K_2^2 C_s^4 np^2}\right) \\
\leq& (2n)^{l+1}\exp\left(-c'\frac{q \theta^2 l^2}{K_2^2 C_s^4 np^2}\right) \\
=& \exp\left(-\left(\frac{c' q \theta^2 l^2 - K_2^2 C_s^4 np^2(l+1)\log(2n)}{K_2^2 C_s^4 np^2}\right)\right) \\
\leq& \exp\left(-\left(\frac{c' q \theta^2 l^2 - 4K_2^2 C_s^4 \cdot l \cdot np^2\log(n)}{K_2^2 C_s^4 np^2}\right)\right) ~,
\end{align*}
where we used the fact that $l > 1$ in the final inequality.

{\textit{Conclusion.}} Noticing that the Type I and Type II error tails are controlled by the same exponential term, with the Type I error having the additional term $\exp(-c_6\|v\|_0)$ from the small-ball argument, we factor out $l \ge 1$ from the numerator exponent to establish the lower bound on $l$ that allows both the Type I and Type II errors to converge asymptotically to $0$. Thus, for sufficiently large $C_5$, we require:
\begin{align*}
l \in \left[\frac{C_5K_2^2C_s^4}{\theta^2}\cdot \frac{np^2\log(n)}{q}, \|v\|_0\right] ~,
\end{align*}
provided that $\|v\|_0\to\infty$ so that $\exp(-c_6\|v\|_0)=o(1)$.
\end{proof-of-theorem}

\subsection{Quadratic Thresholding Support Recovery: Small-Ball Spike}\label{section:QT_subG_support_recov}

We follow a very similar proof approach to Section \ref{quad_recov_proof_exact} but update for Assumption \ref{spike_small_ball_assumption}.

\begin{proof-of-theorem}[\ref{thm:QT_subG_support_recovery}]
Before proceeding with the proof, we first note the two high-probability properties of the sub-Gaussian spike vector that will replace the corresponding lower bounded spike assumptions used in the proof of Theorem~\ref{quad_form_recovery_thm}.

First, we repeat the analysis leading to \eqref{eqn:QT_subG_innerprod_lower_bound}, but we need to adapt it based on this current procedure. Define the event

$$
\mathcal E_{\mathrm{SB}}
:=
\left\{
\left|
\left\{
j\in\operatorname{supp}(v):
|v_j|>\frac{1}{2C_s}
\right\}
\right|
\geq \frac{\|v\|_0}{4}
\right\}.
$$

On the event $\mathcal E_{\mathrm{SB}}$, deleting any single coordinate can remove at most one of the coordinates satisfying $|v_j|>1/(2C_s)$. Therefore, for sufficiently large $\|v\|_0$,

$$
\left|
\left\{
j\in\operatorname{supp}(v^{(i)}):
|v_j|>\frac{1}{2C_s}
\right\}
\right|
\geq \frac{\|v\|_0}{4}-1
\geq \frac{\|v\|_0}{8}
$$

for every $i\in[n]$. Consequently,

$$
\sum_{j\in\operatorname{supp}(v^{(i)})}|v_j|
\geq
\frac{\|v\|_0}{16C_s}.
$$

Since $\bar v^{(i)}$ contains the $l$ largest entries of $v^{(i)}$ in magnitude and $l\leq \|v\|_0-1$, the sum of these $l$ entries is at least the corresponding fraction of the total absolute mass. Hence, uniformly over $i\in[n]$,
\begin{align}
\langle \bar v^{(i)},v^{(i)}\rangle
\sum_{j\in\operatorname{supp}(\bar v^{(i)})}|v_j|
\geq
\frac{l}{\|v\|_0-1}
\sum_{j\in\operatorname{supp}(v^{(i)})}|v_j|
\geq
\frac{l}{16C_s}.
\label{eqn:QT_subG_LOO_innerprod_lower_bound}
\end{align}
Thus, on $\mathcal E_{\mathrm{SB}}$, the lower bound in \eqref{eqn:QT_subG_LOO_innerprod_lower_bound} holds simultaneously for every $i\in[n]$.

Second, the small-ball condition with a union bound over $\|v\|_0$ implies
\begin{align}
\mathbb P\left(\min_{i\in\operatorname{supp}(v)}|v_i|\leq \frac{1}{\tau_3}\right)
\leq \frac{C_s\|v\|_0}{\tau_3}~.\label{small_ball_callback}
\end{align}
Thus, whenever $\tau_3/\|v\|_0\to\infty$, all nonzero entries of the spike are simultaneously bounded below by $1/\tau_3$ with probability tending to one.

We begin first by establishing control over the quadratic form for each $i\in[n]$. Once we have established this control, we will condition on this high probability event. We use the indexing $X^{(i)}$ to represent a matrix identical to $X$ but with zeros in the $i$-th row and column.
Let $\bar v^{(i)}$ be the vector with support size $l$ and entries drawn from ${\pm 1, 0}$ correctly guessing the signs of $l$ largest entries of $v^{(i)}$. Then, accounting for the possibility of $i\in\text{supp}(v)$, we see $l\leq \|v\|_0-1$. Thus, we have
\begin{align*}
\bar v^{(i)T}X^{(i)}\bar v^{(i)}=~&\frac{\theta}{np}\langle \bar v^{(i)}, v^{(i)}\rangle^2 + \bar v^{(i)T}Y^{(i)}\bar v^{(i)} \
\geq \frac{\theta l^2}{256C_s^2np} + \bar v^{(i)T}Y^{(i)}\bar v^{(i)}~.
\end{align*}
Let $T^{(i)}=v^{*(i)T}X^{(i)}v^{*(i)}$ where $v^{*(i)}:=\arg\max\limits*{w\in\mathcal I_{n,l}}w^TX^{(i)}w$. Because the value of the maximum does not change, we importantly select $v^{*(i)}$ such that $(v^{*(i)})*i=0$, which will ensure $\langle v^{*(i)}, v^{(i)} \rangle =\langle v^{*(i)}, v \rangle$. Establishing our thresholding value as $t=\frac{\theta l^2}{512C_s^2np},$
we see
\begin{align*}
\mathbb P(T^{(i)}<t)
\leq~ \mathbb P\left(\frac{\theta l^2}{256C_s^2np}+ \bar v^{(i)T}Y^{(i)}\bar v^{(i)}<t\right)
\leq~ 2\exp\left(-c\frac{q\theta^2 l^2}{262144 C_s^4K_2^2np^2}\right)~,
\end{align*}
where we used Lemma \ref{noise_quadratic_form_control}. Seeing as we must have control over this term for each iteration of $i\in[n],$ we take a union bound to see
\begin{align*}
\mathbb P\left(\exists i\in[n] : T^{(i)}<t\right)
\leq n\mathbb P(T^{(i)}<t)
\leq \exp\left(-c\left(\frac{q\theta^2 l^2- C C_s^4K_2^2np^2\log(n)}{C_s^4K_2^2np^2}\right)\right)~.
\end{align*}
Therefore, as long as $q\theta^2l^2\gtrsim CC_s^4 K_2^2np^2\log(n)$, we see this probability term converges to $0.$ Going forward, we condition on the high probability event

$$
    \mathcal B_1 := \{T^{(i)}\geq t \quad \forall i\in[n]\}.
$$

Now, we control the noise uniformly over all vectors in $\mathcal I_{n,l}$. Define

$$
    \mathcal B_2:=\left\{
    \max_{i\in[n]}\max_{w\in\mathcal I_{n,l}}
    |w^TY^{(i)}w|
    \leq \frac{\theta l^2}{1024C_s^2np}
    \right\}.
$$

For any fixed $i\in[n]$ and $w\in\mathcal I_{n,l}$, Lemma \ref{noise_quadratic_form_control} gives
\begin{align*}
\mathbb P\left(
|w^TY^{(i)}w|
\geq \frac{\theta l^2}{1024C_s^2np}
\right)
&\leq
2\exp\left(
-c\frac{\theta^2ql^2}{1048576C_s^4K_2^2np^2}
\right).
\end{align*}
Since $|\mathcal I_{n,l}|=\binom nl2^l\leq(2n)^l$, a union bound gives
\begin{align*}
\mathbb P(\mathcal B_2^c)
\leq
2n(2n)^l
\exp\left(
-c\frac{\theta^2ql^2}{1048576C_s^4K_2^2np^2}
\right) \
\leq
\exp\left(
-c\left(
\frac{\theta^2ql^2-C_2C_s^4K_2^2np^2l\log(n)}
{C_s^4K_2^2np^2}
\right)
\right),
\end{align*}
for a sufficiently large $C_2$. Thus, as long as $\theta^2ql^2\gtrsim C_2C_s^4K_2^2np^2l\log(n),$ we have $\mathbb P(\mathcal B_2)\to1$.

Therefore, on the event $\mathcal B_1\cap\mathcal B_2$, we have
\begin{align*}
v^{*(i)T}X^{(i)} v^{*(i)}
=\frac{\theta}{np}\left\langle v^{*(i)}, v^{(i)}\right\rangle^2
+v^{*(i)T}Y^{(i)}v^{*(i)}
\geq \frac{\theta l^2}{512C_s^2np}
\end{align*}
and

$$
    \left|v^{*(i)T}Y^{(i)}v^{*(i)}\right|
    \leq \frac{\theta l^2}{1024C_s^2np}.
$$

This gives us an immediate lower bound on the term $\langle v^{*(i)}, v \rangle$. Notice by our choice of $v^{*(i)},$ we have $\langle v^{*(i)}, v^{(i)} \rangle =\langle v^{*(i)}, v \rangle$ and thus, on $\mathcal B_1\cap\mathcal B_2$,
\begin{align*}
\frac{\theta}{np}\left\langle v^{*(i)}, v \right\rangle^2
=T^{(i)}-v^{*(i)T}Y^{(i)}v^{*(i)}
&\geq \frac{\theta l^2}{512C_s^2np}
-\frac{\theta l^2}{1024C_s^2np}
=\frac{\theta l^2}{1024C_s^2np}.
\end{align*}
Consequently, $\left\langle v^{*(i)}, v \right\rangle^2
    \geq {l^2}/({1024C_s^2}).$ Thus,
\begin{align}
\mathbb P \left(\exists i\in[n] : \langle v^{*(i)}, v \rangle^2\leq \frac{l^2}{1024C_s^2} \right)
\leq
2\exp\left(-c\left(\frac{q\theta^2 l^2-C_2C_s^4K_2^2np^2l\log(n)}{C_s^4K_2^2np^2}\right)\right)~,
\label{recovery_inner_prod_lower_bound}
\end{align}
where we have redefined $C_2$ to be sufficiently large enough to absorb the constants from the preceding bounds on $\mathbb P(\mathcal B_1^c)+\mathbb P(\mathcal B_2^c)$.

Finally, now that we have a computed lower bound for the number of spiked entries, we take our computed $v^{*(i)}$ and multiply it back into the $i$-th row of $X$ to determine whether the entry is spiked or not.
By the small-ball condition \eqref{small_ball_callback}, if $i\in$ supp($v$), then we can use the above bound to see
\begin{align*}
\left|(Xv^{*(i)})_i\right|
=\left|\frac{\theta}{np}v_i\langle v^{*(i)}, v\rangle+e_i^TYv^{*(i)}\right|
\geq \left|\frac{\theta}{np}\frac{1}{\tau_3}\frac{l}{32C_s}\right|-\left|e_i^TYv^{*(i)}\right|
=\frac{\theta l}{32C_s\tau_3np}-\left|e_i^TYv^{*(i)}\right|~.
\end{align*}
If $i$ is not in the support, then an inequality is not necessary and we are left only with the latter term involving the noise. Therefore, in either case, we must control the latter term. Notice, this step highlights the importance of our leave-one-out method, as it allows us to introduce independence: $v^{*(i)}$ is created without any dependence on the $i$-th row and column of $X$ (and therefore $Y$). Thus, $e_i^TY$ and $v^{*(i)}$ are completely independent of each other.
Thus, we proceed with an application of Lemma
\ref{noise_quadratic_form_control} one final time, apply a union bound over all $i\in[n]$, and arrive at the following:
\begin{align*}
\mathbb P\left(|e_i^TYv^{*(i)}|\leq\frac{\theta l}{96C_s\tau_3np} \quad \forall i\in[n]\right)
\geq 1-\exp\left(-\left(\frac{c\theta^2ql-C_3C_s^2K_2^2\tau_3^2np^2\log(n)}{9126C_s^2K_2^2\tau_3^2np^2}\right)\right)~,
\end{align*}
for sufficiently large $C_3$, which converges to $1$ when $l\gtrsim  \frac{C_3C_s^2K_2^2\tau_3^2}{c\theta^2}\frac{np^2\log(n)}{q}~.$ Combining this tail bound with \eqref{recovery_inner_prod_lower_bound}, we find the following tail:
\begin{align*}
1-&\exp\left(-\left(\frac{c\theta^2ql-C_3C_s^2K_2^2\tau_3^2np^2\log(n)}{9126C_s^2K_2^2\tau_3^2np^2}\right)\right)
-\exp\left(-\left(\frac{cq\theta^2 l^2-C_2C_s^4K_2^2np^2l\log(n)}{C_s^4K_2^2np^2}\right)\right)\\
&\geq
1-\exp\left(-\left(\frac{\theta^2ql-C_7C_s^2K_2^2\tau_3^2np^2\log(n)}{9126C_s^2K_2^2\tau_3^2np^2}\right)\right)~,
\end{align*}
by upper bounding the smaller tail by the larger one and choosing sufficiently large absolute constant $C_7$, which in turn gives us our overarching condition for $l$:
\begin{align*}
l \in \left[
C_7\frac{C_s^2\tau_3^2}{\theta^2}\cdot\frac{np^2\log(n)}{q},
\|v\|_0-1
\right]~.
\end{align*}
Therefore, in conclusion, we find
\begin{align*}
|(Xv^{*(i)})_i|\geq \frac{\theta l}{48C_s\tau_3np}\quad \text{ for } i \in \text{supp}(v) \text{ with sgn}((Xv^{*(i)})_i)=\text{sgn}(v_i\langle v^{*(i)},v \rangle)~,
\end{align*}
and
\begin{align*}
|(Xv^{*(i)})_i|\leq \frac{\theta l}{96C_s\tau_3np}\quad \text{ for }i\notin\text{supp}(v)~,
\end{align*}
with probability
\begin{align*}
1-\exp\left(-\left(\frac{\theta^2ql-C_7C_s^2K_2^2\tau_3^2np^2\log(n)}{9126C_s^2K_2^2\tau_3^2np^2}\right)\right)-\frac{C_s\|v\|_0}{\tau_3}-\exp(-c_6\|v\|_0)~.
\end{align*}
Seeing as $\tau_3\gg\|v\|_0$ is necessary for the second term to be $o(1)$, we choose $\tau_4=\tau_3/\|v\|_0~.$

\paragraph{Case 2: $\|v\|_0=1$.} In this case, leave-one-out will no longer work, and one would threshold similar to \cite{ding2022subexponentialtimealgorithmssparsepca}. One would achieve a similar window (up to constants) for $l$, with the upper bound being changed to $\|v\|_0=1$ and the lower bound accounting for constants coming from the noise sparsity and small-ball assumption. For brevity, we exclude the details. 
\end{proof-of-theorem}

\subsection{Quadratic Thresholding Exact Recovery Proof: Rademacher}\label{section:Rademacher_sign_recovery}
Unlike \cite{ding2022subexponentialtimealgorithmssparsepca}, the sign recovery is not straightforward due to our leave one out procedure. Thus, we implement a procedure to ensure the estimates of each leave-one-out step are consistent in sign. 

\begin{proof-of-corollary}[\ref{corollary:QT_exact}]
We now proceed to recover the signs of the entries of $v$.

Define $s_i
:=
\operatorname{sgn}\left(
\langle v^{*(i)},v\rangle
\right).$
From \eqref{sign_start}, for every $i\in\operatorname{supp}(v)$, $\operatorname{sgn}((Xv^{*(i)})_i)
    =
    \operatorname{sgn}(v_i)s_i.$
Thus, it remains only to recover the relative signs $s_i$.

To accomplish this, we establish one additional high probability event. Define
\begin{align*}
    \mathcal B_3
    :=
    \left\{
    \max_{w,u\in\mathcal I_{n,l}}
    \left|w^TYu\right|
    \le
    \frac{\theta l^2}{8B^2np}
    \right\}.
\end{align*}
Therefore, Lemma
\ref{noise_quadratic_form_control} gives
\begin{align*}
    \mathbb P\left(
    \left|w^TYu\right|
    \ge
    \frac{\theta l^2}{8B^2np}
    \right)
    \le
    2\exp\left(
    -c\frac{q\theta^2l^2}
    {64K_2^2B^4np^2}
    \right).
\end{align*}
Union bounding over $w,u\in\mathcal I_{n,l}$ gives
\begin{align*}
    \mathbb P(\mathcal B_3^c)
    \leq
    2(2n)^{2l}
    \exp\left(
    -c\frac{q\theta^2l^2}
    {64K_2^2B^4np^2}
    \right)
    \leq
    \exp\left(
    -\left(
    \frac{q\theta^2l^2
    -c_4K_2^2B^4np^2l\log(n)}
    {64K_2^2B^4np^2}
    \right)
    \right),
\end{align*}
for a sufficiently large $c_4$. Therefore, as long as $q\theta^2l^2
\gtrsim
c_4K_2^2B^4np^2l\log(n),$ we have $\mathbb P(\mathcal B_3)\to1$.

Importantly, we do not require $v^{*(i)}$ and $v^{*(j)}$ to be independent of $Y$ in this step. The event $\mathcal B_3$ controls $w^TYu$ uniformly over every possible pair $w,u\in\mathcal I_{n,l}$, and therefore it automatically controls the particular random vectors $v^{*(i)}$ and $v^{*(j)}$ regardless of their dependence on $Y$.

Going forward, we continue to condition on the high probability event
$\mathcal B_1\cap\mathcal B_2\cap\mathcal B_3$ from the proof in Section \ref{quad_recov_proof_exact}. For any $i,j\in[n]$, define the pairwise comparison statistic $D_{ij}
:=
v^{*(i)T}Xv^{*(j)}.$ We then have
\begin{align*}
    D_{ij}
    =
    \frac{\theta}{np}
    \langle v^{*(i)},v\rangle
    \langle v^{*(j)},v\rangle
    +
    v^{*(i)T}Yv^{*(j)}
    =
    s_is_j
    \frac{\theta}{np}
    \left|
    \langle v^{*(i)},v\rangle
    \right|
    \left|
    \langle v^{*(j)},v\rangle
    \right|
    +
    v^{*(i)T}Yv^{*(j)}.
\end{align*}
From \eqref{QT_recov_C_2_tail}, for all $i\in[n],$
\[
\left|
\langle v^{*(i)},v\rangle
\right|
\geq
\frac{l}{2B}
\implies \left|
    \frac{\theta}{np}
    \langle v^{*(i)},v\rangle
    \langle v^{*(j)},v\rangle
    \right|
    \geq
    \frac{\theta l^2}{4B^2np}~.
\]
On the other hand, by the definition of $\mathcal B_3$,
\begin{align*}
    \left|
    v^{*(i)T}Yv^{*(j)}
    \right|
    \leq
    \frac{\theta l^2}{8B^2np}.
\end{align*}
Therefore, the signal contribution is strictly larger than the noise contribution, and hence $\operatorname{sgn}(D_{ij}) = s_is_j.$ Since the support of $v$ has already been recovered exactly, we may choose any reference index $r\in\operatorname{supp}(v).$
We then define $\widehat s_i := \operatorname{sgn}(D_{ir}) = \operatorname{sgn} \left( v^{*(i)T}Xv^{*(r)} \right),
$ for $i\in\operatorname{supp}(v).$ By the preceding argument, $\widehat s_i = s_is_r.$
Thus, $\widehat s$ recovers vector of orientations $s$ up to the global sign $s_r$.

Finally, define the recovered vector $\widehat v$ by
\begin{align*}
    \widehat v_i
    :=
    \begin{cases}
    \widehat s_i
    \operatorname{sgn}((Xv^{*(i)})_i),
    & i\in\operatorname{supp}(v),\\
    0,
    & i\notin\operatorname{supp}(v).
    \end{cases}
\end{align*}
For $i\in\operatorname{supp}(v)$, we have $\widehat v_i
    =
    (s_is_r)
    \operatorname{sgn}(v_i)s_i
    =
    s_rv_i.$
Therefore, $\widehat v=s_rv,$ and hence $R(\widehat v,v)=1.$

Combining the tail bounds for $\mathcal B_1^c$, $\mathcal B_2^c$, the row-noise event, and $\mathcal B_3^c$, we find that the overall probability of exact recovery converges to $1$ under the condition
\begin{align*}
    l \in
    \left[
    \frac{C_8K_2^2B^4}{\theta^2}
    \cdot
    \frac{np^2\log(n)}{q},
    \|v\|_0-1
    \right]~,
\end{align*}
for sufficiently large $C_8$ chosen to subsume the different absolute constants appearing in the preceding tail bounds.
\end{proof-of-corollary}

\subsection{Weak/Strong Recovery Proof}\label{quad_recov_strong_proof}

\begin{proof-of-corollary}[\ref{quad_form_weak_strong_corollary}]
We follow a similar approach to \cite{ding2022subexponentialtimealgorithmssparsepca}, but adapt their probability bounds to reflect the distribution of our noise matrix. The following holds for vectors $v$ with lower bounded entries (Assumption \ref{bounded_spike_assumption}).

Let $S=\operatorname{supp}(v)$ and write $k:=|S|=\|v\|_0$. Observe that
\begin{align*}
v_1(X_S)^TY_Sv_1(X_S)+\frac{\theta}{np}\langle v_1(X_S),v\rangle^2
=~&v_1(X_S)^TX_Sv_1(X_S)\\
\ge~&\frac{1}{\|v\|_2^2}v^TX_Sv\\
=~&\frac{1}{\|v\|_2^2}v^TY_Sv
+\frac{1}{\|v\|_2^2}\frac{\theta\|v\|_2^4}{np}.
\end{align*}
Therefore,
\begin{align*}
\left\langle v_1(X_S),\frac{1}{\|v\|_2}v\right\rangle^2
\ge~&
1-\frac{np}{\theta\|v\|_2^2}
\left(
v_1(X_S)^TY_Sv_1(X_S)
-\frac{1}{\|v\|_2^2}v^TY_Sv
\right)\\
\ge~&
1-\frac{2np}{\theta\|v\|_2^2}\|Y_S\|,
\end{align*}
where we used
$
    \left|v_1(X_S)^TY_Sv_1(X_S)\right|\le \|Y_S\|
    \text {and }
    \left|\frac{1}{\|v\|_2^2}v^TY_Sv\right|\le \|Y_S\|.
$

Consequently,
\begin{align*}
\mathbb P\left(
\left\langle v_1(X_S),\frac{1}{\|v\|_2}v\right\rangle^2
\le 1-\epsilon
\right)
\le
\mathbb P\left(
\|Y_S\|
\ge \frac{\theta\|v\|_2^2}{2np}\epsilon
\right).
\end{align*}

Now, by Assumption \ref{bounded_spike_assumption}, we have
\begin{align}
\|v\|_2^2\ge \frac{1}{B^2}\|v\|_0.
\label{QT_strongrecov_subg_caveat}
\end{align}
Notice we do not assume $\|v\|_0=np$. Rather, we consider two different cases, which are discussed in Section \ref{spike_support_regimes_section}. If $np\to\infty,$ then $\|v\|_0=np+k$ where $k\in[-\log(np)\sqrt{np},\log(np)\sqrt{np}]$ by Lemma \ref{fixedmasklemma}. Thus, with high probability, $ \|v\|_0\ge\frac{1}{2}np.
$ If $np=O(1)$, then on the event\footnote{$\|v\|_0=0$ would trivially attain exact recovery with $v^*=0$, so it is excluded from this weak/strong recovery analysis.} $\|v\|_0\ge1$ we have $\|v\|_2^2/(np)\ge 1/(B^2np)$, which is bounded below by a positive constant since $np=O(1)$. Hence, in either asymptotic setting, ${\|v\|_2^2}/{np}\ge {C_4}/{B^2} $
for some positive constant $C_4$. Thus,
\begin{align*}
\mathbb P\left(
\left\langle v_1(X_S),\frac{1}{\|v\|_2}v\right\rangle^2
\le1-\epsilon
\right)
\le
\mathbb P\left(
\|Y_S\|
\ge \frac{\theta C_4\epsilon}{2B^2}
\right).
\end{align*}

Conditional on the support $S$, the entries of $Y_S$ retain the sub-Gaussian noise assumptions. Using a quadratic form bound from Lemma \ref{noise_quadratic_form_control}, one can apply an $\varepsilon$-net argument similar to the standard operator-norm argument in \cite[Section 4.4]{vershynin2018high} to achieve
$$
    \mathbb P\left(
    \|Y_S\|>
    \frac{C_3K_2}{\sqrt{nq}}
    \left(\sqrt{\|v\|_0}+\sqrt{u}\right)
    \,\middle|\,S
    \right)
    \le 2e^{-u},
$$
for every $u>0$, where $C_3$ is an absolute constant. Therefore,
\begin{align}
\mathbb P\left(
\left\langle v_1(X_S),\frac{1}{\|v\|_2}v\right\rangle^2
\le1-\epsilon
~\middle|~S
\right)
\le
2e^{-u}\label{strong_general_u_tail}
\end{align}
whenever
$$
    \frac{C_3K_2}{\sqrt{nq}}
    \left(\sqrt{\|v\|_0}+\sqrt{u}\right)
    \le
    \frac{\theta C_4\epsilon}{2B^2}.
$$
We now consider the two asymptotic regimes separately.

\medskip
\noindent
\textbf{Case 1: $np\to\infty$.}

By Lemma \ref{fixedmasklemma}, $|k|=O(\log(np)\sqrt{np})$ with probability $1-(np)^{-2}$. Hence, with high probability,
$$
    \sqrt{\frac{\|v\|_0}{n}}
    =
    \sqrt{p+\frac{k}{n}}
    =
    \sqrt{p}(1+o(1)).
$$
Choose some $\tau_7\gg 1$ and $\epsilon$ sufficiently large such that
$$
    \epsilon>
    \frac{2C_3K_2B^2}{C_4\theta}
    \left(
    \sqrt{\frac{p}{q}}+
    \sqrt{\frac{\tau_7}{nq}}
    \right)(1+o(1)).
$$
Then, conditional on the support and on the event in Lemma \ref{fixedmasklemma}, the operator-norm bound gives us control over \eqref{strong_general_u_tail} with tail $2e^{-\tau_7}+(np)^{-2}.$


\medskip
\noindent
\textbf{Case 2: $np=O(1)$.}

Because $\|v\|_0$ is a binomial random variable, we see 
$$
    \mathbb{P}\!\left(\|v\|_0\ge \lceil\tau_5\rceil\right)
    \le
    \left(\frac{e np}{\lceil\tau_5\rceil}\right)^{\lceil\tau_5\rceil}
    \to0 \implies  \frac{\|v\|_0}{n}\le\frac{\tau_5}{n} ~~ \text{w.h.p.}
$$
Choosing $\tau_7\gg1$ and $\epsilon$ sufficiently large such that
$$
    \epsilon>
    \frac{2C_3K_2B^2}{C_4\theta}
    \left(
    \sqrt{\frac{\tau_5}{nq}}
    +
    \sqrt{\frac{\tau_7}{nq}}
    \right),
$$
we have control over \eqref{strong_general_u_tail}
 with tail $2e^{-\tau_7}
+
\left(({e np})/{\lceil\tau_5\rceil}\right)^{\lceil\tau_5\rceil}$




\end{proof-of-corollary}

\section{Diagonal Thresholding}\label{section:DT}

In order to achieve support recovery of a small-ball spike for $np=\Theta(1),$ we must turn toward Diagonal Thresholding as Quadratic Thresholding (Theorem \ref{thm:QT_subG_support_recovery}) requires $np\to\infty.$ Furthermore, Diagonal Thresholding also provides us access to sparsity regimes not covered in Quadratic Thresholding in both the lower bounded and small ball prior. Diagonal Thresholding is very similar to Quadratic Thresholding with $l=1$, but it instead relies solely on the behavior of the diagonal without an additional off-diagonal thresholding step.  

\subsection{Main Results}

\begin{figure}[ht]
    \begin{minipage}[t]{.48\textwidth}
        \begin{algorithm}[H]
\caption{Diagonal Thresholding Distinguishability}\label{alg:DT_dist}
\begin{algorithmic}
\STATE Inputs: $t>0$ and $X$ where $X\sim\mathbb P$ or $X\sim \mathbb Q$.
\STATE 1. Compute $\mathcal T =\max_{i\in[n]}|X_{ii}|$
\IF{$\mathcal T \ge t$}
\STATE return $\mathbb P$
\ELSE
\STATE return $\mathbb Q$
\ENDIF
\end{algorithmic}
\end{algorithm}
    \end{minipage}
    \hfill
    \begin{minipage}[t]{.48\textwidth}
        \begin{algorithm}[H]
\caption{Diagonal Thresholding Recovery}\label{alg:DT_recov}
\begin{algorithmic}
    \STATE Inputs: $t$ and $X$ where $X\sim \mathbb P$. 
    \STATE 1. Initialize $v^*\gets0$
    \FOR{$i\in[n]$}
            \IF{$|X_{ii}|\ge t$}
                \STATE $(v^*)_i:=\text{sign}(X_{ii})$
            \ELSE
                \STATE $(v^*)_i:=0$
        \ENDIF
        \ENDFOR
\end{algorithmic}
\end{algorithm}
    \end{minipage}
\end{figure}

\begin{remark}
The runtimes of Algorithms \ref{alg:DT_dist} and \ref{alg:DT_recov} are linear.
\end{remark}

\begin{theorem}[Distinguishability by Diagonal Thresholding: Small-Ball]
\label{distinguish_gaussian_DT}
Assume $np\to\infty,$ that $v$ satisfies Assumption \ref{spike_small_ball_assumption}, and
that $W$ is sub-Gaussian (Assumption~\ref{wigner_subg_assumption}). 
Let $\mathcal T:= \max_{i\in[n]}X_{ii}.$ Under the additional assumption $q\gg np^2 \log(nq),$ we can use Algorithm \ref{alg:DT_dist} to distinguish between the planted and null distributions. Specifically, given thresholding parameter $t$ satisfying
\begin{align}
K_2\sqrt{\frac{\log(nq)}{nq}}
\lesssim t
\ll \frac{1}{np}~,\label{dt_dist_interval}
\end{align}
it follows that
\begin{align*}
\mathbb P\left(\mathcal T\leq t\right)+ \mathbb Q\left(\mathcal T>t\right)
&\leq
C_s\sqrt{\frac{2npt}{\theta}}
+
2(n+1)q\exp\left(-c\frac{nqt^2}{K_2^2}\right)=o(1)~,
\end{align*}
where all the constants are defined within the proof.
\end{theorem}

See section \ref{section:DT_dist} for the proof. 

\begin{remark}
    For the lower bounded prior, Algorithm \ref{alg:DT_dist} can achieve the same result with $t\le \frac{\theta}{2B^2np}$ replacing the upper bound of \eqref{dt_dist_interval}. Thus, the condition becomes $q\gtrsim np^2\log(nq)$. The proof is identical with the exception of \eqref{diagonal_thresholding_typeI_computation} no longer having a tail bound for the spike term as the probability is $0.$
\end{remark}

Let us compare conditions.  For small-ball, DT requires $q\gg np^2\log(nq)$ and polynomial-time QT requires $q\gtrsim np^2\log(n)$ (Theorem \ref{thm:quad_subG_distinguish} with $l=1$). Notice neither condition implies the other for all regimes of $p$ and $q$. For instance, $q=\log(n)/\sqrt{n}$ and $np^2=n^{-1/2}$ satisfies the QT condition but not the DT condition, while $q=\log(n)/n$ and $np^2=\log(n)/(n\sqrt{\log n\log\log n})$ satisfies the DT condition but not the QT condition. 

For lower bounded, DT ($q\gtrsim np^2 \log(nq)$) beats or matches polynomial-time QT ($q\gtrsim np^2 \log(nq)$); however, non-polynomial QT ($q\gtrsim p \log(n)$) could potentially beat DT similar to above.

Furthermore, Diagonal Thresholding will always be linear in runtime whereas Quadratic Thresholding extends from polynomial to sub-exponential.

In the case of recovery, Diagonal Thresholding also covers new regions of sparsity, which is especially important when $np=\Theta(1)$.


\begin{theorem}[Support Recovery by Diagonal Thresholding: Small-Ball]\label{recovery_gaussian_DT} 
Assume $v$ satisfies Assumption \ref{spike_small_ball_assumption}. Under the assumption $q\gg \|v\|_0^4np^2 \log(nq),$ Algorithm \ref{alg:DT_recov} can achieve support recovery. Specifically, given $v^*$, the output of Algorithm \ref{alg:DT_recov}, and a thresholding parameter $t$ satisfying
\begin{align}
    K_2 \sqrt{\frac{\log(nq)}{nq}} \lesssim t \ll \frac{\theta}{C_s^2\|v\|_0^2 \cdot np} ~,\label{dt_recov_window}
\end{align}
it follows that
\begin{align*}
\mathbb{P}(\text{supp}(v)\not=\text{supp}(v^*)) \leq 2\max\left\{C_s \|v\|_0 \sqrt{\frac{2npt}{\theta}}, 2nq\exp\left(-c\frac{nqt^2}{K_2^2}\right)\right\}=o(1) ~.
\end{align*}
where all constants are defined within the proof.
\end{theorem}

See Section \ref{section:DT_subg_recov} for the proof. 

\begin{remark}
    For the lower bounded prior, Algorithm \ref{alg:DT_recov} can achieve the same result with $t\le \frac{\theta}{2B^2np}$ replacing the upper bound of \eqref{dt_recov_window}. Thus, the condition becomes $q\gtrsim np^2\log(nq)$. The proof is identical with the exception that \eqref{small_ball_cond} equals $0$ when $t\le \frac{\theta}{2B^2np}$.
\end{remark}

Similarly to distinguishing, the sparsity conditions of the above theorem/remark and Theorem \ref{thm:QT_subG_support_recovery} / Theorem \ref{quad_form_recovery_thm} overlap but are incomparable. Furthermore, in the Rademacher prior, QT always beats DT as DT provides no sign recovery and therefore cannot exactly recover.
Thus, combining these two results, we arrive at the following corollary on weak/strong recovery for a small-ball prior. The corollary for the lower bounded prior is identical up to adjusting for the shifted sparsity regimes described in the remark above and the corresponding QT theorem. 



\begin{corollary}[Dual Thresholding: Sub-Gaussian Weak/Strong Recovery]\label{cor:DT_recovery}
    Assume $v$ satisfies Assumption \ref{spike_small_ball_assumption}. If $q\gg (np)^5p \log(nq),$ then Diagonal Thresholding (Algorithm \ref{alg:DT_recov}) can recover supp($v$) in $O(n)$ runtime. If $q\gg (np)^2p\log(n)$ and $np\to\infty$, then Quadratic Thresholding (Algorithm \ref{quad_recovery_alg_recov}) can recover recover supp($v$) in runtime $n^{O(l)}$ depending on tuning parameter $l$. 
   
Given we now have supp($v$), $v_1(X_{\text{supp}(v)})$ achieves weak/strong recovery.
\end{corollary}

The proof of the above Corollary complete with tail bounds is almost identical to the proof and bounds of Corollary \ref{quad_form_weak_strong_corollary}, with one small caveat. Because the entries are now sub-Gaussian, \eqref{QT_strongrecov_subg_caveat} would need to be replaced by a conditional argument to ensure the two-norm $\|v\|_2$ is lower bounded. Specifically, one could condition on Assumption \ref{spike_small_ball_assumption} to replace $B$ by $\tau_4\|v\|_0$ for some chosen $\tau_4\to\infty$. Then, the final tail bounds would be identical to Corollary \ref{quad_form_weak_strong_corollary} except with $B=\tau_4\|v\|_0$ and an additional tail bound term of \eqref{small_ball_callback}. Thus, $\tau_4$ affects recovery possibility; however, further inspection reveals that the sparsity conditions of the corresponding sub-Gaussian results ensure strong recovery. For instance, to strongly recover when $np\to\infty$, one would need $\tau_4^2\|v\|_0^2\frac{p}{q}\ll1,$ which is already satisfied by the sparsity requirement $q\gtrsim\tau_4^2(np)^2p\log(n)$ of Theorem \ref{thm:QT_subG_support_recovery}. The other cases follow in a similar nature.

\subsection{Diagonal Thresholding Distinguishability Proof: Small-Ball Spike}\label{section:DT_dist}
\begin{proof-of-theorem}[\ref{distinguish_gaussian_DT}] We begin by showing a region of sparsity where distinguishing is possible by thresholding the entries on the diagonal.  \\

{\textit{Type I Error.}} First, we must ensure all the noise terms stay beneath the threshold $a$. \begin{align*}
\mathbb{P}(X_{ii} \geq a \mid v_i = 0) =~& \mathbb{P}(X_{ii} \geq a \mid v_i = 0, A_{ii}=0)\cdot\mathbb{P}(A_{ii}=0) 
+ \mathbb{P}(X_{ii} \geq a \mid v_i = 0, A_{ii}\neq 0)\cdot\mathbb{P}(A_{ii}\neq 0) \\
=~& (1-q)\cdot\mathbb{P}(0 \geq a) + q\cdot\mathbb{P}\left(\frac{1}{\sqrt{nq}}W_{ii} \geq a\right) \\
\leq~& 2q \exp\left(-c\frac{nqa^2}{K_2^2}\right) ~.
\end{align*}
Union bounding across all diagonal entries yields:
\begin{align}
\mathbb{P}(\exists v_i = 0 \text{ : } X_{ii} \geq a) \le 2nq \exp\left(-c\frac{nqa^2}{K_2^2}\right) ~.\label{noise_term_control_subg_qt}
\end{align}

{\textit{Type II Error.}}
In order for the type II error to go to zero, we need to ensure that there exists one spiked entry above the threshold $a$ on the diagonal. Observe how
\begin{align}
\mathbb{P}(X_{ii} \leq a \mid v_i \neq 0) =~& \mathbb{P}(X_{ii} \leq a \mid v_i \neq 0, A_{ii}=0)\cdot\mathbb{P}(A_{ii}=0)+ \mathbb{P}(X_{ii} \leq a \mid v_i \neq 0, A_{ii}\neq 0)\cdot\mathbb{P}(A_{ii}\neq 0) \notag \\
=~& (1-q)\cdot\mathbb{P}\left(\frac{\theta}{np}v_i^2 \leq a\right) + q\cdot\mathbb{P}\left(\frac{\theta}{np}v_i^2 + \frac{1}{\sqrt{nq}}W_{ii} \leq a\right) \notag \\
\le~& (1-q)\cdot\mathbb{P}\left(\frac{\theta}{np}v_i^2 \leq 2a\right) + q\cdot\mathbb{P}\left(\frac{\theta}{np}v_i^2 \leq 2a\right) + q\cdot\mathbb{P}\left(\frac{1}{\sqrt{nq}}W_{ii} \leq -a\right) \notag \\
=~& \mathbb{P}\left(|v_i| \leq \sqrt{\frac{2npa}{\theta}}\right) + q\cdot\mathbb{P}\left(\frac{1}{\sqrt{nq}}W_{ii} \leq -a\right) \notag \\
\le~& C_s \sqrt{\frac{2npa}{\theta}} + 2q \exp\left(-c\frac{nqa^2}{K_2^2}\right) ~, \label{diagonal_thresholding_typeI_computation}
\end{align}
where we used the small-ball condition from Assumption~\ref{spike_small_ball_assumption} with $\epsilon=\sqrt{2npa/\theta}$ to evaluate the anti-concentration of the lower tail, and applied a standard sub-Gaussian tail bound for $W_{ii}$ matching Assumption~\ref{wigner_subg_assumption}. Thus,
\begin{align*}
\mathbb P \left(X_{ii} \leq a ~~\forall  v_i \neq 0 \right)=\mathbb P \left(\bigcap _{v_i\not = 0} {X_{ii}\le a } \right) \le \mathbb{P}(X_{ii} \leq a \mid v_i \neq 0) \le C_s \sqrt{\frac{2npa}{\theta}} + 2q \exp\left(-c\frac{nqa^2}{K_2^2}\right)~.
\end{align*}

{\textit{Conclusion.}} Thus, in order for the interval to be nonempty, we need \[K_2 \sqrt{\frac{\log(nq)}{nq}}\lesssim a\ll\frac{1}{np}\implies q\gg np^2 \log(nq)~.\]
\end{proof-of-theorem}

\subsection{Diagonal Thresholding Support Recovery Proof: Small-Ball Spike}\label{section:DT_subg_recov}

\begin{proof-of-theorem}[\ref{recovery_gaussian_DT}]
From \eqref{diagonal_thresholding_typeI_computation}, take a union bound to see
\begin{align*}
\mathbb{P}(\exists v_i \neq 0 \text{ : } X_{ii} \leq a) \leq C_s \|v\|_0 \sqrt{\frac{2npa}{\theta}} + 2q \|v\|_0 \exp\left(-c\frac{nqa^2}{K_2^2}\right) ~.
\end{align*}
To control the size of the noise entries, we use \eqref{noise_term_control_subg_qt} and see we require:
\begin{align}\label{small_ball_cond}
C_s \|v\|_0 \sqrt{\frac{2npa}{\theta}}\to 0
\end{align}
\begin{align}\label{noise_cond}
2nq \exp\left(-c\frac{nqa^2}{K_2^2}\right)\to 0 ~.
\end{align}
We see \eqref{noise_cond} requires $a \gtrsim K_2 \sqrt{\frac{\log(nq)}{nq}}$ and \eqref{small_ball_cond} forces an upper bound on $a$. Combining these two requirements, the threshold parameter $a$ must satisfy:

$$
    K_2 \sqrt{\frac{\log(nq)}{nq}} \lesssim a \ll \frac{\theta}{C_s^2\|v\|_0^2 \cdot np} ~.
$$
Thus, in order for this interval to be nonempty, we need $q\gg C_s^2\|v\|_0^4np^2\log(nq)$
\end{proof-of-theorem}

\section{Identical Search}\label{section:identical_search}

In this section, we analyze the conditions necessary to distinguish or recover when using an Identical Search algorithm. We describe the intuition informally first.

Throughout this entire section, we condition on the spike support being fixed (with its asymptotic size dictated by $np$, as we describe in Section \ref{spike_support_regimes_section}), but we do not fix its location within the observed matrix. Without the noise mask, the Rademacher distribution of the spike entries creates a $\|v\|_0\times \|v\|_0$ block with entries $\pm \frac{\theta}{np}$. Thus, this submatrix sits within $X$, blurred by a symmetric noise mask $Y_{\text{supp}(v)}$ with independent Bernoulli($q$) entries. We want to explore and exploit a regime in which this noise mask is very sparse, thus revealing the $\pm\frac{\theta}{np}$ entries, which allows us to conclude the corresponding entry is in the spike's support. Therefore, going forward, we are concerned only with analyzing this $\|v\|_0\times\|v\|_0$ spiked submatrix of $X$, which we will call $X_{\text{supp}(v)}$. 

\subsection{Distinguishability}

The results of this section diverge depending on a few factors. First, consider the asymptotic behavior of $np.$ Specifically, if $np$ is constant or converges to $0,$ then one cannot distinguish between the null and spiked models by Theorem \ref{distinguish_them_np_equals_constant}. To allow for distinguishing $np=\Theta(1)$, one needs to assume a spike prior with a fixed support size, which is commonly used in similar literature regarding spiked models \cite{perry2016optimalitysuboptimalitypcaspiked, banks2017informationtheoreticboundsphasetransitions}. In a technical sense, the sparsity sampling we employ (where the spike support is not fixed) actually generalizes this fixed support size prior, and thus all of our results capture this stricter sampling; the only update to our results would be the removal of the second term in our tail bounds as they no longer rely on a high probability result regarding the spike support's size. We specify this fact in the theorems. 

Next, the results and conditions further split based on whether the observer tests with an awareness of the value $\theta/p $. Asymptotically, the two cases are practically identical, but we delineate their procedures for precision. Let the test statistic $\mathcal T_{partial}$ be the number of sampled entries from the observation identical to either $\pm\frac{\theta}{np}$. Let the test statistic $\mathcal T_{ag}$ be the largest number of sampled entries from the observation identical to each other. Below, we introduce the hypothesis tests and algorithms for distinguishing. \\ 

\textbf{Hypothesis Test (Spike Partial/Agnostic Info)}: If $\mathcal T_{partial}\ge1$/$\mathcal T_{ag}\ge2$, reject the null. Otherwise, accept the null.

\begin{figure}[ht]
    \begin{minipage}[t]{0.48\textwidth}
        \begin{algorithm}[H]
    \caption{Distinguishability by Identical Search (Spike Partial Info)}\label{identical_search_distinguish_aware}
    \begin{algorithmic}
        \STATE Inputs: $\theta>0, p\in[0,1], q \in [0,1],n\in \mathbb N$ and $X$ where $X\sim \mathbb P$ or $X \sim \mathbb Q$.
        \STATE 1. Initialize $\mathcal T_{partial}\gets0.$
        \FOR{$i\le j$ such that $\{i,j\}\in[n]^2$}
            \IF{$X_{i,j}=\pm\frac{\theta}{np}$}
                \STATE $\mathcal T_{partial}\gets\mathcal T_{partial}+1$
            \ENDIF
        \ENDFOR
    \end{algorithmic}
\end{algorithm}
    \end{minipage}
    \hfill
    \begin{minipage}[t]{0.48\textwidth}
        \begin{algorithm}[H]
    \caption{Distinguishability by Identical Search (Spike Agnostic)}\label{identical_search_distinguish_blind}
    \begin{algorithmic}
        \STATE Inputs: $X$ where $X\sim \mathbb P$ or $X \sim \mathbb Q$.
        \STATE 1. Initialize $\mathcal T_{ag}\gets0.$
        \FOR{$i\le j$ such that $\{i,j\}\in[n]^2$}
            \FOR{$k<l$ such that $k,l\in[n]\backslash\{i,j\}$}
                \IF{$X_{i,j}=\pm X_{k,l}\not=0$}
                    \STATE $\mathcal T_{ag}\gets\mathcal T_{ag}+1$
                \ENDIF
            \ENDFOR
        \ENDFOR
    \end{algorithmic}
\end{algorithm}
    \end{minipage}
\end{figure}


\begin{remark}
    The runtimes of Algorithms \ref{identical_search_distinguish_aware} and \ref{identical_search_distinguish_blind} are $O(n^2)$ and $O(n^{4}),$ respectively. One could potentially improve the runtime by introducing a "Terminate" line within the if statements, as only one instance is required for distinguishability. Furthermore, the latter's runtime is determined by a brute force search, but could also be improved by considering a different method for comparing each of the $\frac{n(n+1)}{2}$ entries. 
\end{remark}

\subsubsection{Type I Error}

Type I error corresponds to being in the null model, but concluding the observation has been planted. In the Spike Partial Info case, this would correspond to a noise matrix entry somehow equaling $\pm\frac{\theta}{np}$, which we will see has probability zero. In the Spike Agnostic case, this could also occur if the sparse noise matrix has two identical entries, which we also see has probability 0. This lemma shows the necessity for a continuous distribution of the dense noise matrix. Without this continuity (perhaps drawing from an atomic distribution), one would be unable to distinguish.

\begin{lemma}\label{idensearch_typeII}
    $\mathbb Q(\mathcal T_{partial}\geq1)=0$ and $\mathbb Q(\mathcal T_{ag}\geq2)=0.$
\end{lemma}


\subsubsection{Type II Error}
Type II error corresponds to the observation being drawn from the planted model, but the test incorrectly concludes it is the null model. Therefore, this would correspond to an instance where the $\pm \frac{\theta}{np}$ entries are present in $X$, but are all blurred by noise. To prevent this, we provide probability bounds and sparsity conditions that guarantees a regime with the presence of at least one (or two, for Spike Agnostic) $\pm\frac{\theta}{np}$ entries within the spiked submatrix of $X$. Therefore, we have the following lemma, which describes the probability of at least one (or two) Bernoulli($q$) random variables out of $\frac{\|v\|_0(\|v\|_0+1)}{2}$ being equal to $0$. 

\begin{lemma}\label{idensearch_typeI}
    Given $\|v\|_0^2\log(q)\to-\infty$, we have 
    \begin{align*}
        \mathbb P(\mathcal T_{partial}<1)=q^{\|v\|_0(\|v\|_0+1)/2}=o(1).
    \end{align*}
    Furthermore, given $\|v\|_0^2\log(\|v\|_0^2q)\to-\infty$, we have 
    \begin{align*}
        \mathbb P(\mathcal T_{ag}<2)=q^{\frac{\|v\|_0(\|v\|_0+1)}{2}}+\frac{\|v\|_0(\|v\|_0+1)}{2}(1-q)q^{\frac{\|v\|_0(\|v\|_0+1)}{2}-1}=o(1)~.
    \end{align*}
\end{lemma}

\subsubsection{Distinguishability Main Result}
As previously mentioned, if we sample from a spike distribution with a fixed mask size, then one can directly reference Lemmas \ref{idensearch_typeI} and \ref{idensearch_typeII} to conclude a distinguishability result with the exact tail bounds from above given that the conditions of the respective lemmas are satisfied. Then, if one wants to distinguish for a constant $np$, one would need $q\ll 1.$ For brevity, we omit this case and instead state the distinguishability theorem for our model under the assumption $np\to\infty.$
\begin{theorem}[Identical Search Distinguishability]\label{iden_search_distinguish} Let $np\to\infty.$ \\
    \textbf{Spike Partial Info. } Assume we know $\frac{\theta}{np}$. Then, given $1-q\gg(np)^{-2}$, we can \textbf{Distinguish} between the planted and null models using the test statistic $\mathcal T_{partial}$. More specifically, Algorithm \ref{identical_search_distinguish_aware} achieves distinguishability with type I and type II errors
    \begin{align*}
        \mathbb P (\mathcal T_{partial}< 1)+\mathbb Q(\mathcal T_{partial}\ge 1) \le 2\max{\left\{q^{(np)^2/2}, (np)^{-2} \right\}}~.
    \end{align*}
    \textbf{Spike Agnostic. } Assume we do not know $\frac{\theta}{np}$. Then, given that $1-q\gg(np)^{-2},$ we can \textbf{Distinguish} between the planted and null models using the test statistic $\mathcal T_{ag}$. More specifically, Algorithm \ref{identical_search_distinguish_blind} achieves distinguishability with type I and type II errors
    \begin{align*}
        \mathbb P (\mathcal T_{ag}< 2)+\mathbb Q(\mathcal T_{ag}\ge 2) \le 2\max\left\{q^{(np)^2/8},{2(np)^2}(1-q)q^{(np)^2/8}  \right\}~.
    \end{align*}
\end{theorem}

\begin{proof}
    Seeing as $np\to\infty$, we know $\|v\|_0=np+k$ for $k\in[-\log(np)\sqrt{np},\log(np)\sqrt{np} ]$ with probability $1-(np)^{-2}$ by Lemma \ref{fixedmasklemma}. 
    
    First, we consider the Spike Partial Info case. Combining Lemmas \ref{idensearch_typeI} and \ref{idensearch_typeII}, we find the type I and type II errors go to $0$ when $$q^{\|v\|_0(\|v\|_0+1)/2}\le q^{(\|v\|_0-1)(\|v\|_0+1)/2}\le q^{\|v\|_0^2/4}\le q^{(np)^2/8}=o(1)~.$$ Thus, when $(np)^2\log(q)\to-\infty$, both errors vanish. When $q\not\to 1,$ this is satisfied by our condition $np\to\infty$ seeing as $\log(q)<0$. When $q\to1,$ we can approximate $(np)^2\log(q)\approx-(np)^2(1-q)$, which diverges to $-\infty$ when $1-q\gg(np)^{-2}.$
    
    Next, considering the Spike Agnostic case, we need $$q^{\frac{\|v\|_0(\|v\|_0+1)}{2}}+\frac{\|v\|_0(\|v\|_0+1)}{2}(1-q)q^{\frac{\|v\|_0(\|v\|_0+1)}{2}-1}\to0~.$$

    Substituting in our value for $\|v\|_0$ and considering leading terms, we see this term is upper bounded by $2\max\left\{q^{(np)^2/8},{2(np)^2}(1-q)q^{(np)^2/8}  \right\},$ where the larger term depends on the values of $q$ and $n.$ From the analysis of the Spike Partial Info case, we know this occurs when $1-q\gg(np)^{-2},$ and we will show this condition also sends the second term to $0$ as well. Consider $1-q\gg1/(np)^{-2};$ more specifically, let $q=1-\tau(n)/(np)^2$ where $1\ll\tau\ll(np)^2.$ Then, we have $$(np)^2(1-q)q^{(np)^2/8}=\tau(1-\frac{\tau}{(np)^2})^{(np)^2/8-1}\approx\tau\exp(-\tau/2)\to0~,$$ as $n\to\infty$ seeing as $np\to\infty$.
\end{proof}

\subsection{Recovery}

Consistent with the above results, Identical Search recovery depends entirely on $p$ and $q$. We continue to condition on the spike support submatrix being fixed in size. Just as we did when distinguishing, we can consider two cases: spike partial info and spike agnostic. For brevity, we will proceed in the standard Bayes-optimal (partial info) setting. However, all of the following procedures would hold (with minor modifications), and the difference between settings would only matter with extreme spike sparsity where $\|v\|_0\in\{0,1,2\}$.

To gain an intuition for our more nuanced Identical Search procedure, we first observe how elementary probability calculations provide a threshold for exact recovery through a naive, row-by-row identical search. Specifically, given that we search every row of the matrix for entries equal to $\pm\frac{\theta}{np}$, we can consider the probability that, for some $i\in\text{supp}(v),$ $(Y_{\text{supp}(v)})_{ij}=0$ for all $j\in\text{supp}(v)$. Formally, let $S_i:=|\{(X_{\text{supp}(v)})_{ij}=\pm\frac{\theta}{np}:j\in\text{supp}(v) \}|$. The probability calculations easily follow when one realizes each entry in the row has probability $(1-q)$ of being equal to $\pm\frac{\theta}{np}$:

\begin{theorem}[Exact Recovery under Sub-Optimal Sparsity Conditions by Naive Row Search]\label{iden_search_recov_suboptimal}
    Consider the procedure of searching every row of $X$ for an entry equal to $\pm\frac{\theta}{np}$, then constructing a recovery vector $v^*$ corresponding to the row with the maximum number of these entries. 
    
    Given either (1) $q\ll\frac{1}{np}\ll1$ or (2) $np=\Theta(1)$ and $q\ll1$, we can achieve \textbf{Exact Recovery} through the above procedure. More specifically, $$\mathbb P ( \{ \exists ~i\in \text{supp}(v) : A_{ij}=0~~ \forall j\in\text{supp}(v) \})=\mathbb P(\cup_{i=1}^{\|v\|_0}\{S_i=\|v\|_0 \}) \ge \mathbb P (S_i=\|v\|_0)=(1-q)^{\|v\|_0}~=1-o(1).$$
\end{theorem}


\begin{remark}
    If one is solely interested in recovering the support (and not the signs), then one could simply search the diagonal of $X$ and achieve Exact Support Recovery with the same probability equal to the one above. The advantage of this adaptation would be decreasing the runtime to $O(n).$
\end{remark}

This idea is meant to serve as a stepping stone rather than provide comprehensive sparsity results. Note that we only exploit information from one of the rows of the mask.  

Instead, one could consider using Bennett's Inequality to find an interpretable tail bound for the probability of $S_i$. However, this direction of research proves to be suboptimal. After performing an extensive investigation using Bennett's Inequality, we obtained weaker results (in terms of sparsity) than what we are about to present. The reason is because Bennett's Inequality exploits knowledge of only one row, whereas the following algorithm and theorem exploits information from the entire matrix.

To answer this question, we shift perspectives and view the problem through the lens of graph theory to derive better sparsity conditions. If we view $X$ as a weighted adjacency matrix, then we can view the problem  as recovering the subgraph with edges equal to $\pm\frac{\theta}{np}$, i.e. $X_{\text{supp}(v)}-Y_{\text{supp}(v)}$. 

When $\|v\|_0\to\infty$, the complement of the noise submatrix blurring the spike clique can be thought of as a Erdos-Renyi graph with loops $\mathcal G(\|v\|_0,1-q).$ Therefore, if this complement of the noise mask is a connected graph, then Algorithm \ref{identical_search_recovery_alg} can completely recover the support and signs of the planted clique. 

\begin{algorithm}
\caption{Exact Recovery by Identical Search}\label{identical_search_recovery_alg}
\begin{algorithmic}
\STATE \vskip-1ex \vspace{-.1in} 
    \STATE \noindent
\begin{minipage}[t]{0.45\linewidth}
\STATE Input: $X\sim\mathbb P$, $\theta/p$.
\STATE Initialize $\widehat v\gets 0\in\mathbb R^n$ and $\mathcal V\gets\emptyset$.
\STATE \textbf{Stage 1: Finding Initial Vertex}
\FOR{$i\le j$ with ${i,j}\in[n]^2$}
\IF{$X_{i,j}=\pm\frac{\theta}{np}$}
\STATE $\widehat v_i\gets 1$
\STATE Add $i$ to $\mathcal V$.
\STATE \textbf{Terminate For Loop.}
\ENDIF
\ENDFOR
\STATE Initialize $\mathcal Q$ with vertex $i\in\mathcal V$.
\end{minipage}%
    \hfill
\begin{minipage}[t]{0.45\linewidth}
\STATE \textbf{Stage 2: Graph Walking}
\WHILE{$\mathcal Q\neq\emptyset$}
\STATE Remove a vertex $i$ from $\mathcal Q$.
\FOR{$j\in[n]$}
\IF{$X_{i,j}=\pm\frac{\theta}{np}$ \textbf{ and } $j\notin\mathcal V$}
\STATE $\widehat v_j\gets\widehat v_i\operatorname{sgn}(X_{i,j})$
\STATE Add $j$ to $\mathcal V$.
\STATE Add $j$ to $\mathcal Q$.
\ENDIF
\ENDFOR
\ENDWHILE
\STATE \textbf{Output:} $\widehat v$.
\end{minipage}
\STATE \vskip1ex \vspace{-.1in}
\end{algorithmic}
\end{algorithm}

\begin{remark}
    In the worst case scenario, we examine $O(n^2)$ entries.
\end{remark}

\begin{theorem}[Exact Recovery by Identical Search when $np\to\infty$]
\label{iden_search_recov_thm_np_to_inf}
Given $np\to\infty$, let $\tau_6:=\tau_6(np)$ satisfy
\begin{align}
\tau_6(np)\to\infty
\qquad\text{and}\qquad
\tau_6(np)=o\left(\frac{\sqrt{np}}{\log(np)}\right).
\label{tau5_condition}
\end{align}
Then Algorithm \ref{identical_search_recovery_alg} achieves \textbf{Exact Recovery} under the assumption
\begin{align}
q\le1-\frac{\log(np)+\tau_6(np)}{np}.
\label{iden_search_recov_q_cond_np_inf}
\end{align}
More specifically, given $v^*$ is the output of Algorithm \ref{identical_search_recovery_alg} and \eqref{iden_search_recov_q_cond_np_inf} is satisfied,
\begin{align}
\mathbb P(\mathcal R(v^*,v)=1)
\ge
1-(np)^{-2}
-\exp\left(-\exp\left(-\frac{\tau_6(np)}{2}\right)\right)
=1-o(1).
\end{align}
\end{theorem}

\begin{proof}
As described above, we are concerned with finding the threshold where the shadow graph $\mathcal G(\|v\|_0,1-q)$ is connected. By Lemma \ref{fixedmasklemma}, with probability at least $1-(np)^{-2}$,
\begin{align}
np-\log(np)\sqrt{np}
\le \|v\|_0
\le
np+\log(np)\sqrt{np}.
\label{support_size_event}
\end{align}
We condition on the event in \eqref{support_size_event} and write
\begin{align}
\|v\|_0=np(1+\delta),
\qquad
|\delta|\le \frac{\log(np)}{\sqrt{np}}.
\label{delta_definition}
\end{align}
By \eqref{iden_search_recov_q_cond_np_inf},
\begin{align*}
(1-q)\|v\|_0
&\ge
\frac{\log(np)+\tau_6(np)}{np}\|v\|_0
=
(1+\delta)\big(\log(np)+\tau_6(np)\big).
\end{align*}
On the other hand, $\log \|v\|_0
=
\log(np)+\log(1+\delta).$
Therefore,
\begin{align}
(1-q)\|v\|_0-\log \|v\|_0
&\ge
\tau_6(np)
+\delta\big(\log(np)+\tau_6(np)\big)
-\log(1+\delta).
\label{connectivity_margin}
\end{align}
From \eqref{delta_definition},
$$
|\delta|\log(np)
\le
\frac{\log^2(np)}{\sqrt{np}}
=o(1).
$$
Furthermore, by \eqref{tau5_condition},

$$
|\delta|\tau_6(np)
\le
\frac{\tau_6(np)\log(np)}{\sqrt{np}}
=o(1).
$$
Thus, uniformly over the support sizes satisfying \eqref{support_size_event}, 
\begin{align}
\delta\big(\log(np)+\tau_6(np)\big)
-\log(1+\delta)
=o(1)\implies (1-q)\|v\|_0-\log \|v\|_0
\ge
\tau_6(np)-o(1).
\label{connectivity_margin_final}
\end{align}
Since $\tau_6(np)\to\infty$, for sufficiently large $np$,
\begin{align}
(1-q)\|v\|_0
\ge
\log \|v\|_0+\frac{\tau_6(np)}{2}.
\label{connectivity_condition}
\end{align}
Taking $c={\tau_6(np)}/{2}$
in the context of \cite[Theorem 7.3]{bollobas2001random}, \eqref{connectivity_condition} implies that the shadow graph $\mathcal G(\|v\|_0,1-q)$ is connected with probability at least 
\begin{align}
1-\exp\left(-\exp\left(-\frac{\tau_6(np)}{2}\right)\right).
\end{align}
Therefore, after removing the conditioning on the support-size event and using Lemma \ref{fixedmasklemma},
\begin{align}
\mathbb P(\mathcal R(v^*,v)=1)
&\ge
1-(np)^{-2}
-\exp\left(-\exp\left(-\frac{\tau_6(np)}{2}\right)\right).
\end{align}
Both error terms converge to zero since $np\to\infty$ and $\tau_6(np)\to\infty$, completing the proof.
\end{proof}


The analysis above does not hold when considering $\|v\|_0=\Theta(1),$ seeing as this corresponds to a case where the graph has a fixed number of vertices that does not grow with $n.$ However, we can still explicitly compute the regimes of $q$ where this fixed vertex subgraph is connected. Specifically, we consider a graph of fixed size $\|v\|_0\times \|v\|_0$ that does not grow with with $n.$ 

When $\|v\|_0=0,$ Algorithm \ref{identical_search_recovery_alg} will trivially return nothing, as there is no spike to recover. When $\|v\|_0=1,$ we are looking at just the probability of there being a loop on this one vertex, i.e. $\mathbb P (\mathcal G \text{ connected})=1-q.$ For $\|v\|_0>1,$ the probabilistic computations complicate further. As we did above, we will consider the loop-less case, i.e. ignoring the diagonal, seeing as further computation would show the asymptotic conditions on $q$ would not change. 

\begin{theorem}[Exact Recovery by Identical Search for $np=\Theta(1)$]\label{iden_search_recovery_thm_const_np}
Let $np=\Theta(1)$. Given $q\ll 1,$ Algorithm \ref{identical_search_recovery_alg} achieves \textbf{Exact Recovery} with probability $1-o(1)$. More specifically, conditional on $\|v\|_0=k$, given $v^*$ as the output of Algorithm \ref{identical_search_recovery_alg}, we have
\begin{align*}
\mathbb P (\mathcal R(v^*,v)=1\mid \|v\|_0=0)&=1;\\
\mathbb P (\mathcal R(v^*,v)=1\mid \|v\|_0=1)&\ge1-q=1-o(1);\\
\mathbb P (\mathcal R(v^*,v)=1\mid \|v\|_0=k)&\ge1-4q=1-o(1),
\end{align*}
for every fixed $k\ge2$. Since $\|v\|_0\sim\operatorname{Bin}(n,p)$ with $np=\Theta(1)$, the sequence ${\|v\|_0}_{n\ge1}$ is tight, and hence these conditional guarantees imply
\(        \mathbb P(\mathcal R(v^*,v)=1)=1-o(1).
   \)
\end{theorem}

\begin{proof} Since $np=\Theta(1)$ and $\|v\|_0\sim\operatorname{Bin}(n,p)$, the sequence ${\|v\|_0}_{n\geq1}$ is tight. Thus, for every $\varepsilon>0$, there exists a fixed integer $M$ such that
$
\mathbb P(\|v\|_0>M)\leq\varepsilon
$
for all sufficiently large $n$.

We first condition on the event $\|v\|_0=k$ for a fixed $k$. If $k=0$, the algorithm correctly returns the zero vector, since there are no spiked entries. If $k=1$, the support can be recovered directly, while sign recovery is not required. We therefore consider $k\geq2$.

Conditional on $\|v\|_0=k$, the induced graph on the $k$ vertices in $\operatorname{supp}(v)$ has independent edges with probability $1-q$. In order for this graph to be disconnected, there must exist a connected component of size $m\in\left[1,\left\lfloor\frac{k}{2}\right\rfloor\right].$ For a fixed subset of $m$ vertices to form a disconnected component, all $m(k-m)$ edges between this subset and its complement must be absent. Hence, if $A_m$ denotes the event that there exists a disconnected component of size $m$, then

$$
\mathbb P(A_m\mid \|v\|_0=k)
\leq
\binom{k}{m}q^{m(k-m)}.
$$
By a union bound, when $q\to0,$
\begin{align*}
\mathbb P(\mathcal G\text{ is disconnected}\mid \|v\|_0=k)
\leq
\sum_{m=1}^{\lfloor k/2\rfloor}
\binom{k}{m}q^{m(k-m)}\
\leq
\sum_{m=1}^{\lfloor k/2\rfloor}
\binom{k}{m}q^{k-1}
\leq
2^kq^{k-1}=o(1)~.
\end{align*}
Now choose $M$ sufficiently large that $\mathbb P(\|v\|_0>M)\leq\varepsilon$ for all sufficiently large $n$. Then
\begin{align*}
\mathbb P(\mathcal G\text{ is disconnected})
&\leq
\mathbb P(\|v\|_0>M)
+
\sum_{k=2}^{M}
\mathbb P(\mathcal G\text{ is disconnected}\mid \|v\|_0=k)\mathbb P(\|v\|_0=k)\\
&\leq
\varepsilon
+
\max_{2\leq k\leq M}
\mathbb P(\mathcal G\text{ is disconnected}\mid \|v\|_0=k).
\end{align*}
Since $M$ is fixed and the conditional probability tends to zero for every fixed $k$, we obtain

$$
\limsup_{n\to\infty}
\mathbb P(\mathcal G\text{ is disconnected})
\leq\varepsilon.
$$
As $\varepsilon>0$ was arbitrary,
$
\mathbb P(\mathcal G\text{ is disconnected})\to0.
$
\end{proof}

\begin{remark}
    When $np=\Theta(1)$, notice Algorithm \ref{identical_search_recovery_alg} and the naive recovery of Theorem \ref{iden_search_recov_suboptimal} have the same sparsity condition: $q\ll1.$ Intuitively, this is because, for a fixed number of vertices, the graph being connected with high probability and at least one vertex being connected to every other vertex with high probability both occur under the condition $q\to0.$ One could default to the naive procedure of Theorem \ref{iden_search_recov_suboptimal}, but we include the above proof to show that Algorithm \ref{identical_search_recovery_alg}, which in practice has a shorter runtime, will also be successful in this sparsity regime. 
    
    Note in the special case $\|v\|_0=1,$ achieving exact recovery and exact (support) recovery are identical tasks as one can only ever achieve sign recovery up to a universal sign flip. 
    
    Furthermore, in the setting $np=\Theta(1),$ we could apply a Bennett's Inequality argument to achieve Almost Exact Recovery, but we exclude the details as they do not improve on the above result. Specifically, when $np=\Theta(1),$ a Bennett's Inequality argument requires $q\to 0$ to achieve Almost Exact recovery, whereas the above algorithm provides Exact Recovery under the same assumption.
\end{remark}

\paragraph{Acknowledgments} This project was supported by the following National Science Foundation (NSF) grant: DMS-2154099. During the preparation of this work, the authors used ChatGPT-4 and Claude Opus 5.5 to brainstorm preliminary proof strategies, verify algebraic steps, aid in literature review and citation, create the numerical experiments, and run a thorough review of the final manuscript. The authors reviewed and verified all mathematical outputs for accuracy and take full responsibility for the final contents of this manuscript.

\newpage

\appendix
\section{Appendix}
\subsection{AMP Literature Review}\label{section:AMP_lit_review}
Developed for compressed sensing \citep{Donoho_2009}, AMP is an iterative algorithm combining matrix multiplication with an entrywise denoising function and an Onsager correction that compensates for correlations created by reusing the same matrix. Its asymptotic behavior can be described by a scalar recursion, known as state-evolution, which tracks an effective Gaussian denoising problem \citep{Bayati_2011}. As emphasized by \cite{perry2016optimalitysuboptimalitypcaspiked}, AMP can share the recovery threshold of PCA while producing more accurate estimates above that threshold. For dense Gaussian noise ($q=1$) and an i.i.d. nonnegative Bernoulli spike with fixed density ($p=\Theta(1))$, \cite{deshpande2014informationtheoretically} proved that Bayes AMP attains the asymptotic minimum matrix mean-squared error in a regime where the associated scalar fixed-point equation has a unique solution. The nonnegative prior is important here and should not be confused with the centered sparse Rademacher prior. For the Gaussian spiked Wigner model, \cite{montanari2021estimation}
established state-evolution with spectral initialization above the BBP threshold ($\theta>1$). Their analysis includes known bounded priors, such as the sparse Rademacher prior with fixed $p$, and shows that prior-dependent denoising can improve
on PCA, with Bayes-optimality under a uniqueness condition on the informative
state-evolution fixed point.
For centered priors, however, zero overlap is itself a state-evolution fixed
point, and spectral initialization does not supply a correlated seed below
the BBP threshold; these guarantees therefore do not establish recovery
throughout the below-BBP information-theoretically possible regime
\citep{montanari2021estimation}.
More recently, \cite{li2022nonasymptotic} obtained finite-sample guarantees
for soft-thresholding AMP in the Gaussian spiked Wigner model, allowing a
growing number of iterations under explicit signal, sparsity, and
initialization assumptions.
Such mean squared error guarantees must be distinguished from exact support
and sign recovery.
Gaussianity is not essential to AMP itself:
\cite{bayati2015universality} established universality for classes of
non-Gaussian matrices, and \cite{hachem2024approximate} developed state-evolution for matrices with sparse variance profiles.
Nevertheless, these extensions do not by themselves transfer the preceding
recovery and optimality guarantees to \eqref{model_sparsity_introduction}.
In our model, the Bernoulli noise mask introduces an atom at zero and, when
$q\to0$, an $n$-dependent normalized noise distribution.
Extending the Gaussian recovery theory therefore requires checking the
noise and initialization assumptions across the spike and noise sparsity
regimes, rather than treating noise sparsity as a direct consequence of the
dense theory.

\subsection{Tables \ref{distinguish_table} and \ref{recovery_table} extended to include small-ball}\label{section:tables_extended}
See Section \ref{section:table_descriptions} for a full description of method thresholds and assumptions. For these tables, we adapt Tables \ref{distinguish_table} and \ref{recovery_table}, respectively, by adding in the thresholds for the small-ball prior assumption (Assumption \ref{spike_small_ball_assumption}). This addition impacts Quadratic and Diagonal Thresholding. To account for the separation between success and runtime for the different assumptions, we use the following notation: 

In Table \ref{distinguish_table_full}/Table \ref{recovery_table_full}, \textcolor{safePink}{$\textbf{QT}_p$} represents the Quadratic Thresholding polynomial-time algorithm achieving Distinguishability/Exact Recovery for a Rademacher spike and Distinguishability/Strong Recovery, respectively, for a Small-Ball spike (and thus also a Lower Bounded spike) planted in sub-Gaussian noise. \textcolor{safePink}{$\textbf{QT}$} represents the same, but both priors succeed in non-polynomial time. \textcolor{safePink}{$\textbf{QT}_{(p)}$} represents the same, but the lower bounded prior runs in polynomial time, while the small-ball prior runs in non-polynomial time. \textcolor{safePink}{$\textbf{QT}_p^{L}$}/\textcolor{safePink}{$\textbf{QT}^{L}$} represent the same, but the small-ball prior no longer succeeds and the lower bounded prior runs in polynomial/non-polynomial time, respectively.

Similarly, in Table \ref{distinguish_table_full}/Table \ref{recovery_table_full}, \textcolor{safeBlue}{$\textbf{DT}_l$} represents Diagonal Thresholding achieving Distinguishability/Strong Recovery, respectively, for the small-ball prior (and thus the lower bounded prior) in linear time. \textcolor{safeBlue}{$\textbf{DT}_l^L$} represents the same, but only the lower bounded spike prior succeeds.

\begin{table}[H]
    \centering
    \renewcommand{\arraystretch}{1.5} 
\begin{tabular}{|c|c|c|c|c|c|c|c|} \hline
    \multicolumn{7}{|c|}{\cellcolor{lightgray}\textbf{Doubly Sparse Distinguishability}} \\ \hline
    
    \cellcolor{lightgray}$1-q\ll 1/n^2$ & \cellcolor{black} & \textcolor{safePink}{$\textbf{QT}_p$}, \textcolor{safeBlue}{$\textbf{DT}_l$}\cellcolor{backgroundOrange} & \textcolor{safePink}{$\textbf{QT}_p$}, \textcolor{safeBlue}{$\textbf{DT}_l$}\cellcolor{backgroundOrange} & \textcolor{safePink}{$\textbf{QT}_p$}, \textcolor{safeBlue}{$\textbf{DT}_l^{L}$} \cellcolor{backgroundOrange} & \textcolor{safePink}{$\textbf{QT}$} \cellcolor{backgroundOrange} &\cellcolor{backgroundOrange} \\ \hline

    \cellcolor{lightgray}$1-q\gtrsim1/n^2$ & \cellcolor{black} & \textcolor{safePink}{$\textbf{QT}_p$}, \textcolor{safeBlue}{$\textbf{DT}_l$} \cellcolor{backgroundOrange} & \textcolor{safePink}{$\textbf{QT}_p$}, \textcolor{safeBlue}{$\textbf{DT}_l$}\cellcolor{backgroundOrange} & \textcolor{safePink}{$\textbf{QT}_p$}, \textcolor{safeBlue}{$\textbf{DT}_l^{L}$} \cellcolor{backgroundOrange} & \textcolor{safePink}{$\textbf{QT}$} \cellcolor{backgroundOrange} &\textcolor{safePurple}{IS} \cellcolor{backgroundOrange}\\ \hline

    \cellcolor{lightgray}$1-q\gtrsim\frac{(\log n)^2}{n^2}$ &  \cellcolor{black} & \textcolor{safePink}{$\textbf{QT}_p$}, \textcolor{safeBlue}{$\textbf{DT}_l$} \cellcolor{backgroundOrange}& \textcolor{safePink}{$\textbf{QT}_p$}, \textcolor{safeBlue}{$\textbf{DT}_l$}\cellcolor{backgroundOrange} & \textcolor{safePink}{$\textbf{QT}_p$}, \textcolor{safeBlue}{$\textbf{DT}_l^{L}$} \cellcolor{backgroundOrange} & \textcolor{safePurple}{IS}, \textcolor{safePink}{$\textbf{QT}$} \cellcolor{backgroundOrange} &\textcolor{safePurple}{IS} \cellcolor{backgroundOrange} \\ \hline

    \cellcolor{lightgray}$1-q\gtrsim\frac{\log n}{n}$ & \cellcolor{black} &  \textcolor{safePink}{$\textbf{QT}_p$}, \textcolor{safeBlue}{$\textbf{DT}_l$} \cellcolor{backgroundOrange}& \textcolor{safePink}{$\textbf{QT}_p$}, \textcolor{safeBlue}{$\textbf{DT}_l$}\cellcolor{backgroundOrange} & \textcolor{safePurple}{IS}, \textcolor{safePink}{$\textbf{QT}_p$}, \textcolor{safeBlue}{$\textbf{DT}_l^{L}$} \cellcolor{backgroundOrange}& \textcolor{safePurple}{IS}, \textcolor{safePink}{$\textbf{QT}$} \cellcolor{backgroundOrange} &\textcolor{safePurple}{IS} \cellcolor{backgroundOrange} \\ \hline

    \cellcolor{lightgray}$1-q\gg\frac{\log n}{n}$ & \cellcolor{black} &  \textcolor{safePink}{$\textbf{QT}_p$}, \textcolor{safeBlue}{$\textbf{DT}_l$} \cellcolor{backgroundOrange}& \textcolor{safePurple}{IS}, \textcolor{safePink}{$\textbf{QT}_p$}, \textcolor{safeBlue}{$\textbf{DT}_l$}\cellcolor{backgroundOrange} & \textcolor{safePurple}{IS}, \textcolor{safePink}{$\textbf{QT}_p$}, \textcolor{safeBlue}{$\textbf{DT}_l^{L}$} \cellcolor{backgroundOrange}& \textcolor{safePurple}{IS}, \textcolor{safePink}{$\textbf{QT}$} \cellcolor{backgroundOrange} &\textcolor{safePurple}{IS} \cellcolor{backgroundOrange} \\ \hline

    \cellcolor{lightgray}$1-q\gtrsim\tau^{-2}$ & \cellcolor{black} & \textcolor{safePurple}{IS}, \textcolor{safePink}{$\textbf{QT}_p$}, \textcolor{safeBlue}{$\textbf{DT}_l$} \cellcolor{backgroundOrange} & \textcolor{safePurple}{IS}, \textcolor{safePink}{$\textbf{QT}_p$}, \textcolor{safeBlue}{$\textbf{DT}_l$}\cellcolor{backgroundOrange}& \textcolor{safePurple}{IS}, \textcolor{safePink}{$\textbf{QT}_p$}, \textcolor{safeBlue}{$\textbf{DT}_l^{L}$} \cellcolor{backgroundOrange} & \textcolor{safePurple}{IS}, \textcolor{safePink}{$\textbf{QT}$} \cellcolor{backgroundOrange} & \textcolor{safePurple}{IS} \cellcolor{backgroundOrange}\\ \hline

    \cellcolor{lightgray}$q\gtrsim\Theta(1)$ & \cellcolor{black}  & \textcolor{safePurple}{IS}, \textcolor{safePink}{$\textbf{QT}_p$}, \textcolor{safeBlue}{$\textbf{DT}_l$}\cellcolor{backgroundOrange} & \textcolor{safePurple}{IS}, \textcolor{safePink}{$\textbf{QT}_p$}, \textcolor{safeBlue}{$\textbf{DT}_l$}\cellcolor{backgroundOrange} & \textcolor{safePurple}{IS}, \textcolor{safePink}{$\textbf{QT}_p$}, \textcolor{safeBlue}{$\textbf{DT}_l^{L}$} \cellcolor{backgroundOrange} & \textcolor{safePurple}{IS}, \textcolor{safePink}{$\textbf{QT}$} \cellcolor{backgroundOrange} &\textcolor{safePurple}{IS} \cellcolor{backgroundOrange} \\ \hline

    \cellcolor{lightgray}$q\gg \sqrt{\frac{\log n}{n}} $& \cellcolor{black} & \textcolor{safePurple}{IS}, \textcolor{safePink}{$\textbf{QT}_p$}, \textcolor{safeBlue}{$\textbf{DT}_l$}\cellcolor{backgroundOrange} & \textcolor{safePurple}{IS}, \textcolor{safePink}{$\textbf{QT}_p$}, \textcolor{safeBlue}{$\textbf{DT}_l^{L}$}\cellcolor{backgroundOrange} & \textcolor{safePurple}{IS}, \textcolor{safePink}{$\textbf{QT}$} \cellcolor{backgroundOrange} &\textcolor{safePurple}{IS} \cellcolor{backgroundOrange} & \textcolor{safePurple}{IS} \cellcolor{backgroundOrange}\\ \hline

    \cellcolor{lightgray}$q\gtrsim \sqrt{\frac{\log n}{n}} $& \cellcolor{black} & \textcolor{safePurple}{IS}, \textcolor{safePink}{$\textbf{QT}_p$}, \textcolor{safeBlue}{$\textbf{DT}_l$}\cellcolor{backgroundOrange} & \textcolor{safePurple}{IS}, \textcolor{safePink}{$\textbf{QT}$} \cellcolor{backgroundOrange} & \textcolor{safePurple}{IS}, \textcolor{safePink}{$\textbf{QT}$} \cellcolor{backgroundOrange} &\textcolor{safePurple}{IS} \cellcolor{backgroundOrange} & \textcolor{safePurple}{IS} \cellcolor{backgroundOrange}\\ \hline

    \cellcolor{lightgray}$q\gg \frac{\tau^2\log n}{n} $& \cellcolor{black}  &\textcolor{safePurple}{IS}, \textcolor{safePink}{$\textbf{QT}_p$}, \textcolor{safeBlue}{$\textbf{DT}_l$}  \cellcolor{backgroundOrange}& \textcolor{safePurple}{IS} \cellcolor{backgroundOrange} & \textcolor{safePurple}{IS}\cellcolor{backgroundOrange} & \textcolor{safePurple}{IS} \cellcolor{backgroundOrange} & \textcolor{safePurple}{IS} \cellcolor{backgroundOrange} \\ \hline

    \cellcolor{lightgray}$q\gtrsim \frac{\tau^2\log n}{n} $& \cellcolor{black}  &\textcolor{safePurple}{IS}, \textcolor{safePink}{$\textbf{QT}_p$}, \textcolor{safeBlue}{$\textbf{DT}_l^{L}$}  \cellcolor{backgroundOrange}& \textcolor{safePurple}{IS} \cellcolor{backgroundOrange} & \textcolor{safePurple}{IS}\cellcolor{backgroundOrange} & \textcolor{safePurple}{IS} \cellcolor{backgroundOrange} & \textcolor{safePurple}{IS} \cellcolor{backgroundOrange} \\ \hline

    \cellcolor{lightgray}$q\gtrsim \frac{\tau\log n}{n}/\frac{\tau^2 \log(nq)}{n} $& \cellcolor{black}  &\textcolor{safePurple}{IS}, \textcolor{safePink}{$\textbf{QT}$} / \textcolor{safeBlue}{$\textbf{DT}_l^{L}$}   \cellcolor{backgroundOrange}& \textcolor{safePurple}{IS} \cellcolor{backgroundOrange} & \textcolor{safePurple}{IS} \cellcolor{backgroundOrange} & \textcolor{safePurple}{IS} \cellcolor{backgroundOrange} & \textcolor{safePurple}{IS} \cellcolor{backgroundOrange} \\ \hline
     
     \cellcolor{lightgray}$q\gtrsim \frac{\log n}{n} $& \cellcolor{black} & \textcolor{safePurple}{IS} & \textcolor{safePurple}{IS}  &  \textcolor{safePurple}{IS} & \textcolor{safePurple}{IS} & \textcolor{safePurple}{IS}\\ \hline
     
     \cellcolor{lightgray}$0\le q\ll\frac{\log n}{n}$ & \cellcolor{black} & \textcolor{safePurple}{IS} & \textcolor{safePurple}{IS} & \textcolor{safePurple}{IS} & \textcolor{safePurple}{IS}  & \textcolor{safePurple}{IS}  \\ \hline
     
     \cellcolor{lightgray}& \cellcolor{lightgray}$np=O(1)$ & \cellcolor{lightgray}$np=\tau\to\infty$ & \cellcolor{lightgray}$np\ll\sqrt{\frac{n}{\log n}}$ & \cellcolor{lightgray}$np\lesssim\sqrt{\frac{n}{\log n}}$ & \cellcolor{lightgray}$np\lesssim \frac{n}{\log n}$ & \cellcolor{lightgray}$np=\Theta(n)$  \\ \hline
\end{tabular}
    \caption{Distinguishability results extended from Table \ref{distinguish_table} to include results for Assumption \ref{spike_small_ball_assumption}. See \ref{section:tables_extended} for details on notation.}
    \label{distinguish_table_full}
\end{table}
\begin{landscape}
\begin{table}[htbp]
\vspace{-.5in}
\hspace{-.5in}
    \centering
    \renewcommand{\arraystretch}{1.5} 
    \begin{tabular}{|c|c|c|c|c|c|c|c|c|c|c|} \hline
    \multicolumn{10}{|c|}{\cellcolor{lightgray}\textbf{Doubly Sparse Recovery}} \\ \hline
    
    \cellcolor{lightgray}$1-q\ll \frac{\log n}{n}$ & \cellcolor{lightgray} &  \textcolor{safePink}{$\textbf{QT}_{p}^{L}$}\textcolor{safeBlue}{$\textbf{DT}_l$} & \textcolor{safePink}{$\textbf{QT}_p$}\textcolor{safeBlue}{$\textbf{DT}_l$}\cellcolor{backgroundOrange} \!& \textcolor{safePink}{$\textbf{QT}_p$}\textcolor{safeBlue}{$\textbf{DT}_l$}\cellcolor{backgroundOrange} & \textcolor{safePink}{$\textbf{QT}_p$}\textcolor{safeBlue}{$\textbf{DT}_l^{L}$} \cellcolor{backgroundOrange} & \textcolor{safePink}{$\textbf{QT}_{(p)}$}\textcolor{safeBlue}{$\textbf{DT}_l^{L}$} \cellcolor{backgroundOrange} & \textcolor{safePink}{$\textbf{QT}_p^{L}$}\textcolor{safeBlue}{$\textbf{DT}_l^{L}$} \cellcolor{backgroundOrange} & \textcolor{safePink}{$\textbf{QT}^{L}$} \cellcolor{backgroundOrange} &\cellcolor{backgroundOrange} \\ \hline

    \cellcolor{lightgray}$1-q\gtrsim\frac{\log n}{n}$ & \cellcolor{lightgray} &  \textcolor{safePink}{$\textbf{QT}_p^{L}$}\textcolor{safeBlue}{$\textbf{DT}_l$} & \textcolor{safePink}{$\textbf{QT}_p$}\textcolor{safeBlue}{$\textbf{DT}_l$}\cellcolor{backgroundOrange} \! & \textcolor{safePink}{$\textbf{QT}_p$}\textcolor{safeBlue}{$\textbf{DT}_l$}\cellcolor{backgroundOrange} & \textcolor{safePink}{$\textbf{QT}_p$}\textcolor{safeBlue}{$\textbf{DT}_l^{L}$} \cellcolor{backgroundOrange} & \textcolor{safePink}{$\textbf{QT}_{(p)}$}\textcolor{safeBlue}{$\textbf{DT}_l^{L}$} \cellcolor{backgroundOrange} & \textcolor{safePink}{$\textbf{QT}_p^{L}$}\textcolor{safeBlue}{$\textbf{DT}_l^{L}$}  \cellcolor{backgroundOrange} & \textcolor{safePink}{$\textbf{QT}^{L}$} \cellcolor{backgroundOrange} &\textcolor{safePurple}{IS} \cellcolor{backgroundOrange}\\ \hline

    \cellcolor{lightgray}$1-q\gtrsim \frac{\log(\frac{n}{\log n})}{n/\log n}$ & \cellcolor{lightgray} &  \textcolor{safePink}{$\textbf{QT}_{p}^{L}$}\textcolor{safeBlue}{$\textbf{DT}_l$} & \textcolor{safePink}{$\textbf{QT}_p$}\textcolor{safeBlue}{$\textbf{DT}_l$}\cellcolor{backgroundOrange} \! & \textcolor{safePink}{$\textbf{QT}_p$}\textcolor{safeBlue}{$\textbf{DT}_l$}\cellcolor{backgroundOrange} & \textcolor{safePink}{$\textbf{QT}_p$}\textcolor{safeBlue}{$\textbf{DT}_l^{L}$} \cellcolor{backgroundOrange} & \textcolor{safePink}{$\textbf{QT}_{(p)}$}\textcolor{safeBlue}{$\textbf{DT}_l^{L}$} \cellcolor{backgroundOrange}& \textcolor{safePink}{$\textbf{QT}_p^{L}$}\textcolor{safeBlue}{$\textbf{DT}_l^{L}$}  \cellcolor{backgroundOrange} & \textcolor{safePurple}{IS}\textcolor{safePink}{$\textbf{QT}^{L}$} \cellcolor{backgroundOrange} &\textcolor{safePurple}{IS} \cellcolor{backgroundOrange} \\ \hline

    \cellcolor{lightgray}$1-q\gtrsim \frac{\log(\sqrt{\frac{n}{\log n}})}{\sqrt{n/\log n}}$ & \cellcolor{lightgray} &  \textcolor{safePink}{$\textbf{QT}_{p}^{L}$}\textcolor{safeBlue}{$\textbf{DT}_l$} &  \textcolor{safePink}{$\textbf{QT}_p$}\textcolor{safeBlue}{$\textbf{DT}_l$}\cellcolor{backgroundOrange} \! & \textcolor{safePink}{$\textbf{QT}_p$}\textcolor{safeBlue}{$\textbf{DT}_l$}\cellcolor{backgroundOrange} & \textcolor{safePink}{$\textbf{QT}_p$}\textcolor{safeBlue}{$\textbf{DT}_l^{L}$}  \cellcolor{backgroundOrange} & \textcolor{safePink}{$\textbf{QT}_{(p)}$}\textcolor{safeBlue}{$\textbf{DT}_l^{L}$} \cellcolor{backgroundOrange}& \textcolor{safePurple}{IS}\textcolor{safePink}{$\textbf{QT}_p^{L}$}\textcolor{safeBlue}{$\textbf{DT}_l^{L}$}  \cellcolor{backgroundOrange}& \textcolor{safePurple}{IS}\textcolor{safePink}{$\textbf{QT}^{L}$} \cellcolor{backgroundOrange} &\textcolor{safePurple}{IS}\cellcolor{backgroundOrange} \\ \hline

    \cellcolor{lightgray}$1-q\gtrsim \frac{\log(\sqrt[3]{\frac{n}{\log n}})}{\sqrt[3]{n/\log n}}$ & \cellcolor{lightgray} &  \textcolor{safePink}{$\textbf{QT}_{p}^{L}$}\textcolor{safeBlue}{$\textbf{DT}_l$} & \textcolor{safePink}{$\textbf{QT}_p$}\textcolor{safeBlue}{$\textbf{DT}_l$}\cellcolor{backgroundOrange} & \textcolor{safePink}{$\textbf{QT}_p$}\textcolor{safeBlue}{$\textbf{DT}_l$}\cellcolor{backgroundOrange} & \textcolor{safePink}{$\textbf{QT}_p$}\textcolor{safeBlue}{$\textbf{DT}_l^{L}$} \cellcolor{backgroundOrange} & \textcolor{safePurple}{IS}\textcolor{safePink}{$\textbf{QT}_{(p)}$}\textcolor{safeBlue}{$\textbf{DT}_l^{L}$} \cellcolor{backgroundOrange} & \textcolor{safePurple}{IS}\textcolor{safePink}{$\textbf{QT}_p^{L}$}\textcolor{safeBlue}{$\textbf{DT}_l^{L}$}  \cellcolor{backgroundOrange} \! & \textcolor{safePurple}{IS}\textcolor{safePink}{$\textbf{QT}^{L}$}\cellcolor{backgroundOrange} & \textcolor{safePurple}{IS}\cellcolor{backgroundOrange}\\ \hline

    \cellcolor{lightgray}$1-q\gtrsim \frac{\log(\sqrt[4]{\frac{n}{\log n}})}{\sqrt[4]{n/\log n}}$ & \cellcolor{lightgray} &  \textcolor{safePink}{$\textbf{QT}_{p}^{L}$}\textcolor{safeBlue}{$\textbf{DT}_l$} & \textcolor{safePink}{$\textbf{QT}_p$}\textcolor{safeBlue}{$\textbf{DT}_l$}\cellcolor{backgroundOrange} & \textcolor{safePink}{$\textbf{QT}_p$}\textcolor{safeBlue}{$\textbf{DT}_l$}\cellcolor{backgroundOrange} & \textcolor{safePurple}{IS},\textcolor{safePink}{$\textbf{QT}_p$}\textcolor{safeBlue}{$\textbf{DT}_l^{L}$}  \cellcolor{backgroundOrange} & \textcolor{safePurple}{IS}\textcolor{safePink}{$\textbf{QT}_{(p)}$}\textcolor{safeBlue}{$\textbf{DT}_l^{L}$} \cellcolor{backgroundOrange} & \textcolor{safePurple}{IS}\textcolor{safePink}{$\textbf{QT}_p^{L}$}\textcolor{safeBlue}{$\textbf{DT}_l^{L}$}  \cellcolor{backgroundOrange} \! & \textcolor{safePurple}{IS}\textcolor{safePink}{$\textbf{QT}^{L}$}\cellcolor{backgroundOrange} & \textcolor{safePurple}{IS}\cellcolor{backgroundOrange}\\ \hline

    \cellcolor{lightgray}$1-q\gtrsim \frac{\log(\sqrt[5]{\frac{n}{\log n}})}{\sqrt[5]{n/\log n}}$ & \cellcolor{lightgray} &  \textcolor{safePink}{$\textbf{QT}_{p}^{L}$}\textcolor{safeBlue}{$\textbf{DT}_l$} & \textcolor{safePink}{$\textbf{QT}_p$}\textcolor{safeBlue}{$\textbf{DT}_l$}\cellcolor{backgroundOrange} & \textcolor{safePurple}{IS}\textcolor{safePink}{$\textbf{QT}_p$}\textcolor{safeBlue}{$\textbf{DT}_l$}\cellcolor{backgroundOrange} & \textcolor{safePurple}{IS}\textcolor{safePink}{$\textbf{QT}_p$}\textcolor{safeBlue}{$\textbf{DT}_l^{L}$}  \cellcolor{backgroundOrange} & \textcolor{safePurple}{IS}\textcolor{safePink}{$\textbf{QT}_{(p)}$}\textcolor{safeBlue}{$\textbf{DT}_l^{L}$} \cellcolor{backgroundOrange} & \textcolor{safePurple}{IS}\textcolor{safePink}{$\textbf{QT}_p^{L}$}\textcolor{safeBlue}{$\textbf{DT}_l^{L}$}  \cellcolor{backgroundOrange} \! & \textcolor{safePurple}{IS}\textcolor{safePink}{$\textbf{QT}^{L}$}\cellcolor{backgroundOrange} & \textcolor{safePurple}{IS}\cellcolor{backgroundOrange}\\ \hline

    \cellcolor{lightgray}$1-q\gtrsim\frac{\log \tau}{\tau}$ & \cellcolor{lightgray} &  \textcolor{safePink}{$\textbf{QT}_{p}^{L}$}\textcolor{safeBlue}{$\textbf{DT}_l$} & \textcolor{safePurple}{IS}\textcolor{safePink}{$\textbf{QT}_p$}\textcolor{safeBlue}{$\textbf{DT}_l$}\cellcolor{backgroundOrange} & \textcolor{safePurple}{IS}\textcolor{safePink}{$\textbf{QT}_p$}\textcolor{safeBlue}{$\textbf{DT}_l$}\cellcolor{backgroundOrange} & \textcolor{safePurple}{IS}\textcolor{safePink}{$\textbf{QT}_p$}\textcolor{safeBlue}{$\textbf{DT}_l^{L}$}  \cellcolor{backgroundOrange} & \textcolor{safePurple}{IS}\textcolor{safePink}{$\textbf{QT}_{(p)}$}\textcolor{safeBlue}{$\textbf{DT}_l^{L}$} \cellcolor{backgroundOrange} & \textcolor{safePurple}{IS}\textcolor{safePink}{$\textbf{QT}_p^{L}$}\textcolor{safeBlue}{$\textbf{DT}_l^{L}$}  \cellcolor{backgroundOrange} \! & \textcolor{safePurple}{IS}\textcolor{safePink}{$\textbf{QT}^{L}$}\cellcolor{backgroundOrange} & \textcolor{safePurple}{IS}\cellcolor{backgroundOrange}\\ \hline

    \cellcolor{lightgray}$q\gtrsim\Theta(1)$ & \cellcolor{lightgray} &  \textcolor{safePink}{$\textbf{QT}_{p}^{L}$}\textcolor{safeBlue}{$\textbf{DT}_l$} & \textcolor{safePurple}{IS}\textcolor{safePink}{$\textbf{QT}_p$}\textcolor{safeBlue}{$\textbf{DT}_l$}\cellcolor{backgroundOrange} \! & \textcolor{safePurple}{IS}\textcolor{safePink}{$\textbf{QT}_p$}\textcolor{safeBlue}{$\textbf{DT}_l$}\cellcolor{backgroundOrange} & \textcolor{safePurple}{IS}\textcolor{safePink}{$\textbf{QT}_p$}\textcolor{safeBlue}{$\textbf{DT}_l^{L}$}  \cellcolor{backgroundOrange} & \textcolor{safePurple}{IS}\textcolor{safePink}{$\textbf{QT}_{(p)}$}\textcolor{safeBlue}{$\textbf{DT}_l^{L}$} \cellcolor{backgroundOrange}& \textcolor{safePurple}{IS}\textcolor{safePink}{$\textbf{QT}_p^{L}$}\textcolor{safeBlue}{$\textbf{DT}_l^{L}$} \cellcolor{backgroundOrange} & \textcolor{safePurple}{IS}\textcolor{safePink}{$\textbf{QT}^{L}$} \cellcolor{backgroundOrange} &\textcolor{safePurple}{IS} \cellcolor{backgroundOrange} \\ \hline



    \cellcolor{lightgray}$1\gg q\gtrsim (\frac{\log n}{n})^{\frac{1}{5}} \!$& \cellcolor{lightgray} & \textcolor{safePurple}{IS}\textcolor{safePink}{$\textbf{QT}_{p}^{L}$}\textcolor{safeBlue}{$\textbf{DT}_l$} & \textcolor{safePurple}{IS}\textcolor{safePink}{$\textbf{QT}_p$}\textcolor{safeBlue}{$\textbf{DT}_l$}\cellcolor{backgroundOrange} \! & \textcolor{safePurple}{IS}\textcolor{safePink}{$\textbf{QT}_p$}\textcolor{safeBlue}{$\textbf{DT}_l^{L}$}  \cellcolor{backgroundOrange} & \textcolor{safePurple}{IS}\textcolor{safePink}{$\textbf{QT}_{(p)}$}\textcolor{safeBlue}{$\textbf{DT}_l^{L}$} \cellcolor{backgroundOrange} & \textcolor{safePurple}{IS}\textcolor{safePink}{$\textbf{QT}_p^{L}$}\textcolor{safeBlue}{$\textbf{DT}_l^{L}$}  \cellcolor{backgroundOrange}& \textcolor{safePurple}{IS}\textcolor{safePink}{$\textbf{QT}^{L}$} \cellcolor{backgroundOrange} &\textcolor{safePurple}{IS} \cellcolor{backgroundOrange} & \textcolor{safePurple}{IS} \cellcolor{backgroundOrange}\\ \hline

    \cellcolor{lightgray}$q\gtrsim (\frac{\log n}{n})^{\frac{1}{4}} $& \cellcolor{lightgray} & \textcolor{safePurple}{IS}\textcolor{safePink}{$\textbf{QT}_{p}^{L}$}\textcolor{safeBlue}{$\textbf{DT}_l$} & \textcolor{safePurple}{IS}\textcolor{safePink}{$\textbf{QT}_p$}\textcolor{safeBlue}{$\textbf{DT}_l$}\cellcolor{backgroundOrange} \! & \textcolor{safePurple}{IS}\textcolor{safePink}{$\textbf{QT}_{(p)}$}\textcolor{safeBlue}{$\textbf{DT}_l^{L}$}  \cellcolor{backgroundOrange} & \textcolor{safePurple}{IS}\textcolor{safePink}{$\textbf{QT}_{(p)}$}\textcolor{safeBlue}{$\textbf{DT}_l^{L}$} \cellcolor{backgroundOrange} & \textcolor{safePurple}{IS}\textcolor{safePink}{$\textbf{QT}_p^{L}$}\textcolor{safeBlue}{$\textbf{DT}_l^{L}$}  \cellcolor{backgroundOrange}& \textcolor{safePurple}{IS}\textcolor{safePink}{$\textbf{QT}^{L}$} \cellcolor{backgroundOrange} &\textcolor{safePurple}{IS} \cellcolor{backgroundOrange} & \textcolor{safePurple}{IS} \cellcolor{backgroundOrange}\\ \hline

    \cellcolor{lightgray}$q\gtrsim (\frac{\log n}{n})^{\frac{1}{3}} $& \cellcolor{lightgray} & \textcolor{safePurple}{IS}\textcolor{safePink}{$\textbf{QT}_{p}^{L}$}\textcolor{safeBlue}{$\textbf{DT}_l$} & \textcolor{safePurple}{IS}\textcolor{safePink}{$\textbf{QT}_p$}\textcolor{safeBlue}{$\textbf{DT}_l$}\cellcolor{backgroundOrange} \! & \textcolor{safePurple}{IS}\textcolor{safePink}{$\textbf{QT}_{(p)}$}\textcolor{safeBlue}{$\textbf{DT}_l^{L}$}  \cellcolor{backgroundOrange} & \textcolor{safePurple}{IS}\textcolor{safePink}{$\textbf{QT}_{p}^{L}$}\textcolor{safeBlue}{$\textbf{DT}_l^{L}$} \cellcolor{backgroundOrange} & \textcolor{safePurple}{IS}\textcolor{safePink}{$\textbf{QT}_p^{L}$}\textcolor{safeBlue}{$\textbf{DT}_l^{L}$}  \cellcolor{backgroundOrange}& \textcolor{safePurple}{IS}\textcolor{safePink}{$\textbf{QT}^{L}$} \cellcolor{backgroundOrange} &\textcolor{safePurple}{IS} \cellcolor{backgroundOrange} & \textcolor{safePurple}{IS} \cellcolor{backgroundOrange}\\ \hline

    \cellcolor{lightgray}$q\gg (\frac{\log n}{n})^{\frac{2}{5}} $& \cellcolor{lightgray} & \textcolor{safePurple}{IS}\textcolor{safePink}{$\textbf{QT}_{p}^{L}$}\textcolor{safeBlue}{$\textbf{DT}_l$} & \textcolor{safePurple}{IS}\textcolor{safePink}{$\textbf{QT}_p$}\textcolor{safeBlue}{$\textbf{DT}_l$}\cellcolor{backgroundOrange} \! & \textcolor{safePurple}{IS}\textcolor{safePink}{$\textbf{QT}_{(p)}$}\textcolor{safeBlue}{$\textbf{DT}_l^{L}$}  \cellcolor{backgroundOrange} & \textcolor{safePurple}{IS}\textcolor{safePink}{$\textbf{QT}_{p}^{L}$}\textcolor{safeBlue}{$\textbf{DT}_l^{L}$} \cellcolor{backgroundOrange} & \textcolor{safePurple}{IS}\textcolor{safePink}{$\textbf{QT}^{L}$} \cellcolor{backgroundOrange}& \textcolor{safePurple}{IS} \cellcolor{backgroundOrange} &\textcolor{safePurple}{IS} \cellcolor{backgroundOrange} & \textcolor{safePurple}{IS} \cellcolor{backgroundOrange}\\ \hline


    \cellcolor{lightgray}$q\gtrsim (\frac{\log n}{n})^{\frac{1}{2}} $& \cellcolor{lightgray} & \textcolor{safePurple}{IS}\textcolor{safePink}{$\textbf{QT}_{p}^{L}$}\textcolor{safeBlue}{$\textbf{DT}_l$} & \textcolor{safePurple}{IS}\textcolor{safePink}{$\textbf{QT}_p$}\textcolor{safeBlue}{$\textbf{DT}_l$}\cellcolor{backgroundOrange} \! & \textcolor{safePurple}{IS}\textcolor{safePink}{$\textbf{QT}_p^{L}$}\textcolor{safeBlue}{$\textbf{DT}_l^{L}$}  \cellcolor{backgroundOrange} & \textcolor{safePurple}{IS}\textcolor{safePink}{$\textbf{QT}_{p}^{L}$}\textcolor{safeBlue}{$\textbf{DT}_l^{L}$} \cellcolor{backgroundOrange} & \textcolor{safePurple}{IS}\textcolor{safePink}{$\textbf{QT}^{L}$} \cellcolor{backgroundOrange}& \textcolor{safePurple}{IS} \cellcolor{backgroundOrange} &\textcolor{safePurple}{IS} \cellcolor{backgroundOrange} & \textcolor{safePurple}{IS} \cellcolor{backgroundOrange}\\ \hline

    \cellcolor{lightgray}$q\gtrsim (\frac{\log n}{n})^{\frac{3}{5}} $& \cellcolor{lightgray} & \textcolor{safePurple}{IS}\textcolor{safePink}{$\textbf{QT}_{p}^{L}$}\textcolor{safeBlue}{$\textbf{DT}_l$} & \textcolor{safePurple}{IS}\textcolor{safePink}{$\textbf{QT}_p$}\textcolor{safeBlue}{$\textbf{DT}_l$}\cellcolor{backgroundOrange} \! & \textcolor{safePurple}{IS}\textcolor{safePink}{$\textbf{QT}_p^{L}$}\textcolor{safeBlue}{$\textbf{DT}_l^{L}$}  \cellcolor{backgroundOrange} & \textcolor{safePurple}{IS}\textcolor{safePink}{$\textbf{QT}^{L}$}\cellcolor{backgroundOrange} & \textcolor{safePurple}{IS}\textcolor{safePink}{$\textbf{QT}^{L}$} \cellcolor{backgroundOrange}& \textcolor{safePurple}{IS} \cellcolor{backgroundOrange} &\textcolor{safePurple}{IS} \cellcolor{backgroundOrange} & \textcolor{safePurple}{IS} \cellcolor{backgroundOrange}\\ \hline

    \cellcolor{lightgray}$q\gtrsim (\frac{\log n}{n})^{\frac{2}{3}} $& \cellcolor{lightgray} & \textcolor{safePurple}{IS}\textcolor{safePink}{$\textbf{QT}_{p}^{L}$}\textcolor{safeBlue}{$\textbf{DT}_l$} & \textcolor{safePurple}{IS}\textcolor{safePink}{$\textbf{QT}_p$}\textcolor{safeBlue}{$\textbf{DT}_l$}\cellcolor{backgroundOrange} \! & \textcolor{safePurple}{IS}\textcolor{safePink}{$\textbf{QT}^{L}$} \cellcolor{backgroundOrange} & \textcolor{safePurple}{IS}\textcolor{safePink}{$\textbf{QT}^{L}$}\cellcolor{backgroundOrange} & \textcolor{safePurple}{IS}\textcolor{safePink}{$\textbf{QT}^{L}$} \cellcolor{backgroundOrange}& \textcolor{safePurple}{IS} \cellcolor{backgroundOrange} &\textcolor{safePurple}{IS} \cellcolor{backgroundOrange} & \textcolor{safePurple}{IS} \cellcolor{backgroundOrange}\\ \hline

    \cellcolor{lightgray}$q\gtrsim (\frac{\log n}{n})^{\frac{3}{4}} $& \cellcolor{lightgray} & \textcolor{safePurple}{IS}\textcolor{safePink}{$\textbf{QT}_{p}^{L}$}\textcolor{safeBlue}{$\textbf{DT}_l$} & \textcolor{safePurple}{IS}\textcolor{safePink}{$\textbf{QT}_p$}\textcolor{safeBlue}{$\textbf{DT}_l$}\cellcolor{backgroundOrange} \! & \textcolor{safePurple}{IS}\textcolor{safePink}{$\textbf{QT}^{L}$} \cellcolor{backgroundOrange} & \textcolor{safePurple}{IS}\textcolor{safePink}{$\textbf{QT}^{L}$}\cellcolor{backgroundOrange} & \textcolor{safePurple}{IS} \cellcolor{backgroundOrange}& \textcolor{safePurple}{IS} \cellcolor{backgroundOrange} &\textcolor{safePurple}{IS} \cellcolor{backgroundOrange} & \textcolor{safePurple}{IS} \cellcolor{backgroundOrange}\\ \hline

    \cellcolor{lightgray}$q\gtrsim (\frac{\log n}{n})^{\frac{4}{5}} $& \cellcolor{lightgray} & \textcolor{safePurple}{IS}\textcolor{safePink}{$\textbf{QT}_{p}^{L}$}\textcolor{safeBlue}{$\textbf{DT}_l$} & \textcolor{safePurple}{IS}\textcolor{safePink}{$\textbf{QT}_p$}\textcolor{safeBlue}{$\textbf{DT}_l$}\cellcolor{backgroundOrange} \! & \textcolor{safePurple}{IS}\textcolor{safePink}{$\textbf{QT}^{L}$} \cellcolor{backgroundOrange} & \textcolor{safePurple}{IS} \cellcolor{backgroundOrange} & \textcolor{safePurple}{IS} \cellcolor{backgroundOrange}& \textcolor{safePurple}{IS} \cellcolor{backgroundOrange} &\textcolor{safePurple}{IS} \cellcolor{backgroundOrange} & \textcolor{safePurple}{IS} \cellcolor{backgroundOrange}\\ \hline

    \cellcolor{lightgray}$q\gg \frac{\tau^6\log n}{n} $& \cellcolor{lightgray} & \textcolor{safePurple}{IS}\textcolor{safePink}{$\textbf{QT}_{p}^{L}$}\textcolor{safeBlue}{$\textbf{DT}_l$} & \textcolor{safePurple}{IS}\textcolor{safePink}{$\textbf{QT}_p$} \textcolor{safeBlue}{$\textbf{DT}_l$}\cellcolor{backgroundOrange}  & \textcolor{safePurple}{IS} \cellcolor{backgroundOrange} & \textcolor{safePurple}{IS}\cellcolor{backgroundOrange} & \textcolor{safePurple}{IS} \cellcolor{backgroundOrange}& \textcolor{safePurple}{IS} \cellcolor{backgroundOrange} \! & \textcolor{safePurple}{IS} \cellcolor{backgroundOrange} & \textcolor{safePurple}{IS} \cellcolor{backgroundOrange} \\ \hline

    \cellcolor{lightgray}$q\gg \frac{\tau^4\log n}{n}/\frac{\tau^2\log(nq)}{n} $& \cellcolor{lightgray} & \textcolor{safePurple}{IS}\textcolor{safePink}{$\textbf{QT}_{p}^{L}$}\textcolor{safeBlue}{$\textbf{DT}_l$} & \textcolor{safePurple}{IS}\textcolor{safePink}{$\textbf{QT}_{p}$}/\textcolor{safeBlue}{$\textbf{DT}_l^L$}\cellcolor{backgroundOrange} & \textcolor{safePurple}{IS} \cellcolor{backgroundOrange} & \textcolor{safePurple}{IS}\cellcolor{backgroundOrange} & \textcolor{safePurple}{IS} \cellcolor{backgroundOrange} & \textcolor{safePurple}{IS} \cellcolor{backgroundOrange} \! & \textcolor{safePurple}{IS} \cellcolor{backgroundOrange} & \textcolor{safePurple}{IS} \cellcolor{backgroundOrange} \\ \hline

    \cellcolor{lightgray}$q\gg \frac{\tau^3\log n}{n}/\frac{\tau^2\log(nq)}{n} $& \cellcolor{lightgray} & \textcolor{safePurple}{IS}\textcolor{safePink}{$\textbf{QT}_{p}^{L}$}\textcolor{safeBlue}{$\textbf{DT}_l$} & \textcolor{safePurple}{IS}\textcolor{safePink}{$\textbf{QT}_{(p)}$}/\textcolor{safeBlue}{$\textbf{DT}_l^L$}\cellcolor{backgroundOrange} & \textcolor{safePurple}{IS} \cellcolor{backgroundOrange} & \textcolor{safePurple}{IS}\cellcolor{backgroundOrange} & \textcolor{safePurple}{IS} \cellcolor{backgroundOrange} & \textcolor{safePurple}{IS} \cellcolor{backgroundOrange} \! & \textcolor{safePurple}{IS} \cellcolor{backgroundOrange} & \textcolor{safePurple}{IS} \cellcolor{backgroundOrange} \\ \hline

    \cellcolor{lightgray}$q\gtrsim \frac{\tau^2\log n}{n}/\frac{\tau^2\log(nq)}{n} $& \cellcolor{lightgray} & \textcolor{safePurple}{IS}\textcolor{safePink}{$\textbf{QT}_{p}^{L}$}\textcolor{safeBlue}{$\textbf{DT}_l$} & \textcolor{safePurple}{IS}\textcolor{safePink}{$\textbf{QT}_{p}^{L}$}/\textcolor{safeBlue}{$\textbf{DT}_l^L$} \cellcolor{backgroundOrange} & \textcolor{safePurple}{IS} \cellcolor{backgroundOrange} & \textcolor{safePurple}{IS} \cellcolor{backgroundOrange} & \textcolor{safePurple}{IS} \cellcolor{backgroundOrange} & \textcolor{safePurple}{IS} \cellcolor{backgroundOrange} \! & \textcolor{safePurple}{IS} \cellcolor{backgroundOrange} & \textcolor{safePurple}{IS} \cellcolor{backgroundOrange} \\ \hline

    \cellcolor{lightgray}$q\gtrsim \frac{\tau\log n}{n}/\frac{\tau^2\log(nq)}{n} $& \cellcolor{lightgray} & \textcolor{safePurple}{IS}\textcolor{safePink}{$\textbf{QT}_p^{L}$}\textcolor{safeBlue}{$\textbf{DT}_l$} & \textcolor{safePurple}{IS}\textcolor{safePink}{$\textbf{QT}^{L}$}/\textcolor{safeBlue}{$\textbf{DT}_l^L$} \cellcolor{backgroundOrange} & \textcolor{safePurple}{IS} \cellcolor{backgroundOrange} \! & \textcolor{safePurple}{IS} \cellcolor{backgroundOrange} & \textcolor{safePurple}{IS} \cellcolor{backgroundOrange} & \textcolor{safePurple}{IS} \cellcolor{backgroundOrange} & \textcolor{safePurple}{IS} \cellcolor{backgroundOrange} & \textcolor{safePurple}{IS} \cellcolor{backgroundOrange} \\ \hline
     
     \cellcolor{lightgray}$q\gtrsim \frac{\log n}{n} $& \cellcolor{lightgray} & \textcolor{safePurple}{IS}\textcolor{safePink}{$\textbf{QT}_p^{L}$}\textcolor{safeBlue}{$\textbf{DT}_l^L$} & \textcolor{safePurple}{IS} & \textcolor{safePurple}{IS} & \textcolor{safePurple}{IS} & \textcolor{safePurple}{IS} & \textcolor{safePurple}{IS} & \textcolor{safePurple}{IS} & \textcolor{safePurple}{IS}\\ \hline
     
     \cellcolor{lightgray}$0\le q\ll\frac{\log n}{n}$ & \cellcolor{lightgray} & \textcolor{safePurple}{IS} & \textcolor{safePurple}{IS} & \textcolor{safePurple}{IS} & \textcolor{safePurple}{IS} & \textcolor{safePurple}{IS} & \textcolor{safePurple}{IS} & \textcolor{safePurple}{IS}  & \textcolor{safePurple}{IS}  \\ \hline
     
     \cellcolor{lightgray}& \cellcolor{lightgray}$np$ & \cellcolor{lightgray}$np=$ & \cellcolor{lightgray}$np=$ & \cellcolor{lightgray}$np$ & \cellcolor{lightgray}$np\ll $ & \cellcolor{lightgray}$np\ll $ & \cellcolor{lightgray}$np\lesssim$ & \cellcolor{lightgray}$np\lesssim$ & \cellcolor{lightgray}$np$ \\
     
     \cellcolor{lightgray}& \cellcolor{lightgray}$\to0$ & \cellcolor{lightgray}$\Theta(1)$ & \cellcolor{lightgray}$\tau\to\infty$ & \cellcolor{lightgray}$ \sqrt[5]{\frac{n}{\log n}}$ & \cellcolor{lightgray}$\sqrt[4]{\frac{n}{\log n}}$ & \cellcolor{lightgray}$\sqrt[3]{\frac{n}{\log n}}$ & \cellcolor{lightgray}$\sqrt{\frac{n}{\log n}}$ & \cellcolor{lightgray}$ \frac{n}{\log n}$ & \cellcolor{lightgray} $=1$\\ \hline
\end{tabular}
    \caption{Recovery results extended from Table \ref{recovery_table} to include results for Assumption \ref{spike_small_ball_assumption}. See \ref{section:tables_extended} for details on notation.}
    \label{recovery_table_full}
\end{table}
\end{landscape}

\subsection{Fixed Mask Proof}\label{appendix:fixed_mask_proof}

\begin{proof-of-lemma}[\ref{fixedmasklemma}] 
Since $\tilde v$ is dense, we have
$$
\|v\|_0=\sum_{j=1}^n |(b)_j|=:S_n,
$$
where, by Assumption \ref{ind_bern_spike_assumption}, the random variables ${|(b)_j|}_{j=1}^n$ are independent Bernoulli random variables with parameter $p$. Thus, $S_n\sim\operatorname{Bin}(n,p)$ and $
\mathbb E[S_n]=np.$ 



Bernstein's Inequality for a sum of independent Bernoulli random variables gives
\begin{align*}
\mathbb P\left(|S_n-np|\geq \sqrt{np}\log(np)\right)
&\leq
2\exp\left(
-\frac{np\log^2(np)}
{2\left(np+\frac{1}{3}\sqrt{np}\log(np)\right)}
\right)\
=
2\exp\left(
-\frac{\log^2(np)}
{2\left(1+\frac{\log(np)}{3\sqrt{np}}\right)}
\right).
\end{align*}
Thus, for sufficiently large $np$,
$$
\frac{1}
{2\left(1+\frac{\log(np)}{3\sqrt{np}}\right)}
\geq \frac{1}{3}\implies \mathbb P\left(|S_n-np|\geq \sqrt{np}\log(np)\right)
\leq
2\exp\left(-\frac{\log^2(np)}{3}\right) \leq (np)^{-2}~,
$$
seeing as ${\log^2(np)}/{3}-2\log(np)\to\infty.$



\end{proof-of-lemma}

\subsection{Quadratic Thresholding Distinguishability Proof: Lower Bounded Spike}\label{quad_dist_proof}

\begin{proof-of-theorem}[\ref{quad_form_distinguish_thm}]
    Because independence is not required for distinguishability, we follow a similar procedure to \cite[Theorem 2.19]{ding2022subexponentialtimealgorithmssparsepca}, but adapt the bounds to reflect our sub-Gaussian, sparse noise matrix $Y.$

    \textbf{\textit{Type I Error.}} First, we consider pulling from the null distribution $\mathbb Q$ and mistakenly labeling it as the planted distribution. We consider any vector $\tilde v\in\mathcal I_{n,l}$. Using Lemma \ref{noise_quadratic_form_control}, we see
    \begin{align*}
        \mathbb Q\left(\tilde v^TX\tilde v\ge t\right)=\mathbb Q\left(\tilde v^TY\tilde v\ge t\right)\le 2\exp\left(-c\frac{q\theta^2 l^2}{4K_2^2B^4np^2}\right)~.
    \end{align*}
    Taking a union bound over the set $\mathcal I_{n,l},$ we find
    \begin{align*}
        \mathbb Q\left(\mathcal T\ge t\right)\le~& 2|\mathcal I_{n,l}|\exp\left(-c\frac{q\theta^2 l^2}{4K_2^2B^4np^2}\right) \\
        =~&2\binom{n}{l}2^{l}\exp\left(-c\frac{q\theta^2 l^2}{4K_2^2B^4np^2}\right) \\
        \le~& \left(2n\right)^{l+1}\exp\left(-c\frac{q\theta^2 l^2}{4K_2^2B^4np^2}\right) \\
        =~& \exp\left(-\frac{cq\theta^2 l^2-4K_2^2B^4np^2\left(l+1\right)\log\left(2n\right)}{4K_2^2B^4np^2}\right) \\
        \le~& \exp\left(-\left(\frac{cq\theta^2 l^2-16K_2^2B^4\cdot l \cdot np^2\log\left(n\right)}{4K_2^2B^4np^2}\right)\right)~,
    \end{align*}
    where we used the fact that $l\ge1$ in the final inequality. 
    
    \textbf{\textit{Type II Error.}} Next, we consider being in the planted distribution $\mathbb P,$ but mistakenly labeling it as the null distribution. To ensure this probability goes to $0,$ we consider the vector $\bar v\in\mathcal I_{n,l}$ that correctly guesses the signs of the top $l$ entries of the true spike $v.$ Then, we have $l\le \| v\|_0$ and
    \begin{align*}
        \bar v^{T}X^{}\bar v^{}=~&\frac{\theta}{np}\langle \bar v^{}, v^{}\rangle^2 + \bar v^{T}Y^{}\bar v^{} \\
        \ge~& \frac{\theta l^2}{B^2np} + \bar v^{T}Y^{}\bar v^{}~.
    \end{align*}
    Setting our test statistic to be $\mathcal T:=\max\limits_{w\in\mathcal I_{n,l}}w^TXw$ and establishing our thresholding value as $t=\frac{\theta l^2}{2B^2np}$, we see 
    \begin{align*}
        \mathbb P\left(\mathcal T^{}<t\right)\le~& \mathbb P\left(\bar v^{T}X^{}\bar v^{}<t\right) \\
        \le~& \mathbb P\left(2t+ \bar v^{T}Y^{}\bar v^{}<t\right) \\
        =~& \mathbb P\left(\bar v^{T}Y^{}\bar v^{}<-t\right) \\
        \le~& 2\exp\left(-c\frac{q\theta^2 l^2}{4K_2^2B^4np^2}\right)~,
    \end{align*}
    where we used Lemma \ref{noise_quadratic_form_control}.
     
     \textbf{\textit{Conclusion.}} Noticing that control over the type I error tail gives us control over the type II error tail, we use the fact that $l\ge 1$ to establish the interval of $l$ that allows the type I and II errors to converge to $0$ for $C_{1}$ sufficiently large enough:
    \begin{align*}
        l \in \left[\frac{C_{1}K_2^2B^4}{\theta^2}\cdot \frac{np^2\log(n)}{q}, \|v\|_0\right]~.
    \end{align*}
\end{proof-of-theorem}

\subsection{Quadratic Thresholding Support Recovery Proof: Lower Bounded Spike}\label{quad_recov_proof_exact}

\begin{proof-of-theorem}[\ref{quad_form_recovery_thm}]
    We begin first by establishing control over the quadratic form for each $i\in[n]$. Once we have established this control, we will condition on this high probability event. We use the indexing $X^{(i)}$ to represent a matrix identical to $X$ but with zeros in the $i$-th row and column.
    Let $\bar v^{(i)}$ be the vector with support size $l$ and entries drawn from $\{\pm 1, 0\}$ correctly guessing the signs of $l$ of the largest true spike entries of $v^{(i)}$. Then, accounting for the possibility of $i\in\text{supp}(v)$, we see $l\le ||v||_0-1$. Thus, we have 
    \begin{align*}
        \bar v^{(i)T}X^{(i)}\bar v^{(i)}=~&\frac{\theta}{np}\langle \bar v^{(i)}, v^{(i)}\rangle^2 + \bar v^{(i)T}Y^{(i)}\bar v^{(i)} \\
        \ge~& \frac{\theta l^2}{B^2np} + \bar v^{(i)T}Y^{(i)}\bar v^{(i)}~.
    \end{align*}
    Let $T^{(i)}=v^{*(i)T}X^{(i)}v^{*(i)}$ where $v^{*(i)}:=\arg\max\limits_{w\in\mathcal I_{n,l}^{(i)}}w^TX^{(i)}w$. Because the value of the maximum does not change, we importantly select $v^{*(i)}$ such that $(v^{*(i)})_i=0$, which will ensure $\langle v^{*(i)}, v^{(i)} \rangle =\langle v^{*(i)}, v \rangle$. Establishing our thresholding value as $t=\frac{\theta l^2}{2B^2np}$, we see 
    \begin{align*}
        \mathbb P(T^{(i)}<t)\le~& \mathbb P(\bar v^{(i)T}X^{(i)}\bar v^{(i)}<t) \\
        \le~& \mathbb P(2t+ \bar v^{(i)T}Y^{(i)}\bar v^{(i)}<t) \\
        =~& \mathbb P(\bar v^{(i)T}Y^{(i)}\bar v^{(i)}<-t) \\
        \le~& 2\exp\left(-c\frac{q\theta^2 l^2}{4K_2^2B^4np^2}\right)~,
    \end{align*}
    where we used Lemma \ref{noise_quadratic_form_control}.
    Seeing as we must have control over this term for each iteration of $i\in[n],$ we take a union bound to see
    \begin{align*}
        \mathbb P\left(\exists i\in[n] : T^{(i)}<t\right)\le~& \mathbb P\left(\bigcup_{i\in[n]}\{ T^{(i)}<t\}\right) \\
        \le~& n\mathbb P(T^{(i)}<t) \\
        \le~& \exp\left(-c\left(\frac{q\theta^2 l^2-8K_2^2B^4np^2\log(n)}{4K_2^2B^4np^2}\right)\right)~.
    \end{align*}
    Therefore, as long as $q\theta^2l^2\gtrsim8K_2^2B^4np^2\log(n),$ we see this probability term converges to $0.$
    Going forward, we condition on the high probability event $\mathcal B_1 := \{T^{(i)}\ge t \quad \forall i\in[n]\}$.

    We now establish a second high probability event that will allow us to control the noise contribution of the maximizing vectors $v^{*(i)}$. Define
    \begin{align*}
        \mathcal B_2 :=
        \left\{
        \max_{i\in[n]}
        \max_{w\in\mathcal I_{n,l}^{(i)}}
        \left|w^TY^{(i)}w\right|
        \le
        \frac{\theta l^2}{4B^2np}
        \right\}~.
    \end{align*}
    For any fixed $w\in\mathcal I_{n,l}^{(i)}$, Lemma \ref{noise_quadratic_form_control} gives
    \begin{align*}
        \mathbb P\left(
        \left|w^TY^{(i)}w\right|
        \ge
        \frac{\theta l^2}{4B^2np}
        \right)
        \le
        2\exp\left(
        -c\frac{q\theta^2l^2}{16K_2^2B^4np^2}
        \right).
    \end{align*}
    Union bounding over $i\in[n]$ and $w\in\mathcal I_{n,l}^{(i)}$ gives
    \begin{align*}
        \mathbb P(\mathcal B_2^c)
        &\le
        2n(2n)^l
        \exp\left(
        -c\frac{q\theta^2l^2}{16K_2^2B^4np^2}
        \right) \\
        &\le
        \exp\left(
        -c\left(
        \frac{q\theta^2l^2
        -c_2K_2^2B^4np^2l\log(n)}
        {16K_2^2B^4np^2}
        \right)
        \right),
    \end{align*}
    for a sufficiently large absolute constant $c_2$. Therefore, as long as $q\theta^2l^2
        \gtrsim
        C_2K_2^2B^4np^2l\log(n),$ we have $\mathbb P(\mathcal B_2)\to1$. Going forward, we condition on the high probability event $\mathcal B_1\cap\mathcal B_2$. Therefore, we have 
    \begin{align*}
        v^{*(i)T}X^{(i)} v^{*(i)}
        =
        \frac{\theta}{np}\left\langle v^{*(i)}, v^{(i)}\right\rangle^2
        +v^{*(i)T}Y^{(i)}v^{*(i)} 
        \ge
        \frac{\theta l^2}{2B^2np},
    \end{align*}
    while, by the definition of $\mathcal B_2$,
    \begin{align*}
        \left|v^{*(i)T}Y^{(i)}v^{*(i)}\right|
        \le
        \frac{\theta l^2}{4B^2np}~.
    \end{align*}
    Thus,
    \begin{align*}
        \frac{\theta}{np}
        \left\langle v^{*(i)},v^{(i)}\right\rangle^2
        =
        v^{*(i)T}X^{(i)}v^{*(i)}
        -
        v^{*(i)T}Y^{(i)}v^{*(i)} \ge
        \frac{\theta l^2}{2B^2np}
        -
        \frac{\theta l^2}{4B^2np}
        =
        \frac{\theta l^2}{4B^2np}.
    \end{align*}
    Therefore,
    \begin{align*}
        \left\langle v^{*(i)},v^{(i)}\right\rangle^2
        \ge
        \frac{l^2}{4B^2}.
    \end{align*}
    By our choice of $v^{*(i)}$, we have
    $\langle v^{*(i)},v^{(i)}\rangle
    =
    \langle v^{*(i)},v\rangle$.
    Combining the bounds above, we obtain
    \begin{align}
        \mathbb P\left(
        \exists i\in[n]:
        \langle v^{*(i)},v\rangle^2
        \le
        \frac{l^2}{4B^2}
        \right)
        &\le~ \mathbb P(\mathcal B_1^c)+\mathbb P(\mathcal B_2^c) \notag \\
        &\le 2 \exp\left(
        -c\left(
        \frac{q\theta^2l^2
        -\max{\{8,c_2 l\}}K_2^2B^4np^2\log(n)}
        {4K_2^2B^4np^2}
        \right)
        \right)\label{QT_recov_C_2_tail}~.
    \end{align}

    Finally, now that we have a computed lower bound for the correlation, we take our computed $v^{*(i)}$ and multiply it back into the $i$-th row of $X$ to determine whether the entry is spiked or not. If $i\in$ supp$(v)$, then we can use the above bound to see
    \begin{align}\label{gaussian_bottleneck}
        \left|(Xv^{*(i)})_i\right|
        =
        \left|
        \frac{\theta}{np}v_i\langle v^{*(i)},v\rangle
        +e_i^TYv^{*(i)}
        \right|
        \ge
        \left|
        \frac{\theta}{np}
        \frac{1}{B}
        \frac{l}{2B}
        \right|
        -
        \left|e_i^TYv^{*(i)}\right|
        =
        \frac{\theta l}{2B^2np}
        -
        \left|e_i^TYv^{*(i)}\right|~.
    \end{align}
    If $i$ is not in the support, then an inequality is not necessary and we are left only with the latter term involving the noise. Therefore, in either case, we must control the latter term. Notice, this step highlights the importance of our leave-one-out method, as it allows us to introduce independence: $v^{*(i)}$ is created without any dependence on the $i$-th row and column of $X$ (and therefore $Y$). Thus, $e_i^TY$ and $v^{*(i)}$ are completely independent of each other. 
    Thus, we proceed with an application of Lemma 
    \ref{noise_quadratic_form_control} one final time and arrive at the following bound: 
    \begin{align*}
        \mathbb P\left(
        |e_i^TYv^{*(i)}|
        \ge
        \frac{\theta l}{6B^2np}
        \right)
        \le
        2\exp\left(
        -c\frac{nq}{K_2^2l}
        \frac{\theta^2l^2}{36B^4n^2p^2}
        \right)
        =
        2\exp\left(
        -c\frac{\theta^2ql}{36K_2^2B^4np^2}
        \right)~.
    \end{align*}
    Taking a union bound over all $i\in[n],$ we see
    \begin{align*}
        \mathbb P\left(
        |e_i^TYv^{*(i)}|
        \le
        \frac{\theta l}{6B^2np}
        \quad \forall i\in[n]
        \right)
        \ge
        1-
        \exp\left(
        -c\left(
        \frac{\theta^2ql
        -72K_2^2B^4np^2\log(n)}
        {36K_2^2B^4np^2}
        \right)
        \right)~,
    \end{align*}
    which converges to $1$ when
    \begin{align*}
        l>
        \frac{c_3K_2^2B^4}{\theta^2}
        \frac{np^2\log(n)}{q}~,
    \end{align*}
    for $c_3$ sufficiently large. Combining this tail bound with the bound on $\mathcal B_1^c$ and $\mathcal B_2^c$ from \eqref{QT_recov_C_2_tail}, we find that the overall probability converges to $1$ under the condition
    \begin{align*}
        l \in
        \left[
        \frac{C_2K_2^2B^4}{\theta^2}
        \cdot
        \frac{np^2\log(n)}{q},
        \|v\|_0-1
        \right]~,
    \end{align*}
    for sufficiently large $C_2$ that is chosen to subsume the different tail bounds.

    Therefore, in conclusion, we find 
    \begin{align}
        |(Xv^{*(i)})_i|
        \ge
        \frac{\theta l}{3B^2np}
        \quad \text{ for } i \in \text{supp}(v)
        \text{ with }
        \operatorname{sgn}((Xv^{*(i)})_i)
        =
        \operatorname{sgn}(v_i\langle v^{*(i)},v\rangle)~,\label{sign_start}
    \end{align}
    and 
    \begin{align*}
        |(Xv^{*(i)})_i|
        \le
        \frac{\theta l}{6B^2np}
        \quad \text{ for }i\notin\text{supp}(v)~,
    \end{align*}
    with probability
    \begin{align*}
        1-
        \exp\left(
        -\left(
        \frac{cq\theta^2l^2
        -C_2K_2^2B^4np^2\log(n)}
        {4K_2^2B^4np^2}
        \right)
        \right)~.
    \end{align*}

    \paragraph{Case 2: $\|v\|_0=1$.}
    In this specific case, the above proof method does not work as $l\le\|v\|_0-1.$ More specifically, if the sole spiked entry is $v_i,$ then it would be excluded from $v^{*(i)}.$ Thus, in the case $\|v\|_0=1,$ we perform diagonal thresholding without the use of $l.$ In order to do so, we must show all of the $n-1$ entries on the diagonal fall beneath the thresholding parameter and the spike entry is larger than this value. Using Lemma \ref{noise_quadratic_form_control}, we see every noise element on the diagonal falls beneath the thresholding parameter: 
    \begin{align}
        \mathbb P\left(
        \exists i\in[n] :
        |Y_{ii}|
        \ge
        \frac{\theta}{2B^2np}
        \right)
        \le
        n\exp\left(
        -\frac{c\theta^2}{4B^4K_2^2}
        \frac{q}{np^2}
        \right)~,
        \label{recovery_l1_noise}
    \end{align}
    under the condition $q\ge C_2 B^4K_2^2np^2\log(n)$ for sufficiently large $C_2$. Thus, we find a lower bound for the $n-1$ entries not in the spike support. As for the sole spiked entry, we use the above control on the noise terms to see
    \begin{align*}
        |X_{ii}|
        =
        \frac{\theta}{np}v_i^2-|Y_{ii}|
        \ge
        \frac{\theta}{np}
        \left|\frac{1}{B}\frac{1}{B}\right|
        -
        \frac{\theta}{2B^2np}
        \ge
        \frac{\theta}{2B^2np}~,
    \end{align*}
    with probability controlled by \eqref{recovery_l1_noise}.
\end{proof-of-theorem}

\bibliographystyle{plainnat}
\bibliography{bib.bib} 

\end{document}